\documentclass{article}

\usepackage[accepted]{icml2024}

\usepackage{amsmath}
\usepackage{amssymb}
\usepackage{mathtools}
\usepackage{amsthm}
\usepackage{hyperref}

\usepackage{microtype}
\usepackage{graphicx}
\usepackage{booktabs}

\usepackage[utf8]{inputenc} 
\usepackage[T1]{fontenc}    
\usepackage{url}            
\usepackage{amsfonts}       
\usepackage{nicefrac}       
\usepackage{xcolor}         
\usepackage{algorithm}
\usepackage{subcaption}
\usepackage{caption}
\usepackage{float}

\usepackage[capitalize,noabbrev]{cleveref}

\theoremstyle{plain}
\newtheorem{theorem}{Theorem}[section]

\newtheorem{lemma}[theorem]{Lemma}

\theoremstyle{definition}
\newtheorem{definition}[theorem]{Definition}

\theoremstyle{remark}

\usepackage[textsize=tiny]{todonotes}

\newcommand{\fid}{\texttt{FID}}
\newcommand{\dbi}{\texttt{DBI}}
\newcommand{\avgfid}{\texttt{AVG-FID}}

\newcommand{\linfid}{\texttt{LIN-FID}}
\newcommand{\flin}{F_{lin}}
\newcommand{\argmax}{\text{argmax}}
\newcommand{\argmin}{\text{argmin}}
\newcommand{\cscg}{\text{CSC}_{\mathcal{G}}}
\newcommand{\defeq}{\vcentcolon=}

\icmltitlerunning{Feature Importance Disparities for Data Bias Investigations}

\begin{document}

\twocolumn[
\icmltitle{Feature Importance Disparities for Data Bias Investigations}




\begin{icmlauthorlist}
\icmlauthor{Peter W Chang}{hbs}
\icmlauthor{Leor Fishman}{hcol}
\icmlauthor{Seth Neel}{hbs}
\end{icmlauthorlist}

\icmlaffiliation{hbs}{SAFR AI Lab, Harvard Business School $\&$ Kempner Institute, Boston, MA}
\icmlaffiliation{hcol}{Harvard College, Cambridge, MA}

\icmlcorrespondingauthor{Seth Neel}{sethneel.ai@gmail.com}

\icmlkeywords{explainability, fairness, auditing}

\vskip 0.3in
]



\printAffiliationsAndNotice{}

\begin{abstract}
    It is widely held that one cause of downstream bias in classifiers is bias present in the training data. Rectifying such biases may involve context-dependent interventions such as training separate models on subgroups, removing features with bias in the collection process, or even conducting real-world experiments to ascertain sources of bias. Despite the need for such data bias investigations, few automated methods exist to assist practitioners in these efforts. In this paper, we present one such method that given a dataset $X$ consisting of protected and unprotected features, outcomes $y$, and a regressor $h$ that predicts $y$ given $X$, outputs a tuple $(f_j, g)$, with the following property: $g$ corresponds to a subset of the training dataset $(X, y)$, such that the $j^{th}$ feature $f_j$ has much larger (or smaller) \emph{influence} in the subgroup $g$, than on the dataset overall, which we call \emph{feature importance disparity} ($\fid$).  We show across $4$ datasets and $4$ common feature importance methods of broad interest to the machine learning community that we can efficiently find subgroups with large $\fid$ values even over exponentially large subgroup classes and in practice these groups correspond to subgroups with potentially serious bias issues as measured by standard fairness metrics.
\end{abstract}

\section{Introduction}
\label{sec:intro}
Machine learning is rapidly becoming a more important, yet more opaque part of our lives and decision making -- with increasingly high stakes use cases such as recidivism analysis \cite{angwin}, loan granting and terms \cite{dastille} and child protective services \cite{child}. One of the hopes of wide-scale ML deployment has been that those algorithms might be free of our human biases and imperfections. This hope was, unfortunately, naive. Over the last decade, an interdisciplinary body of research has shown that machine learning algorithms can be deeply biased in both subtle and direct ways \cite{fairml}, and has focused on developing countless techniques to produce fairer models \cite{fair_ml_survey}. 

One of the primary causes of model bias is bias inherent in the training data, rather than an explicitly biased training procedure. While the majority of work on fairness seeks to remove bias by learning a fairer representation of the data \cite{zem_fair} or by explicitly constraining the downstream classifier to conform to a specific fairness notion \cite{eo_fair, red_fair}, fairness notions have been shown to be brittle and often times contradictory \cite{dwork_comp, Kleinberg2016InherentTI}. More importantly, these approaches elide what could be a more important question for the practitioner: \emph{What is the source of bias in the training data, and what subgroups in the data are being effected?} 

We call the process of answering this question a \emph{data bias investigation} ($\dbi$), and in this paper we develop a technique to aid in a $\dbi$ by allowing an analyst to (efficiently and provably) identify structured subsets of the training data to focus their bias investigation. Prior work has typically focused on identifying such subsets by finding subsets of the training data that maximally violate a specific fairness criterion, which typically corresponds to the classifier having a higher error rate on the group than on the population. We take a very different approach; rather than optimizing for a specific fairness notion, we find subgroups where \emph{a specific feature in the data has out-sized impact in the subgroup, relative to the population as a whole.} 

To build some intuition for why feature importance disparities, $\fid$ as we call them, might be a useful notion when looking for dataset bias, consider the following example from Cynthia Dwork, widely regarded as pioneer in the field of algorithmic bias, in a $2015$ in a New York Times interview \cite{nyt}:
    \textit{Suppose we have a minority group in which bright students are steered toward studying math, and suppose that in the majority group bright students are steered instead toward finance. An easy way to find good students is to look for students studying finance, and if the minority is small, this simple classification scheme could find most of the bright students. But not only is it unfair to the bright students in the minority group, it is also low utility.}
    
Unpacking this example further, the feature $\texttt{is-finance-major}$ is predictive in finding bright students in the population at large but not in the minority group. Meanwhile, the feature $\texttt{is-math-major}$ is highly predictive in the minority group but not at all in the majority group. The classifier that only selects students who study finance is unfair to the minority group, exactly because of the differing importance of these features in the two groups. Once the subgroups have been identified, the actions the analyst then takes is then entirely context-dependent. This could include ``easy'' fixes like training a separate model on the subgroup found or excluding a specific feature from training, or more complex remedies like investigating how the specific feature or the outcome variable was collected in the group, and if there is bias in that collection process. 

While simple models like decision trees or linear regression come equipped with intuitive notions of feature importance, in general there is no definitive notion of feature importance for complex models. Approaches that have garnered substantial attention include local model agnostic methods like LIME \cite{lime}, SHAP \cite{shap}, model-specific saliency maps \cite{saliency}, and example-based counterfactual explanations \cite{molnar2022}. Concerns about the stability and robustness of the most widely used feature importance notions, including the ones we study, have been raised \cite{Dai_2022, agarwal2022rethinking, alvarezmelis2018robustness, bansal2020sam, Dimanov2020YouST, slack2020fooling} and these notions are often at odds with each other, so none can be considered definitive \cite{disagree}. Regardless of these limitations, these notions are used widely in practice today, and are still useful as a diagnostic tool as we eventually propose. 

Fixing any of these notions of feature importance, given a small set of protected subgroups, it would be simple to iterate through the subgroups and features, compute the $\fid$ for each feature with respect to each subgroup, and then select the feature and subgroup that shows the largest disparity. However, it is also known in the fairness literature that while a classifier may look fair when comparing a given fairness metric across a handful of sensitive subgroups, when the notion of a sensitive subgroup is generalized to encompass combinations and interactions between sensitive features (known as \emph{rich subgroups} \cite{rich}), large disparities can emerge. We verify that this phenomenon of rich subgroups uncovering much larger disparities than marginal subgroups alone also holds for feature importance in Subsection~\ref{sec:rich_marg_main}. Even for simple definitions of rich subgroup such as conjunctions of binary features, the number of subgroups is exponential in the number of sensitive attributes, and so it is infeasible to compute the metric on each sensitive subgroup and find the largest value by brute force. This raises an obvious question in light of the prior discussion, although one that to the best of our knowledge has not been thoroughly studied: \emph{When applied to classifiers and datasets where bias is a concern, do these feature importance notions uncover substantial differences in feature importance across rich subgroups, and can they be efficiently detected?}

\subsection{Case Study: Data Bias Investigation on COMPAS}
\label{sec:case_study}

Before diving into the technical details of our method, we illustrate how our technique can yield interesting potential interventions when training a random forest classifier $h$ to predict \texttt{two-year-recidivism} on the \texttt{COMPAS} dataset \cite{angwin}. Using the feature importance method SHAP \cite{shap} and our Algorithm~\ref{alg:cap}, we find that the feature $\texttt{priors-count}$ is substantially more important when predicting \texttt{two-year-recidivism} on a subgroup largely defined by Native-American and African-American males that makes up ${\sim} 9\%$ of the training set (Figure~\ref{fig:case_study}). While we further discuss in Appendix~\ref{sec:limitations} that disparities in feature importance do not guarantee that a subgroup has a fairness disparity, we show in Subsection~\ref{subsec:fairness} that empirically this is often the case. In this example, we find that conventional fairness metrics are slightly worse on this subgroup relative to the population, with ${\sim}1\%$ lower accuracy and ${\sim}1\%$ higher false positive rate (FPR). One simple solution to increase accuracy on the subgroup and potentially reduce the disparity is to train a separate model $h_g$ for the subgroup, which we find drops the FPR in the subgroup by $7.5\%$ and increases accuracy by $7\%$. Furthermore, the disparity between the average SHAP value for the feature $\texttt{priors-count}$ using $h_g$, and the average SHAP value over whole population using $h$ is more than halved. Another intuitive technical solution is to train a new model, $h_{-f}$, without the feature $\texttt{priors-count}$. With this solution, there is a moderate decrease in fairness disparity, but it also comes with a noticeable drop in model performance, likely due to removing the predictive power of the feature $\texttt{priors-count}$. On a qualitative level, identifying this subgroup could also motivate further research into how Native-American and African-American males in the \texttt{COMPAS} dataset are policed differently, possibly resulting in measurement biases in the $\texttt{priors-count}$ feature or the \texttt{two-year-recidivism} outcome.

\begin{figure}[h]
  \centering
  \includegraphics[width=.98\linewidth]{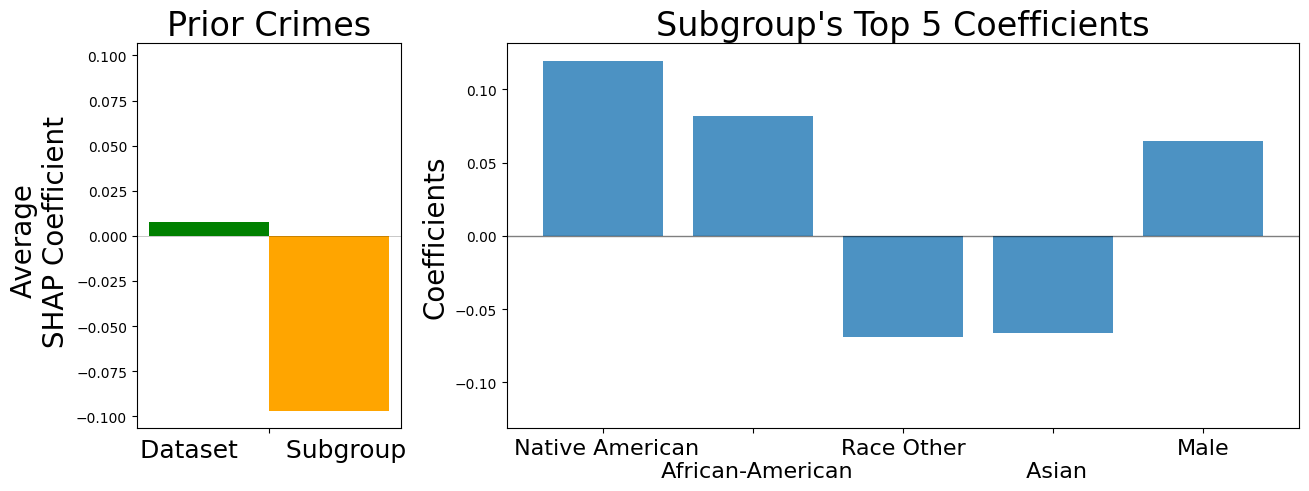}
  \caption{Exploring a high $\fid$ subgroup/feature pair for COMPAS. The first graph compares the average SHAP feature importance for \texttt{priors-count} in the subgroup vs. the dataset as a whole. The second graph shows the $5$ largest coefficients of the linear function of sensitive attributes that define the subgroup.}
  \label{fig:case_study}
\end{figure}

\subsection{Results}

While the prior example clearly illustrates the utility of our method at a high level, the devil is in the details, and in the rest of the paper, we formalize the notions of feature importance and protected rich subgroups along with our methods for efficiently detecting potentially biased subgroups.

Our most important contribution is introducing the notion of \emph{feature importance disparity} in the context of recently developed feature importance notions, and with respect to rich subgroups (Definition~\ref{def:fid}). We categorize a feature importance notion as \emph{separable} or not, based on whether it can be expressed as a sum over points in the subgroup (Definition~\ref{def:sep}) and define a variant of $\fid$, the \emph{average feature importance disparity} ($\avgfid$, Definition~\ref{def:avg_sep}). Our main theoretical contribution is Theorem~\ref{thm:sep} in Section~\ref{sec:sep}, which says informally that although the problem of finding the maximal $\fid$ subgroup is NP-hard in the worst case (Appendix~\ref{sec:hardness_app}), given access to an oracle for \emph{cost-sensitive classification} with respect to the rich subgroup class $\mathcal{G}$, (Definition~\ref{def:csc}), Algorithm~\ref{alg:cap} efficiently learns the subgroup with maximal $\fid$ for any separable feature importance notion.


In Section~\ref{sec:experiments}, we conduct a thorough empirical evaluation, auditing for large $\fid$ subgroups on the Student \cite{student}, COMPAS \cite{angwin}, Bank \cite{MORO201422}, and Folktables \cite{folktables} datasets, using LIME, SHAP, saliency maps, and linear regression coefficient as feature importance notions. Our experiments establish the following: (i) Across all (dataset, importance notion) pairs, we can find subgroups defined as functions of sensitive features that have large $\fid$ with respect to a given feature (Table~\ref{tab:summary}, Figures~\ref{fig:all_de}, \ref{fig:vid_dist}). (ii) Inspecting the coefficients of these subgroups yields interesting discussion about potential dataset bias (Figure~\ref{fig:case_study}, Section~\ref{sec:key_sub_app}). (iii) In about half the cases, rich subgroups yield higher out of sample $\fid$ compared to only searching subgroups defined by a single sensitive attribute, justifying the use of rich subgroups (Section~\ref{sec:rich_marg_main}). (iv) These subgroups have disparities in accepted fairness metrics such as demographic parity and calibration (Table~\ref{tab:fair_compare}). Conversely, rich subgroups that maximally violate fairness metrics also express large $\fid$ values (Table~\ref{tab:gerry_compare})

These results generalize out of sample, both in terms of the $\fid$ values found, and the sizes of the corresponding subgroups found (Appendix~\ref{subsec:valid}). Taken together, these theoretical and empirical results highlight our methods as an important addition to the diagnostic toolkit for $\dbi$ in tabular datasets with sensitive features.
\section{Related Work}
\label{sec:related}

There is substantial work investigating bias in the context of machine learning models and their training data \cite{fairml, fair_ml_survey}. We are motivated at a high level by existing work on dataset bias \cite{kamiran2012data, tommasi2017deeper, li2019repair}, however, to the best of our knowledge, this is the first work investigating the disparity in feature importance values in the context of rich subgroups as a fairness diagnostic. For more related work, see Appendix~\ref{sec:rel_work_app}.

\textbf{Anomalous Subgroup Discovery}. In terms of approach, two closely related works are \cite{Dai_2022} and \cite{balagopalan2022explainability} which link fairness concerns on sensitive subgroups with model explanation quality, as measured by properties like stability and fidelity. Our work differs in that we are focused on the magnitude of explanation disparities themselves rather than their ``quality,'' and that we extend our results to the rich subgroup setting. Our algorithm for searching an exponentially large subgroup space is a novel and necessary addition to work in this space. Another area of research looks to prove that a chosen score function satisfies the linear time subset scanning property \cite{neill2012spatialsubset} which can then be leveraged to search the subgroup space for classifier bias \cite{Zhang2016ltss, boxer2023auditing} in linear time. While it is hard to say with absolute certainty that this approach would not be useful it is not immediately apparent how we would force a subset scanning method to optimize over \emph{rich subgroups}.

\textbf{Rich Subgroups and Multicalibration}. At a technical level, the most closely related papers are \cite{gerry, calib} which introduce the notion of the rich subgroup class $\mathcal{G}$ over sensitive features in the context of learning classifiers that are with respect to equalized odds or calibration. Our Algorithm~\ref{alg:cap} fits into the paradigm of ``oracle-efficient" algorithms for solving constrained optimization problems introduced in \cite{msr} and developed in the context of rich subgroups in \cite{gerry, rich, calib}. There has been much recent interest in learning multicalibrated predictors because of connections to uncertainty estimation and omnipredictors \cite{multicalib_2, multicalib_omni, multicalib_roth}. None of these works consider feature importance disparities. 

\textbf{Feature Importance Notions}. 
For the field of interpretable or explainable machine learning, we refer to the survey by \cite{molnar2022}. The most relevant works cover methods used to investigate the importance of a feature in a given subset of the dataset. Local explanation methods assign a feature importance for every point $(x,y)$ and define a notion of importance in a subgroup by summing or averaging over the points in the subgroup as we do in Definitions~\ref{def:sep}, \ref{def:avg_sep}.
\section{Preliminaries}
\label{sec:preliminaries}

Let $X^n$ represent our dataset, consisting of $n$ individuals defined by the tuple $((x,x'),y)$ where $x \in \mathcal{X}_{sense}$ is the vector of protected features, $x' \in \mathcal{X}_{safe}$ is the vector of unprotected features, and $y \in \mathcal{Y}$ denotes the label. With $X =  (x,x') \in \mathcal{X} = \mathcal{X}_{sense} \times \mathcal{X}_{safe} \subset \mathbb{R}^d$ denoting a joint feature, the data points $(X,y)$ are drawn i.i.d. from a distribution $\mathcal{R}$. Let $h: \mathcal{X} \to \mathcal{Y}$ denote a classifier or regressor that predicts $y$ from $X$. We define a \emph{rich subgroup class} $\mathcal{G} = \{g_{\alpha}\}_{\alpha \in \Omega}$ as a collection of functions $g: \mathcal{X}_{sens} \to [0,1]$, where $g(x')$ denotes the membership of point $X =  (x,x')$ in group $g$. Note that this is the same subgroup definition as in \cite{gerry}, but without the constraint that $g(x') \in \{0 ,1\}$, which supports varying degrees of group membership. E.g. a biracial person may be $.5$ a member of two racial groups. Let $f_j, j \in [d]$ denote the $j^{th}$ feature in $\mathcal{X} \subset \mathbb{R}^d$. Then for a classifier $h$ and subgroup $g \in \mathcal{G}$, let $F$ be a feature importance notion where $F(f_j, g(X^n), h)$ denote the \emph{importance  $h$ attributes to feature $j$ in the subgroup $g(X^n)$}, and $F(f_j, X^n, h)$ be the importance  $h$ attributes to $f_j$ on the entire dataset. We will provide more specific instantiations of $F$ shortly, but we state our definition of $\fid$ in the greatest possible generality below. 

\begin{definition}
\label{def:fid}
  (Feature Importance Disparity). Given a classifier $h$, a subgroup of $X^n$ defined by $g \in G$, and a feature $f_j \in [d]$, then given a feature importance notion $F(\cdot)$, the feature importance disparity relative to $g$ is defined as:
  $$\fid(f_j, g, h) = \mathbb{E}_{X \sim \mathcal{R}}\lvert F(f_j, g(X^n), h) - F(f_j, X^n, h) \rvert$$
\end{definition}

We will suppress $h$ and write $\fid(j, g)$ unless necessary to clarify the classifier we are using. Now, given $h$ and $X^n$, our goal is to find the feature subgroup pair $(j^*, g^{*}) \in [d] \times \mathcal{G}$ that maximizes $\fid(j, g)$, or $(j^*, g^{*}) = \argmax_{g \in \mathcal{G}, j \in [d]}\fid(j, g)$. 


We now get more concrete about our feature importance notion $F(\cdot)$. First, we define the class of \emph{separable} feature importance notions:

\begin{definition}
\label{def:sep}
    (Locally Separable). A feature importance notion $F(\cdot)$ is locally separable if it can be decomposed as a point wise sum of local model explanation values $F'$:
   \begin{equation*}
    F(f_j, X^n, h) = \sum_{X \in X^n} F'(f_j,X,h)
   \end{equation*}
\end{definition}

It follows that for separable notions, $F(f_j, g(X^n), h) = \sum_{X \in X^n} g(X)F'(f_j,X,h)$. Given a local model explanation $F'$, we can define a more specific form of $\fid$, the \emph{average feature importance disparity} ($\avgfid$), which compares the average feature importance within a subgroup to the average importance on the dataset.

\begin{definition}
\label{def:avg_sep}
    (Average Case Locally Separable $\fid$). For a $g \in \mathcal{G}$, let $|g| = \displaystyle\sum_{X \in X^n}g(X)$. Given a local model explanation $F'(\cdot)$, we define the corresponding: 
   \begin{equation*}
        \begin{split}
            \avgfid(f_j, g, h) = & \mathbb{E}_{X^n \sim \mathcal{R}^n}\lvert \frac{1}{|g|}\sum_{X \in X^n}g(X)F'(f_j, X, h) \\
            & - \frac{1}{n}\sum_{X \in X^n}F'(f_j, X, h) \rvert 
        \end{split}
   \end{equation*}
\end{definition} 

Note that $\avgfid$ is not equivalent to a separable $\fid$, since we divide by $|g|$, impacting every term in the summation. In Section~\ref{sec:sep}, we show that we can optimize for $\avgfid$ by optimizing a version of the $\fid$ problem with size constraints, which we can do efficiently via Algorithm~\ref{alg:cap}. 

This notion of $\emph{separability}$ is crucial to understanding the remainder of the paper. In Section~\ref{sec:sep}, we show that for any \emph{separable} $\fid$, Algorithm~\ref{alg:cap} is an (oracle) efficient way to compute the largest $\fid$ subgroup of a specified size in polynomial time. By ``oracle efficient,'' we follow \cite{msr, gerry} where we mean access to an optimization oracle that can solve (possibly NP-hard) problems. While this sounds like a strong assumption, in practice we can take advantage of modern optimization algorithms that can solve hard non-convex optimization problems (e.g. training neural networks). This framework has led to the development of many practical algorithms with a strong theoretical grounding \cite{msr, gerry, rich, calib}, and as shown in Section~\ref{sec:experiments} works well in practice here as well. The type of oracle we need is called a Cost Sensitive Classification (CSC) oracle, which we define in Appendix~\ref{sec:cscg_app}.

\section{Optimizing for $\avgfid$}
\label{sec:optimizing}

In this section, we show how to (oracle) efficiently compute the rich subgroup that maximizes the $\avgfid$. Rather than optimize $\avgfid$ directly, our Algorithm~\ref{alg:cap} solves an optimization problem that maximizes the $\fid$ subject to a group size constraint:

\begin{equation}
\label{eq:constr}
    \begin{split}
        & \quad \max_{g \in \mathcal{G}}|F(f_j,g(X^n),h) - F(f_j,X^n,h)| \\
        \textrm{s.t.} & \quad \Phi_L(g) \equiv \alpha_L -  \frac{1}{n}\sum_{X \in X^n}g(X) \leq 0, \\
        & \quad   \Phi_U(g) \equiv \frac{1}{n}\sum_{X \in X^n}g(X) - \alpha_U \leq 0
    \end{split}
\end{equation}

where $\Phi_L$ and $\Phi_U$ are "size violation" functions given a subgroup function $g$.
We denote the optimal solution to Equation~\ref{eq:constr} by $g^*_{[\alpha_L, \alpha_U]}$.
We focus on optimizing the constrained $\fid$ since the following primitive also allows us to efficiently optimize $\avgfid$:
\begin{enumerate}
\item Discretize $[0,1]$ into intervals $(\frac{i-1}{n}, \frac{i}{n}]_{i=1}^{n}$. Given feature $f_j$, compute $g^*_{(\frac{i-1}{n}, \frac{i}{n}]}$ for $i=1... n$.
\item Outputting $g_{k^*},$ where  $k^* = \argmax_{k}\frac{k}{n}|F(f_j, g_k, h)|$ approximately maximizes the $\avgfid$ given an appropriately large number of intervals $n$.
\end{enumerate}

Our proof for this is available in Appendix~\ref{sec:avgsepfidproof_app}. We now state our main theorem, which shows that we can solve the constrained $\fid$ problem in Equation~\ref{eq:constr} with polynomially many calls to $\text{CSC}_{\mathcal{G}}$.  

\label{sec:sep}

\begin{theorem}
\label{thm:sep}
Let $F$ be a separable $\fid$ notion, fix a classifier $h$, subgroup class $\mathcal{G}$, and oracle $\text{CSC}_{\mathcal{G}}$. Then choosing accuracy constant $\nu$ and bound constant $B$ and fixing a feature of interest $f_j$,  we will run Algorithm~\ref{alg:cap} twice; once with $\fid$ given by $F$, and once with $\fid$ given by $-F$. Let $\hat{p}_{\mathcal{G}}^{T}$ be the distribution returned after $T = O(\frac{4n^2 B^2}{\nu^2})$ iterations by Algorithm~\ref{alg:cap} that achieves the larger value of $ \; \mathbb{E}[\fid(j, g)]$. Then:  

\begin{equation}
\begin{aligned}
 & \fid(j, g_j^{*})-  \mathbb{E}_{g \sim \hat{p}_{\mathcal{G}}^{T}}[\fid(j, g)] \leq \nu \\
 & \quad \lvert \Phi_L(g) \rvert, \lvert \Phi_U(g) \rvert  \leq  \frac{1 + 2\nu}{B}
\end{aligned}
\end{equation}
\end{theorem}

We defer the proof of Theorem~\ref{thm:sep} to Appendix~\ref{sec:proof_app}. In summary, rather than optimizing over $g \in \mathcal{G}$, we optimize over distributions $\Delta(\mathcal{G})$. This allows us to cast the optimization problem in Equation~\ref{eq:constr} as a linear program so we can form the Lagrangian $L$, which is the sum of the feature importance values and the size constraint functions weighted by the dual variables $\lambda$, and apply strong duality. We can then cast the constrained optimization as computing the Nash equilibrium of a two-player zero-sum game, and apply the classical result of \cite{freund} which says that if both players implement $\emph{no-regret}$ strategies, then we converge to the Nash equilibrium at a rate given by the average regret of both players converging to zero. Algorithm~\ref{alg:cap} implements the no-regret algorithm exponentiated gradient descent \cite{expgrad} for the $\max$ player, who optimizes $\lambda$, and best-response via a CSC solve for the $\min$ player, who aims to maximize subgroup disparity to optimize the rich subgroup distribution.

We note that rather than computing the group $g$ that maximizes $\fid(j, g)$ subject to the size constraint, our algorithm outputs a distribution over groups $\hat{p}_{\mathcal{G}}^{T}$ that satisfies this process \emph{on average} over the groups. In theory, this seems like a drawback for interpretability. However, in practice we simply take the groups $g_t$ found at each round and output the ones that are in the appropriate size range, and have largest $\fid$ values. The results in Section~\ref{sec:experiments} validate that this heuristic choice is able to find groups that are both feasible and have large $\fid$ values. This method also generalizes out of sample showing that the $\fid$ is not artificially inflated by multiple testing (Appendix~\ref{subsec:valid}). Moreover, our method provides a menu of potential groups $(g_t)_{t=1}^{T}$ that can be quickly evaluated for large $\fid$, which can be a useful feature to find interesting biases not present in the maximal subgroup.
\section{Experiments}
\label{sec:experiments}
Here we report the results of our empirical investigation across $16$ different dataset/$\fid$-notion pairings. These results confirm that our method can find large $\avgfid$ values corresponding to rich subgroups defined as simple functions of protected attributes (Table~\ref{tab:summary}, Figures~\ref{fig:all_de}, \ref{fig:vid_dist}), and are larger than those found by optimizing over marginal subgroups alone (Section~\ref{sec:rich_marg_main}). Moreover, in Section~\ref{subsec:fairness} we find that high $\avgfid$ subgroups tend to have significant disparities in traditional fairness metrics (Table~\ref{tab:fair_compare}), and that rich subgroups that maximize a fairness notion like FPR disparity also express high $\avgfid$ features, albeit smaller than those found by Algorithm~\ref{alg:cap} (Table~\ref{tab:gerry_compare}). Perhaps most significantly, for a tool designed to assist the process of $\dbi$, is that the high $\avgfid$ subgroups found correspond to subgroups and features that are suggestive of potential dataset bias, one example of which we covered in the initial case study and provide further examples of in Section~\ref{sec:key_sub_app}. In Appendix~\ref{sec:synthetic_experiment}, we construct two synthetic datasets, one where the feature importance is the same across all subgroups, and one where there is a deliberately introduced disparity in a specific rich subgroup, in order to verify that our methods i) avoid false discovery and ii) can effectively pick out a high-$\fid$ subgroup. We also include results showing that the rich subgroups found generalize out of sample both in terms of $\avgfid$ and group size $|g|$ (Appendix~\ref{subsec:valid}), are relatively robust to the choice of hypothesis class of $h$ (Appendix~\ref{sec:hypothesis_app}), and that our algorithms converge quickly in practice (Appendix~\ref{sec:conv_app}). The code used for our experiments is available at \href{https://github.com/safr-ai-lab/xai-disparity}{github.com/safr-ai-lab/xai-disparity}.

\subsection{Experimental Details}
\label{subsec:exp_details}

\textbf{Datasets}: We used four popular datasets for the experiments: Student, COMPAS, Bank, and Folktables. For each test, we used COMPAS twice, once predicting two-year recidivism and once predicting decile risk score (labeled COMPAS R and COMPAS D respectively). For each dataset, we specified "sensitive" features which are features generally covered by equal protection or privacy laws (e.g. race, gender, age, health data). Appendix~\ref{sec:details_app} contains more details.

\textbf{Computing the $\avgfid$}: We study $3$ separable notions of $\fid$ based on local model explanations Local-Interpretable, Model-Agnostic (LIME) \cite{lime}, Shapley Additive Explanations (SHAP) \cite{shap}, and saliency maps (GRAD) \cite{saliency}. For every method and dataset, we optimize the constrained $\fid$ over $\alpha$ ranges $(\alpha_L, \alpha_U) = \{ [.01,.05], [.05,.1], [.1,.15], [.15,.2], [.2,.25] \}$. These small ranges allowed us to reasonably compare the $\fid$ values, reported in Table~\ref{tab:summary}. Additionally, these ranges span subgroup sizes that may be of particular interest in fairness research and dataset auditing work. All values of $\avgfid$ reported in the results are \emph{out of sample}; i.e. the $\avgfid$ values are computed on a test set that was not used to optimize the subgroups. Datasets were split into $80-20$ train-test split except for Student which was split $50-50$ due to its small size. Across all datasets, when the $\fid$ was LIME or SHAP, we set $h$ to be a random forest, when it was GRAD we used logistic regression as it requires a classifier whose outputs are differentiable in the inputs. The exact choice of classifier does not have any notable impact on the outcomes as we discuss in Appendix~\ref{sec:hypothesis_app}. Due to computation constraints, GRAD was only tested on the COMPAS R dataset. We defer the  details in how we implemented the importance notions and Algorithm~\ref{alg:cap} to Appendix~\ref{sec:details_app}. 

\textbf{Linear Feature Importance Disparity}: In addition to the $3$ separable notions of $\fid$, we also studied an approach for a non-separable notion of importance. Linear regression (labeled LR in results) is a popular model that is inherently interpretable; the coefficients of a weighted least squares (WLS) solution represent the importance of each feature. We can thus define another variant of $\fid$, the linear feature importance disparity ($\linfid$), as the difference in the WLS coefficient of feature $f_j$ on subgroup $g$ and on the dataset $X^n$. As $\linfid$ is differentiable with respect to $g$, we are able to find a locally optimal $g$ with high $\linfid$ using a non-convex optimizer; we used ADAM. For details and proofs, see Appendix~\ref{sec:lin_fid}.

\subsection{Experimental Results}
\label{subsec:exp_summary}
Table~\ref{tab:summary} summarizes the results of the experiments, which are visualized in Figure~\ref{fig:all_de} on a log-ratio scale for better cross-notion comparison. Across each dataset and importance notion, our methods were able to find subgroups with high $\fid$, often differing by orders of magnitude. For example, on Folktables with LIME as the importance notion, there is a subgroup on which \texttt{age} is on average $225$ times more important than it is for the whole population. Table~\ref{tab:summary} also provides the defining features, listed as the sensitive features which have the largest coefficients in $g$.

\begin{figure*}[h]
  \centering
  \includegraphics[width=.65\linewidth]{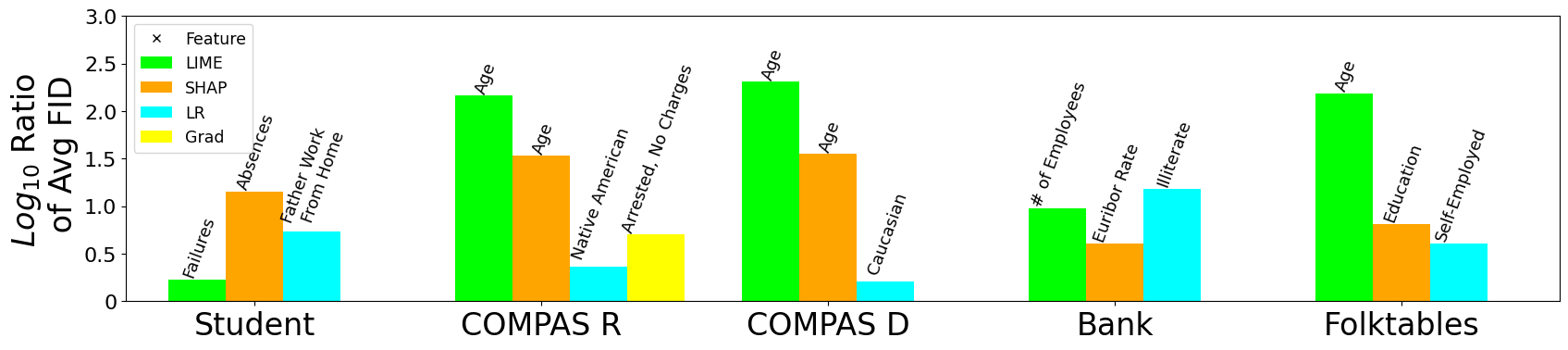}
  \caption{Summary of the highest $\fid$s found for each (dataset, method). This is displayed as $\big | log_{10}(R) \big |$ where $R$ is the ratio of average importance per data point in $g^*$ to the average importance on $X$ for separable notions, or the ratio of coefficients for $\linfid$. This scale allows comparison across different importance notions. The feature associated with each $g^*$ is written above the bar.}
  \label{fig:all_de}
\end{figure*}

\begin{table*}[t]
\centering
  \caption{Summary of the subgroup with highest $\avgfid$ for each experiment along with the corresponding feature, subgroup size, and defining features. Experiments were run across multiple $(\alpha_L, \alpha_U)$ ranges with the highest $\avgfid$ found being displayed. $\mu(F)$ is the average feature importance value on the specified group.}
  \label{tab:summary}
  \scalebox{.8}{
  \begin{tabular}{lllcccl}
    \toprule
    Dataset & Notion & Feature $f_j$ & $\mu(F(f_j, X))$ & $\mu(F(f_j, g))$ & $|g|$ & Defining Features\\
    \midrule
    Student & LIME & Failures & $-.006$ & $-.011$ & $.01$ & Alcohol Use, Urban Home\\
     & SHAP & Absences & $-.15$ & $-2.1$ & $.02$ & Parent Status, Urban Home\\
     & LR & Father WFH & $21.7$ & $-4.0$ & $.03$ & Alcohol Use, Health\\
    COMPAS R & LIME & Age & $.0009$ & $-.14$ & $.05$ & Native-American\\
     & SHAP & Age & $.012$ & $.41$ & $.04$ & Asian-American\\
     & LR & Native American & $.5$ & $1.17$ & $.04$ & Asian/Hispanic-American\\
     & GRAD & Arrest, No Charge & $.09$ & $.02$ & $.05$ & Native-American\\
    COMPAS D & LIME & Age & $-.0003$ & $-.06$ & $.02$ & Native/Black-American\\
     & SHAP & Age & $.06$ & $2.35$ & $.07$ & Black/Asian-American\\
     & LR & Caucasian & $6.7$ & $10.7$ & $.04$ & Native-American\\
    Bank & LIME & \# of Employees & $-.003$ & $.03$ & $.03$ & Marital Status\\
     & SHAP & Euribor Rate & $-.004$ & $.016$ & $.03$ & Marital Status\\
     & LR & Illiterate & $-.07$ & $-.0045$ & $.01$ & Age, Marital Status\\
    Folktables & LIME & Age & $-.0007$ & $-.11$ & $.21$ & Marital Status\\
     & SHAP & Education & $.023$ & $.15$ & $.03$ & Asian-American\\
     & LR & Self-Employed & $-.26$ & $-.06$ & $.02$ & White-American\\
  \bottomrule
\end{tabular}}
\end{table*}

A natural follow up question that arises from this experiment is what does the distribution of $\fid$s look like for a given dataset? Figure~\ref{fig:vid_dist} shows a distribution of the $10$ features on the Bank dataset with the highest $\fid$ values. As we can see, there are a few features where large $\fid$ subgroups can be found, but it tails off significantly. This pattern is replicated across all datasets and feature importance notions. This is a positive result for practical uses, as an analyst or domain expert can focus on a handful of features that perform drastically differently when performing a $\dbi$.

\begin{figure}[h]
  \centering
  \includegraphics[width=.85\linewidth]{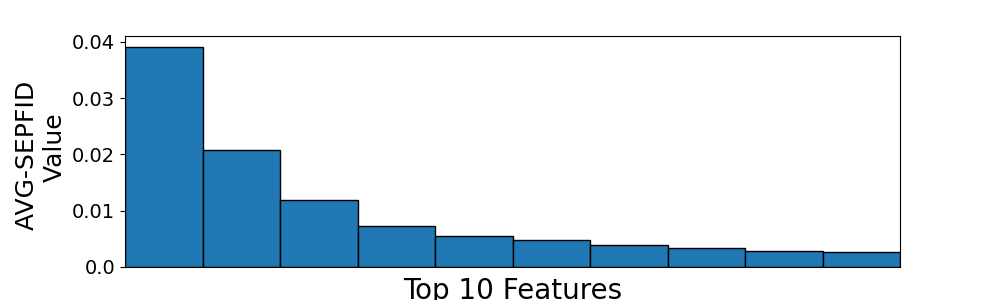}
  \caption{Distribution of $\avgfid$ on the top features from the BANK dataset using LIME. We see a sharp drop off in $\avgfid$. This pattern is seen in all datasets and notions.}
  \label{fig:vid_dist}
\end{figure}

Earlier in Section~\ref{sec:case_study}, we examined a specific case where the $\fid$ found as a result of our method revealed a biased subgroup where the fairness disparities could then be mitigated with targeted approaches. In the next section, across every dataset and feature importance notion, we find similar examples exposing some form of potential bias. We note that not every single (subgroup, feature) pair discovered necessarily implies a fairness concern. For example, $\avgfid$ could be driven in part by correlations between $f_j$ and the sensitive attributes that define $g$. Since it remains true in all fairness work that two contexts that are statistically equivalent may have very different fairness implications in the real world, our method should be viewed as a tool to aid practitioners in \texttt{DBI} rather than as definitive proof of bias.

\subsection{Discussion of High $\fid$ Subgroups}
\label{sec:key_sub_app}


\begin{figure}[b!]
    \centering
    \begin{subfigure}{.49\textwidth}
        \centering
        \includegraphics[width=\linewidth]{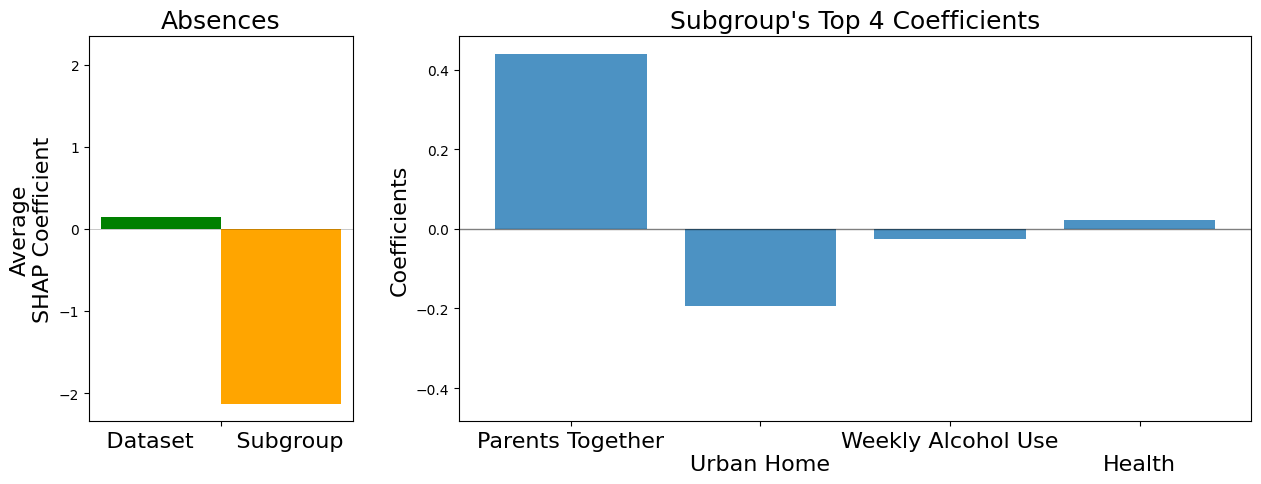}
        \caption{Student: Predicting grade outcomes}
        \label{fig:fid_sub1}
    \end{subfigure}
    \begin{subfigure}{.49\textwidth}
        \centering
        \includegraphics[width=\linewidth]{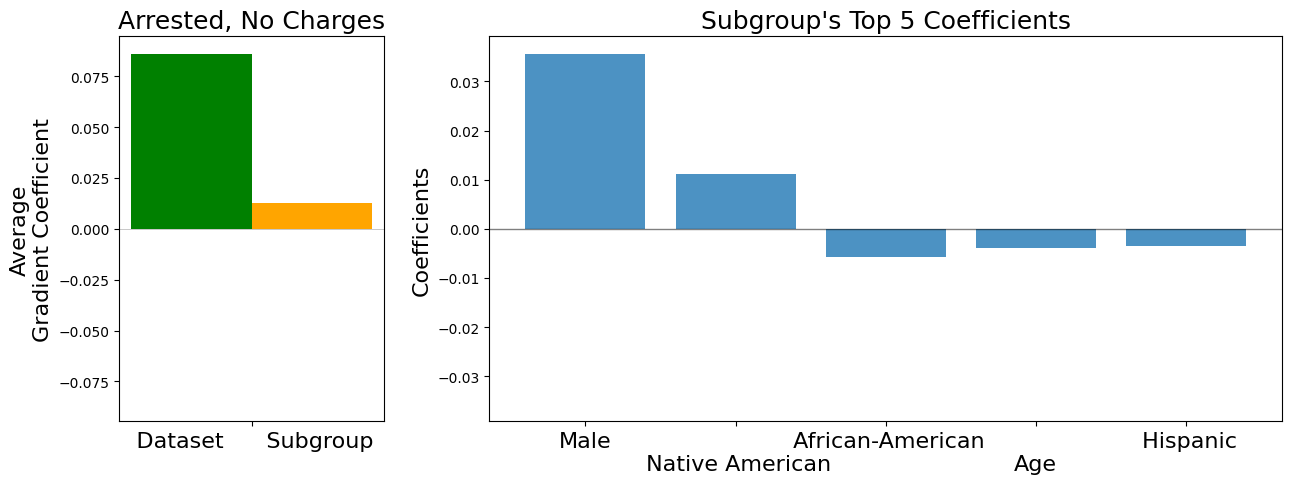}
        \caption{COMPAS: Predicting 2-year recidivism}
        \label{fig:fid_sub2}
    \end{subfigure}
    \vskip .2in
    \begin{subfigure}{.49\textwidth}
        \centering
        \includegraphics[width=\linewidth]{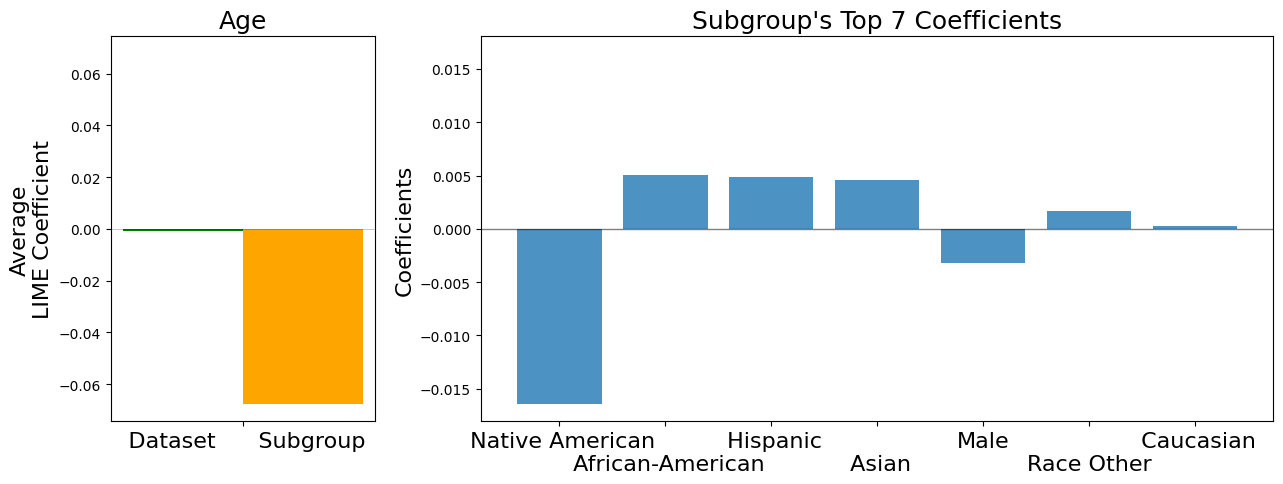}
        \caption{COMPAS: Predicting decile risk score}
        \label{fig:fid_sub3}
    \end{subfigure}
    \begin{subfigure}{.49\textwidth}
        \centering
        \includegraphics[width=\linewidth]{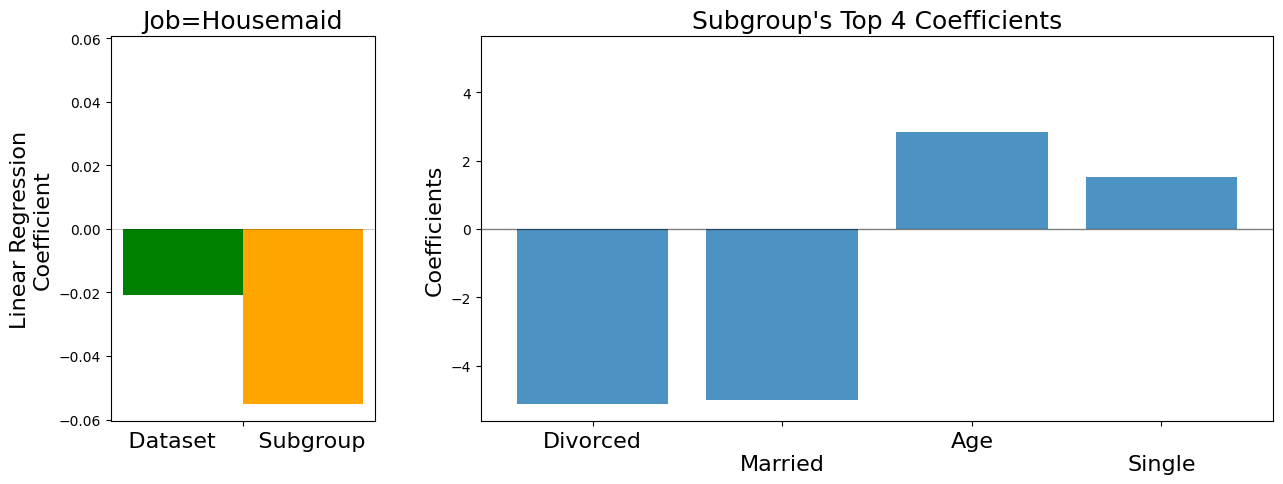}
        \caption{Bank: Predicting whether bank deposit is made}
        \label{fig:fid_sub4}
    \end{subfigure}
    \vskip .2in
    \begin{subfigure}{.49\textwidth}
        \centering
        \includegraphics[width=\linewidth]{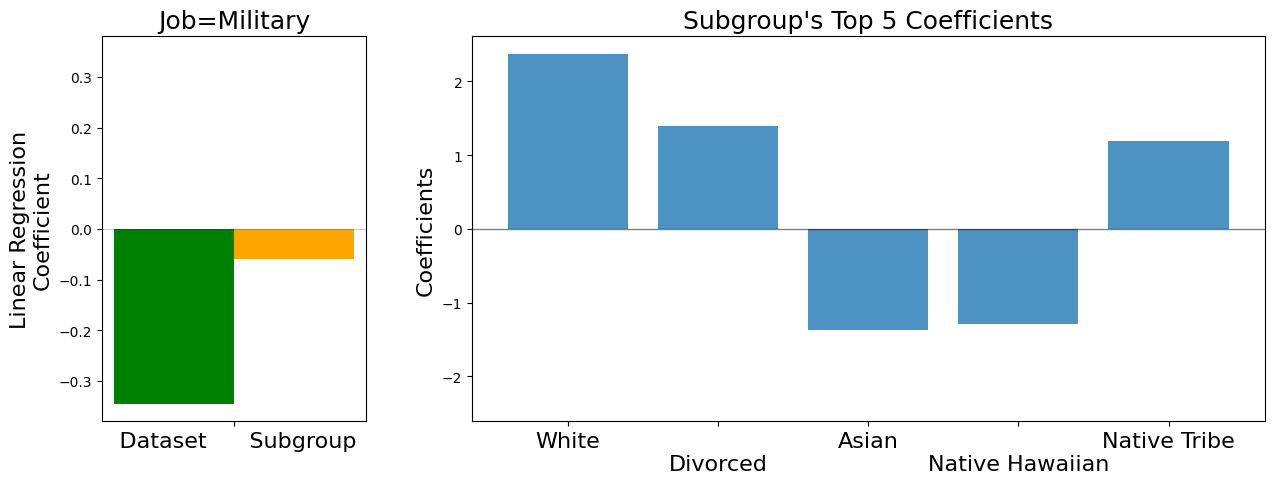}
        \caption{Folktables: Predicting income >\$50k}
        \label{fig:fid_sub5}
    \end{subfigure}
\caption{Exploration of key subgroup/feature pairs found for each dataset. The first graph shows the change in feature importance from whole dataset to subgroup. The second graph shows the main coefficients that define the subgroup.}
\label{fig:key_groups}
\end{figure}
In Figure~\ref{fig:key_groups}, we highlight selections of an interesting (feature, subgroup, method) pair for each dataset. Figure~\ref{fig:fid_sub1} shows that on the Student dataset the feature \texttt{absences} which is of near zero importance on the dataset as a whole, is very negatively correlated with student performance on a subgroup whose top $2$ features indicate whether a student's parents are together, and if they live in an rural neighborhood. Figure~\ref{fig:fid_sub2} shows that on the COMPAS dataset with method GRAD, the feature \texttt{arrested-but-with-no-charges} is typically highly important when predicting \texttt{two-year-recidivism}. However, it carries significantly less importance on a subgroup that is largely defined as Native American males. When predicting the decile risk score on COMPAS, LIME indicates that age is not important on the dataset as a whole; however, for non-Native American, female minorities, older age can be used to explain a lower \texttt{Decile Score}. On the Bank dataset using $\linfid$, we see that a linear regression trained on points from a subgroup defined by older, single individuals, puts more importance on \texttt{job=housemaid} when predicting likelihood in signing up for an account. Finally on Folktables, we see that $\linfid$ assigns much lower weight to the \texttt{job=military} feature among a subgroup that is mainly white and divorced people than in the overall dataset when predicting income. These interesting examples, in conjunction with the results reported in Table~\ref{tab:summary}, highlight the usefulness of our method in finding subgroups where a concerned analyst or domain expert could dig deeper to determine how biases might be manifesting themselves in the data and how to correct for them. 

\subsection{Comparison of $\fid$ Values on Rich vs. Marginal Subgroups}
\label{sec:rich_marg_main}

To better quantify the advantage of rich subgroups, we performed the same analysis but only searching over the marginal subgroup space. For each dataset and importance notion pair, we established the finite list of subgroups defined by a single sensitive characteristic and computed the $\fid$ for each of these subgroups. In Figure~\ref{fig:rich_marg_main}, we compare some of the maximal $\avgfid$ rich subgroups shown in Figure~\ref{fig:all_de} to the maximal $\avgfid$ marginal subgroup for the same feature. In about half of the cases, the $\avgfid$ of the marginal subgroup was similar to the rich subgroup. In the other cases, expanding our subgroup classes to include rich subgroups defined by linear functions of the sensitive attributes enabled us to find a subgroup that had a higher $\avgfid$. For example, in Figure~\ref{fig:rich_marg_compas}, we can see that on the COMPAS R dataset using GRAD as the importance notion, \texttt{Arrested, No Charges} had a rich subgroup with $\avgfid$ that was $4$ times less than on the full dataset. However, we were unable to find any subgroup in the marginal space where the importance of the feature was nearly as different (comparisons for all datasets are available in Appendix~\ref{sec:rich_marg}). In some cases, the marginal subgroup performs slightly better than the rich subgroup (Figure~\ref{fig:rich_marg}). This happens when using rich subgroups does not offer any substantial advantage over marginal subgroups, and the empirical error tolerance in Algorithm~\ref{alg:cap} stops the convergence early.

Perhaps an even more important practical advantage of Algorithm~\ref{alg:cap}'s ability to optimize over rich subgroups, is that it allows protected subgroups to be defined as functions of continuous variables. For example, \texttt{age} is easily included in our formulation, while capturing \texttt{age} with marginal subgroups requires first bucketing into age groups and then one-hot encoding these groups, which comes with statistical, explanatory, and computational drawbacks. As mentioned in Section~\ref{sec:preliminaries}, our framework also allows individuals to be part of multiple groups, for a example a multiracial individual who might be better represented as a fractional member of different racial groups, rather than a member of a single discrete one. This kind of data would be impossible to capture with marginal subgroups.

\begin{figure*}
    \centering
    \begin{subfigure}{.22\textwidth}
        \centering
        \includegraphics[width=\linewidth]{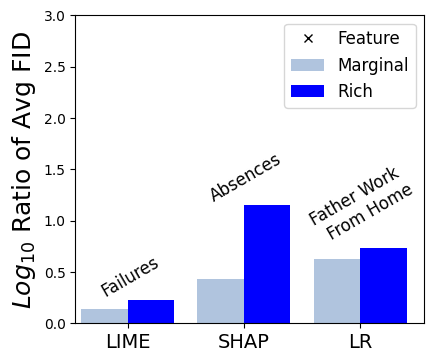}
        \caption{Student}
        \label{fig:rich_marg_student}
    \end{subfigure}
    \begin{subfigure}{.27\textwidth}
        \centering
        \includegraphics[width=\linewidth]{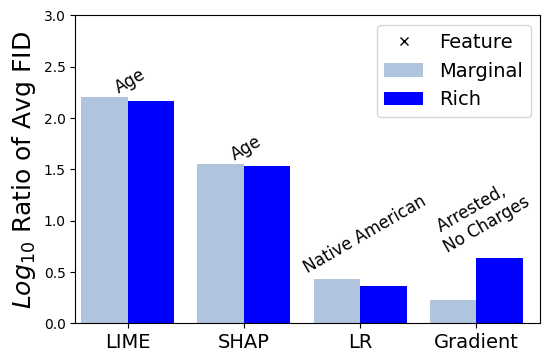}
        \caption{COMPAS R}
        \label{fig:rich_marg_compas}
    \end{subfigure}
\caption{Comparisons of some maximal $\fid$ rich subgroups to the maximal $\fid$ marginal subgroup on the same feature using the same log-scale as in Figure~\ref{fig:all_de}. The feature associated with the subgroups is written above each bar.}
\label{fig:rich_marg_main}
\end{figure*}

\begin{table*}[h]
\centering
  \caption{Fairness metrics of high $\avgfid$ subgroups. COMPAS D and Student were excluded since they use non-binary $y$, making classification metrics less comparable. We measured the $3$ fairness types outlined by \cite{fairml}: $P(\hat{Y}=1)$, true/false positive rates, and expected calibration error. $\textit{<metric>}_{\Delta}$ is the metric on $g$ minus the metric on $X$.}
  \label{tab:fair_compare}
  \scalebox{.8}{
  \begin{tabular}{llllcccc}
    \toprule
    Dataset & Notion & $F$ & Defining Features of $g$ & $\hat{Y}_{\Delta}$ & $TPR_{\Delta}$ & $FPR_{\Delta}$ & $ECE_{\Delta}$ \\
    \midrule
    COMPAS R & LIME & Age & Native-American & $-.16$ & $-.2$ & $-.07$ & $.24$ \\
     & SHAP & Age & Asian-American & $.37$ & $.26$ & $.4$ & $-.12$ \\
     & GRAD & Arrest, No Charge & Native-American & $-.24$ & $-.35$ & $-.12$ & $.37$ \\
    Bank & LIME & \# of Employees & Marital Status & $.11$ & $.08$ & $.07$ & $-.15$ \\
     & SHAP & Euribor Rate & Marital Status & $.11$ & $.08$ & $.07$ & $-.17$ \\
    Folktables & LIME & Age & Marital Status & $-.15$ & $-.09$ & $-.07$ & $.19$ \\
     & SHAP & Education & Asian-American & $.17$ & $.2$ & $.05$ & $-.09$ \\
  \bottomrule
\end{tabular}}
\end{table*}

\begin{table*}[!h]
\centering
  \caption{Comparing top features and respective $\avgfid$ of $g$ found via our method ($\avgfid_{FID}$) and found by \cite{gerry} ($\avgfid_{gerry}$). As in Table~\ref{tab:fair_compare}, COMPAS D and Student were excluded.}
  \label{tab:gerry_compare}
  \scalebox{.8}{
  \begin{tabular}{lllclc}
    \toprule
    Dataset & Notion & $F_{FID}$ & $\avgfid_{FID}$ & $F_{gerry}$ & $\avgfid_{gerry}$ \\
    \midrule
    COMPAS R & LIME & Age & $.14$ & Age & $.04$ \\
     & SHAP & Age & $.4$ & Age & $.06$\\
     & GRAD & Arrest, No Charge & $.09$ & Male & $.03$ \\
    Bank & LIME & \# of Employees & $.03$ & \# of Employees & $.008$ \\
     & SHAP & Euribor Rate & $.016$ & Emp Var Rate & $.004$ \\
    Folktables & LIME & Age & $.11$ & Age & $.05$ \\
     & SHAP & Education & $.13$ & Age & $.05$ \\
  \bottomrule
\end{tabular}}
\end{table*}

\subsection{Fairness Metrics}
\label{subsec:fairness}
While large $\avgfid$ values with respect to a given feature and importance notion do not guarantee disparities in common fairness metrics, which are not typically defined in terms of a specific reference feature, it is natural to ask if these notions are correlated: do subgroups with large $\avgfid$ have large disparities in fairness metrics, and do subgroups that have large disparities in fairness metrics have particularly large $\avgfid$ values for some feature?

We examine the first question in Table~\ref{tab:fair_compare}. We find that these high $\avgfid$ subgroups tend to have significant disparities in traditional fairness metrics. Although the metrics are not always \emph{worse} on $g$, this reinforces the intuition that subgroups with high $\avgfid$ require greater scrutiny. In Table~\ref{tab:gerry_compare}, we study the converse question, where we use the GerryFair code of \cite{gerry} to find rich subgroups that maximally violate FPR disparity, and then compute the $\avgfid$ on those subgroups. We find that they also have features with high $\avgfid$, albeit not as large as those found by Algorithm~\ref{alg:cap}, which explicitly optimizes for $\avgfid$. These two results highlight the usefulness of our method in identifying potentially high risk subgroups.


\section{Discussion}
In this paper we establish feature importance disparities as an important tool to aid in $\texttt{DBI}$s. One benefit to our work, is that as progress is made in feature importance methods, they can be leveraged in our Algorithm~\ref{alg:cap} (if they are separable). As we discuss, given observed feature/subgroup disparities, there are many possible next steps an analyst could take to root out dataset bias, from interventions in the training or feature selection process, to more involved investigations into how the data was collected. This perspective is complementary to recent work in the fairness literature that focuses on understanding and documenting data sources \cite{datasheets, fair_data}. On the algorithms side, there has been little work to systematically identify the sources of bias in the data, as opposed to developing methods that remove bias during model training or representation learning. This paper represents one attempt at developing this kind of ``data-centric'' fairness method through the lens of feature importance. There is much more work that can be done to determine what the right subset of features to collect are, how to detect and mitigate biases in the data collection process, and which subsets of data or features should be retained prior to model training.


\section*{Acknowledgements}

The authors would like to thank Cynthia Dwork for the initial conversations that motivated the problem direction.



\section*{Impact Statement}
This paper presents work whose goal is to empower practitioners to investigate potential sources of bias in common datasets used for machine learning. As such, this work has the potential to contribute to fairer and more democratized algorithmic decision making across many application domains. We emphasize, as is a common refrain in the related body of literature, that no method is a silver bullet towards mitigating bias in ML systems. 

\clearpage
\bibliography{references}

\begin{thebibliography}{64}
\providecommand{\natexlab}[1]{#1}
\providecommand{\url}[1]{\texttt{#1}}
\expandafter\ifx\csname urlstyle\endcsname\relax
  \providecommand{\doi}[1]{doi: #1}\else
  \providecommand{\doi}{doi: \begingroup \urlstyle{rm}\Url}\fi

\bibitem[Agarwal et~al.(2018{\natexlab{a}})Agarwal, Beygelzimer, Dudik, Langford, and Wallach]{msr}
Agarwal, A., Beygelzimer, A., Dudik, M., Langford, J., and Wallach, H.
\newblock A reductions approach to fair classification.
\newblock In Dy, J. and Krause, A. (eds.), \emph{Proceedings of the 35th International Conference on Machine Learning}, volume~80 of \emph{Proceedings of Machine Learning Research}, pp.\  60--69. PMLR, 10--15 Jul 2018{\natexlab{a}}.
\newblock URL \url{https://proceedings.mlr.press/v80/agarwal18a.html}.

\bibitem[Agarwal et~al.(2018{\natexlab{b}})Agarwal, Beygelzimer, Dud{\'{\i}}k, Langford, and Wallach]{red_fair}
Agarwal, A., Beygelzimer, A., Dud{\'{\i}}k, M., Langford, J., and Wallach, H.~M.
\newblock A reductions approach to fair classification.
\newblock In Dy, J.~G. and Krause, A. (eds.), \emph{Proceedings of the 35th International Conference on Machine Learning, {ICML} 2018, Stockholmsm{\"{a}}ssan, Stockholm, Sweden, July 10-15, 2018}, volume~80 of \emph{Proceedings of Machine Learning Research}, pp.\  60--69. {PMLR}, 2018{\natexlab{b}}.
\newblock URL \url{http://proceedings.mlr.press/v80/agarwal18a.html}.

\bibitem[Agarwal et~al.(2022{\natexlab{a}})Agarwal, Johnson, Pawelczyk, Krishna, Saxena, Zitnik, and Lakkaraju]{agarwal2022rethinking}
Agarwal, C., Johnson, N., Pawelczyk, M., Krishna, S., Saxena, E., Zitnik, M., and Lakkaraju, H.
\newblock Rethinking stability for attribution-based explanations, 2022{\natexlab{a}}.

\bibitem[Agarwal et~al.(2022{\natexlab{b}})Agarwal, Krishna, Saxena, Pawelczyk, Johnson, Puri, Zitnik, and Lakkaraju]{agarwal2022openxai}
Agarwal, C., Krishna, S., Saxena, E., Pawelczyk, M., Johnson, N., Puri, I., Zitnik, M., and Lakkaraju, H.
\newblock Openxai: Towards a transparent evaluation of model explanations.
\newblock In Koyejo, S., Mohamed, S., Agarwal, A., Belgrave, D., Cho, K., and Oh, A. (eds.), \emph{Advances in Neural Information Processing Systems}, volume~35, pp.\  15784--15799. Curran Associates, Inc., 2022{\natexlab{b}}.
\newblock URL \url{https://proceedings.neurips.cc/paper_files/paper/2022/file/65398a0eba88c9b4a1c38ae405b125ef-Paper-Datasets_and_Benchmarks.pdf}.

\bibitem[Alvarez-Melis \& Jaakkola(2018)Alvarez-Melis and Jaakkola]{alvarezmelis2018robustness}
Alvarez-Melis, D. and Jaakkola, T.~S.
\newblock On the robustness of interpretability methods, 2018.

\bibitem[Angwin et~al.(2016)Angwin, Larson, Mattu, and Kirchner]{angwin}
Angwin, J., Larson, J., Mattu, S., and Kirchner, L.
\newblock Machine bias: There’s software used across the country to predict future criminals. and it’s biased against blacks., 2016.
\newblock URL \url{https://www.propublica.org/article/machine-bias-risk-assessments-in-criminal-sentencing}.
\newblock [accessed: 15.11.2020].

\bibitem[Baehrens et~al.(2010)Baehrens, Schroeter, Harmeling, Kawanabe, Hansen, and M{{\"u}}ller]{saliency3}
Baehrens, D., Schroeter, T., Harmeling, S., Kawanabe, M., Hansen, K., and M{{\"u}}ller, K.-R.
\newblock How to explain individual classification decisions.
\newblock \emph{Journal of Machine Learning Research}, 11\penalty0 (61):\penalty0 1803--1831, 2010.
\newblock URL \url{http://jmlr.org/papers/v11/baehrens10a.html}.

\bibitem[Balagopalan et~al.(2022)Balagopalan, Zhang, Hamidieh, Hartvigsen, Rudzicz, and Ghassemi]{balagopalan2022explainability}
Balagopalan, A., Zhang, H., Hamidieh, K., Hartvigsen, T., Rudzicz, F., and Ghassemi, M.
\newblock The road to explainability is paved with bias: Measuring the fairness of explanations.
\newblock In \emph{Proceedings of the 2022 ACM Conference on Fairness, Accountability, and Transparency}, FAccT '22, pp.\  1194–1206, New York, NY, USA, 2022. Association for Computing Machinery.
\newblock ISBN 9781450393522.
\newblock \doi{10.1145/3531146.3533179}.
\newblock URL \url{https://doi.org/10.1145/3531146.3533179}.

\bibitem[Bansal et~al.(2020)Bansal, Agarwal, and Nguyen]{bansal2020sam}
Bansal, N., Agarwal, C., and Nguyen, A.
\newblock Sam: The sensitivity of attribution methods to hyperparameters, 2020.

\bibitem[Barocas et~al.(2019)Barocas, Hardt, and Narayanan]{fairml}
Barocas, S., Hardt, M., and Narayanan, A.
\newblock \emph{Fairness and Machine Learning: Limitations and Opportunities}.
\newblock fairmlbook.org, 2019.
\newblock \url{http://www.fairmlbook.org}.

\bibitem[Begley et~al.(2020)Begley, Schwedes, Frye, and Feige]{exp_fair}
Begley, T., Schwedes, T., Frye, C., and Feige, I.
\newblock Explainability for fair machine learning.
\newblock \emph{CoRR}, abs/2010.07389, 2020.
\newblock URL \url{https://arxiv.org/abs/2010.07389}.

\bibitem[Bolukbasi et~al.(2016)Bolukbasi, Chang, Zou, Saligrama, and Kalai]{home}
Bolukbasi, T., Chang, K.-W., Zou, J., Saligrama, V., and Kalai, A.
\newblock Man is to computer programmer as woman is to homemaker? debiasing word embeddings.
\newblock In \emph{Proceedings of the 30th International Conference on Neural Information Processing Systems}, NIPS'16, pp.\  4356–4364, Red Hook, NY, USA, 2016. Curran Associates Inc.
\newblock ISBN 9781510838819.

\bibitem[Boxer et~al.(2023)Boxer, au2, and Neill]{boxer2023auditing}
Boxer, K.~S., au2, E. M.~I., and Neill, D.~B.
\newblock Auditing predictive models for intersectional biases, 2023.

\bibitem[Buolamwini \& Gebru(2018)Buolamwini and Gebru]{gendershades}
Buolamwini, J. and Gebru, T.
\newblock Gender shades: Intersectional accuracy disparities in commercial gender classification.
\newblock In Friedler, S.~A. and Wilson, C. (eds.), \emph{Proceedings of the 1st Conference on Fairness, Accountability and Transparency}, volume~81 of \emph{Proceedings of Machine Learning Research}, pp.\  77--91. PMLR, 23--24 Feb 2018.
\newblock URL \url{https://proceedings.mlr.press/v81/buolamwini18a.html}.

\bibitem[Bureau(2020)]{pums}
Bureau, U.~C.
\newblock 2014-2018 acs pums data dictionary, 2020.
\newblock URL \url{https://www2.census.gov/programs-surveys/acs/tech_docs/pums/data_dict/PUMS_Data_Dictionary_2014-2018.pdf}.

\bibitem[Caton \& Haas(2023)Caton and Haas]{fair_ml_survey}
Caton, S. and Haas, C.
\newblock Fairness in machine learning: A survey.
\newblock \emph{ACM Comput. Surv.}, aug 2023.
\newblock ISSN 0360-0300.
\newblock \doi{10.1145/3616865}.
\newblock URL \url{https://doi.org/10.1145/3616865}.
\newblock Just Accepted.

\bibitem[Chouldechova(2017)]{chouldechova}
Chouldechova, A.
\newblock Fair prediction with disparate impact: A study of bias in recidivism prediction instruments.
\newblock \emph{Big Data}, 5\penalty0 (2):\penalty0 153--163, 2017.
\newblock \doi{10.1089/big.2016.0047}.
\newblock URL \url{https://doi.org/10.1089/big.2016.0047}.
\newblock PMID: 28632438.

\bibitem[Cortez \& Silva(2008)Cortez and Silva]{student}
Cortez, P. and Silva, A.
\newblock Using data mining to predict secondary school student performance.
\newblock In \emph{Proceedings of 5th FUture BUsiness TEChnology Conference}, pp.\  5--12, 2008.
\newblock URL \url{http://www3.dsi.uminho.pt/pcortez/student.pdf}.

\bibitem[Dai et~al.(2022)Dai, Upadhyay, Aivodji, Bach, and Lakkaraju]{Dai_2022}
Dai, J., Upadhyay, S., Aivodji, U., Bach, S.~H., and Lakkaraju, H.
\newblock Fairness via explanation quality.
\newblock In \emph{Proceedings of the 2022 {AAAI}/{ACM} Conference on {AI}, Ethics, and Society}. {ACM}, jul 2022.
\newblock \doi{10.1145/3514094.3534159}.
\newblock URL \url{https://doi.org/10.1145\%2F3514094.3534159}.

\bibitem[Dastile et~al.(2020)Dastile, Celik, and Potsane]{dastille}
Dastile, X., Celik, T., and Potsane, M.
\newblock Statistical and machine learning models in credit scoring: A systematic literature survey.
\newblock \emph{Applied Soft Computing}, 91:\penalty0 106263, 2020.
\newblock ISSN 1568-4946.
\newblock \doi{https://doi.org/10.1016/j.asoc.2020.106263}.
\newblock URL \url{https://www.sciencedirect.com/science/article/pii/S1568494620302039}.

\bibitem[Dimanov et~al.(2020)Dimanov, Bhatt, Jamnik, and Weller]{Dimanov2020YouST}
Dimanov, B., Bhatt, U., Jamnik, M., and Weller, A.
\newblock You shouldn't trust me: Learning models which conceal unfairness from multiple explanation methods.
\newblock In \emph{SafeAI@AAAI}, 2020.
\newblock URL \url{https://api.semanticscholar.org/CorpusID:211838210}.

\bibitem[Ding et~al.(2021)Ding, Hardt, Miller, and Schmidt]{folktables}
Ding, F., Hardt, M., Miller, J., and Schmidt, L.
\newblock Retiring adult: New datasets for fair machine learning.
\newblock \emph{Advances in Neural Information Processing Systems 34 (NeurIPS)}, 2021.
\newblock URL \url{https://par.nsf.gov/biblio/10341458}.

\bibitem[Dwork \& Ilvento(2018)Dwork and Ilvento]{dwork_comp}
Dwork, C. and Ilvento, C.
\newblock Fairness under composition.
\newblock \emph{CoRR}, abs/1806.06122, 2018.
\newblock URL \url{http://arxiv.org/abs/1806.06122}.

\bibitem[Dwork et~al.(2012)Dwork, Hardt, Pitassi, Reingold, and Zemel]{fairness_through_awareness}
Dwork, C., Hardt, M., Pitassi, T., Reingold, O., and Zemel, R.
\newblock Fairness through awareness.
\newblock In \emph{Proceedings of the 3rd Innovations in Theoretical Computer Science Conference}, ITCS '12, pp.\  214–226, New York, NY, USA, 2012. Association for Computing Machinery.
\newblock ISBN 9781450311151.
\newblock \doi{10.1145/2090236.2090255}.
\newblock URL \url{https://doi.org/10.1145/2090236.2090255}.

\bibitem[Fabris et~al.(2022)Fabris, Messina, Silvello, and Susto]{fair_data}
Fabris, A., Messina, S., Silvello, G., and Susto, G.~A.
\newblock Algorithmic fairness datasets: the story so far.
\newblock \emph{CoRR}, abs/2202.01711, 2022.
\newblock URL \url{https://arxiv.org/abs/2202.01711}.

\bibitem[Freund \& Schapire(1996)Freund and Schapire]{freund}
Freund, Y. and Schapire, R.~E.
\newblock Game theory, on-line prediction and boosting.
\newblock In \emph{Proceedings of the Ninth Annual Conference on Computational Learning Theory}, COLT '96, pp.\  325–332, New York, NY, USA, 1996. Association for Computing Machinery.
\newblock ISBN 0897918118.
\newblock \doi{10.1145/238061.238163}.
\newblock URL \url{https://doi.org/10.1145/238061.238163}.

\bibitem[Gebru et~al.(2018)Gebru, Morgenstern, Vecchione, Vaughan, Wallach, III, and Crawford]{datasheets}
Gebru, T., Morgenstern, J., Vecchione, B., Vaughan, J.~W., Wallach, H.~M., III, H.~D., and Crawford, K.
\newblock Datasheets for datasets.
\newblock \emph{CoRR}, abs/1803.09010, 2018.
\newblock URL \url{http://arxiv.org/abs/1803.09010}.

\bibitem[Gopalan et~al.(2022)Gopalan, Kalai, Reingold, Sharan, and Wieder]{multicalib_omni}
Gopalan, P., Kalai, A.~T., Reingold, O., Sharan, V., and Wieder, U.
\newblock {Omnipredictors}.
\newblock In Braverman, M. (ed.), \emph{13th Innovations in Theoretical Computer Science Conference (ITCS 2022)}, volume 215 of \emph{Leibniz International Proceedings in Informatics (LIPIcs)}, pp.\  79:1--79:21, Dagstuhl, Germany, 2022. Schloss Dagstuhl -- Leibniz-Zentrum f{\"u}r Informatik.
\newblock ISBN 978-3-95977-217-4.
\newblock \doi{10.4230/LIPIcs.ITCS.2022.79}.
\newblock URL \url{https://drops.dagstuhl.de/opus/volltexte/2022/15675}.

\bibitem[Grabowicz et~al.(2022)Grabowicz, Perello, and Mishra]{marry}
Grabowicz, P.~A., Perello, N., and Mishra, A.
\newblock Marrying fairness and explainability in supervised learning.
\newblock In \emph{2022 {ACM} Conference on Fairness, Accountability, and Transparency}. {ACM}, jun 2022.
\newblock \doi{10.1145/3531146.3533236}.
\newblock URL \url{https://doi.org/10.1145%2F3531146.3533236}.

\bibitem[Hardt et~al.(2016{\natexlab{a}})Hardt, Price, and Srebro]{eo_fair}
Hardt, M., Price, E., and Srebro, N.
\newblock Equality of opportunity in supervised learning.
\newblock In Lee, D.~D., Sugiyama, M., von Luxburg, U., Guyon, I., and Garnett, R. (eds.), \emph{Advances in Neural Information Processing Systems 29: Annual Conference on Neural Information Processing Systems 2016, December 5-10, 2016, Barcelona, Spain}, pp.\  3315--3323, 2016{\natexlab{a}}.
\newblock URL \url{https://proceedings.neurips.cc/paper_files/paper/2016/file/9d2682367c3935defcb1f9e247a97c0d-Paper.pdf}.

\bibitem[Hardt et~al.(2016{\natexlab{b}})Hardt, Price, and Srebro]{moritz}
Hardt, M., Price, E., and Srebro, N.
\newblock Equality of opportunity in supervised learning.
\newblock In \emph{Proceedings of the 30th International Conference on Neural Information Processing Systems}, NIPS'16, pp.\  3323–3331, Red Hook, NY, USA, 2016{\natexlab{b}}. Curran Associates Inc.
\newblock ISBN 9781510838819.

\bibitem[Hebert-Johnson et~al.(2018)Hebert-Johnson, Kim, Reingold, and Rothblum]{calib}
Hebert-Johnson, U., Kim, M., Reingold, O., and Rothblum, G.
\newblock Multicalibration: Calibration for the ({C}omputationally-identifiable) masses.
\newblock In Dy, J. and Krause, A. (eds.), \emph{Proceedings of the 35th International Conference on Machine Learning}, volume~80 of \emph{Proceedings of Machine Learning Research}, pp.\  1939--1948. PMLR, 10--15 Jul 2018.
\newblock URL \url{https://proceedings.mlr.press/v80/hebert-johnson18a.html}.

\bibitem[Hu et~al.(2023)Hu, Livni~Navon, Reingold, and Yang]{multicalib_2}
Hu, L., Livni~Navon, I.~R., Reingold, O., and Yang, C.
\newblock Omnipredictors for constrained optimization.
\newblock In Krause, A., Brunskill, E., Cho, K., Engelhardt, B., Sabato, S., and Scarlett, J. (eds.), \emph{Proceedings of the 40th International Conference on Machine Learning}, volume 202 of \emph{Proceedings of Machine Learning Research}, pp.\  13497--13527. PMLR, 23--29 Jul 2023.
\newblock URL \url{https://proceedings.mlr.press/v202/hu23b.html}.

\bibitem[Ingram et~al.(2022)Ingram, Gursoy, and Kakadiaris]{recid_fair}
Ingram, E., Gursoy, F., and Kakadiaris, I.~A.
\newblock Accuracy, fairness, and interpretability of machine learning criminal recidivism models.
\newblock In \emph{2022 IEEE/ACM International Conference on Big Data Computing, Applications and Technologies (BDCAT)}, pp.\  233--241, Los Alamitos, CA, USA, dec 2022. IEEE Computer Society.
\newblock \doi{10.1109/BDCAT56447.2022.00040}.
\newblock URL \url{https://doi.ieeecomputersociety.org/10.1109/BDCAT56447.2022.00040}.

\bibitem[Joseph et~al.(2018)Joseph, Kearns, Morgenstern, Neel, and Roth]{meritocratic}
Joseph, M., Kearns, M., Morgenstern, J., Neel, S., and Roth, A.
\newblock Meritocratic fairness for infinite and contextual bandits.
\newblock In \emph{Proceedings of the 2018 AAAI/ACM Conference on AI, Ethics, and Society}, AIES '18, pp.\  158–163, New York, NY, USA, 2018. Association for Computing Machinery.
\newblock ISBN 9781450360128.
\newblock \doi{10.1145/3278721.3278764}.
\newblock URL \url{https://doi.org/10.1145/3278721.3278764}.

\bibitem[Jung et~al.(2021)Jung, Lee, Pai, Roth, and Vohra]{multicalib_roth}
Jung, C., Lee, C., Pai, M., Roth, A., and Vohra, R.
\newblock Moment multicalibration for uncertainty estimation.
\newblock In Belkin, M. and Kpotufe, S. (eds.), \emph{Proceedings of Thirty Fourth Conference on Learning Theory}, volume 134 of \emph{Proceedings of Machine Learning Research}, pp.\  2634--2678. PMLR, 15--19 Aug 2021.
\newblock URL \url{https://proceedings.mlr.press/v134/jung21a.html}.

\bibitem[Kamiran \& Calders(2012)Kamiran and Calders]{kamiran2012data}
Kamiran, F. and Calders, T.
\newblock Data preprocessing techniques for classification without discrimination.
\newblock \emph{Knowledge and information systems}, 33\penalty0 (1):\penalty0 1--33, 2012.
\newblock \doi{https://doi.org/10.1007/s10115-011-0463-8}.

\bibitem[Kearns et~al.(2018)Kearns, Neel, Roth, and Wu]{gerry}
Kearns, M., Neel, S., Roth, A., and Wu, Z.~S.
\newblock Preventing fairness gerrymandering: Auditing and learning for subgroup fairness.
\newblock In Dy, J. and Krause, A. (eds.), \emph{Proceedings of the 35th International Conference on Machine Learning}, volume~80 of \emph{Proceedings of Machine Learning Research}, pp.\  2564--2572. PMLR, 10--15 Jul 2018.
\newblock URL \url{https://proceedings.mlr.press/v80/kearns18a.html}.

\bibitem[Kearns et~al.(2019)Kearns, Neel, Roth, and Wu]{rich}
Kearns, M., Neel, S., Roth, A., and Wu, Z.~S.
\newblock An empirical study of rich subgroup fairness for machine learning.
\newblock In \emph{Proceedings of the Conference on Fairness, Accountability, and Transparency}, FAT* '19, pp.\  100–109, New York, NY, USA, 2019. Association for Computing Machinery.
\newblock ISBN 9781450361255.
\newblock \doi{10.1145/3287560.3287592}.
\newblock URL \url{https://doi.org/10.1145/3287560.3287592}.

\bibitem[Keddell(2019)]{child}
Keddell, E.
\newblock Algorithmic justice in child protection: Statistical fairness, social justice and the implications for practice.
\newblock \emph{Social Sciences}, 8\penalty0 (10), 2019.
\newblock ISSN 2076-0760.
\newblock \doi{10.3390/socsci8100281}.
\newblock URL \url{https://www.mdpi.com/2076-0760/8/10/281}.

\bibitem[Kindler(2005)]{sion}
Kindler, J.
\newblock A simple proof of sion's minimax theorem.
\newblock \emph{Am. Math. Mon.}, 112\penalty0 (4):\penalty0 356--358, 2005.
\newblock URL \url{http://www.jstor.org/stable/30037472}.

\bibitem[Kingma \& Ba(2015)Kingma and Ba]{adam}
Kingma, D.~P. and Ba, J.
\newblock Adam: {A} method for stochastic optimization.
\newblock In Bengio, Y. and LeCun, Y. (eds.), \emph{3rd International Conference on Learning Representations, {ICLR} 2015, San Diego, CA, USA, May 7-9, 2015, Conference Track Proceedings}, 2015.
\newblock URL \url{http://arxiv.org/abs/1412.6980}.

\bibitem[Kivinen \& Warmuth(1997)Kivinen and Warmuth]{expgrad}
Kivinen, J. and Warmuth, M.~K.
\newblock Exponentiated gradient versus gradient descent for linear predictors.
\newblock \emph{Inf. Comput.}, 132\penalty0 (1):\penalty0 1--63, 1997.
\newblock \doi{10.1006/inco.1996.2612}.
\newblock URL \url{https://doi.org/10.1006/inco.1996.2612}.

\bibitem[Kleinberg et~al.(2017)Kleinberg, Mullainathan, and Raghavan]{kleinberg}
Kleinberg, J., Mullainathan, S., and Raghavan, M.
\newblock {Inherent Trade-Offs in the Fair Determination of Risk Scores}.
\newblock In Papadimitriou, C.~H. (ed.), \emph{8th Innovations in Theoretical Computer Science Conference (ITCS 2017)}, volume~67 of \emph{Leibniz International Proceedings in Informatics (LIPIcs)}, pp.\  43:1--43:23, Dagstuhl, Germany, 2017. Schloss Dagstuhl--Leibniz-Zentrum fuer Informatik.
\newblock ISBN 978-3-95977-029-3.
\newblock \doi{10.4230/LIPIcs.ITCS.2017.43}.
\newblock URL \url{http://drops.dagstuhl.de/opus/volltexte/2017/8156}.

\bibitem[Kleinberg et~al.(2016)Kleinberg, Mullainathan, and Raghavan]{Kleinberg2016InherentTI}
Kleinberg, J.~M., Mullainathan, S., and Raghavan, M.
\newblock Inherent trade-offs in the fair determination of risk scores.
\newblock In \emph{Information Technology Convergence and Services}, 2016.
\newblock URL \url{https://api.semanticscholar.org/CorpusID:12845273}.

\bibitem[Krishna et~al.(2022)Krishna, Han, Gu, Pombra, Jabbari, Wu, and Lakkaraju]{disagree}
Krishna, S., Han, T., Gu, A., Pombra, J., Jabbari, S., Wu, S., and Lakkaraju, H.
\newblock The disagreement problem in explainable machine learning: {A} practitioner's perspective.
\newblock \emph{CoRR}, abs/2202.01602, 2022.
\newblock URL \url{https://arxiv.org/abs/2202.01602}.

\bibitem[Kusner et~al.(2017)Kusner, Loftus, Russell, and Silva]{causal}
Kusner, M.~J., Loftus, J., Russell, C., and Silva, R.
\newblock Counterfactual fairness.
\newblock In Guyon, I., Luxburg, U.~V., Bengio, S., Wallach, H., Fergus, R., Vishwanathan, S., and Garnett, R. (eds.), \emph{Advances in Neural Information Processing Systems}, volume~30. Curran Associates, Inc., 2017.
\newblock URL \url{https://proceedings.neurips.cc/paper_files/paper/2017/file/a486cd07e4ac3d270571622f4f316ec5-Paper.pdf}.

\bibitem[Li \& Vasconcelos(2019)Li and Vasconcelos]{li2019repair}
Li, Y. and Vasconcelos, N.
\newblock Repair: Removing representation bias by dataset resampling.
\newblock In \emph{2019 IEEE/CVF Conference on Computer Vision and Pattern Recognition (CVPR)}, pp.\  9564--9573, 2019.
\newblock \doi{10.1109/CVPR.2019.00980}.

\bibitem[Liu et~al.(2022)Liu, Zhong, Seltzer, and Rudin]{gam}
Liu, J., Zhong, C., Seltzer, M., and Rudin, C.
\newblock Fast sparse classification for generalized linear and additive models.
\newblock In Camps-Valls, G., Ruiz, F. J.~R., and Valera, I. (eds.), \emph{Proceedings of The 25th International Conference on Artificial Intelligence and Statistics}, volume 151 of \emph{Proceedings of Machine Learning Research}, pp.\  9304--9333. PMLR, 28--30 Mar 2022.
\newblock URL \url{https://proceedings.mlr.press/v151/liu22f.html}.

\bibitem[Lundberg(2020)]{lundberg2020explaining}
Lundberg, S.~M.
\newblock Explaining quantitative measures of fairness.
\newblock In \emph{Fair \& Responsible AI Workshop@ CHI2020}, 2020.

\bibitem[Lundberg \& Lee(2017)Lundberg and Lee]{shap}
Lundberg, S.~M. and Lee, S.-I.
\newblock A unified approach to interpreting model predictions.
\newblock In \emph{Proceedings of the 31st International Conference on Neural Information Processing Systems}, NIPS'17, pp.\  4768–4777, Red Hook, NY, USA, 2017. Curran Associates Inc.
\newblock ISBN 9781510860964.

\bibitem[Miller(2015)]{nyt}
Miller, C.
\newblock Algorithms and bias: Q. and a. with cynthia dwork, Aug 2015.

\bibitem[Molnar(2022)]{molnar2022}
Molnar, C.
\newblock \emph{Interpretable Machine Learning}.
\newblock 2 edition, 2022.
\newblock URL \url{https://christophm.github.io/interpretable-ml-book}.

\bibitem[Moro et~al.(2014)Moro, Cortez, and Rita]{MORO201422}
Moro, S., Cortez, P., and Rita, P.
\newblock A data-driven approach to predict the success of bank telemarketing.
\newblock \emph{Decision Support Systems}, 62:\penalty0 22--31, 2014.
\newblock ISSN 0167-9236.
\newblock \doi{https://doi.org/10.1016/j.dss.2014.03.001}.
\newblock URL \url{https://www.sciencedirect.com/science/article/pii/S016792361400061X}.

\bibitem[Neill(2012)]{neill2012spatialsubset}
Neill, D.~B.
\newblock {Fast Subset Scan for Spatial Pattern Detection}.
\newblock \emph{Journal of the Royal Statistical Society Series B: Statistical Methodology}, 74\penalty0 (2):\penalty0 337--360, 01 2012.
\newblock ISSN 1369-7412.
\newblock \doi{10.1111/j.1467-9868.2011.01014.x}.
\newblock URL \url{https://doi.org/10.1111/j.1467-9868.2011.01014.x}.

\bibitem[Pleiss et~al.(2017)Pleiss, Raghavan, Wu, Kleinberg, and Weinberger]{calib2}
Pleiss, G., Raghavan, M., Wu, F., Kleinberg, J., and Weinberger, K.~Q.
\newblock On fairness and calibration.
\newblock In Guyon, I., Luxburg, U.~V., Bengio, S., Wallach, H., Fergus, R., Vishwanathan, S., and Garnett, R. (eds.), \emph{Advances in Neural Information Processing Systems}, volume~30. Curran Associates, Inc., 2017.
\newblock URL \url{https://proceedings.neurips.cc/paper_files/paper/2017/file/b8b9c74ac526fffbeb2d39ab038d1cd7-Paper.pdf}.

\bibitem[Quinlan(1986)]{trees}
Quinlan, J.~R.
\newblock Induction of decision trees.
\newblock \emph{Machine Learning}, 1\penalty0 (1):\penalty0 81--106, 1986.
\newblock \doi{10.1007/BF00116251}.
\newblock URL \url{https://doi.org/10.1007/BF00116251}.

\bibitem[Ribeiro et~al.(2016)Ribeiro, Singh, and Guestrin]{lime}
Ribeiro, M.~T., Singh, S., and Guestrin, C.
\newblock "why should i trust you?": Explaining the predictions of any classifier.
\newblock In \emph{Proceedings of the 22nd ACM SIGKDD International Conference on Knowledge Discovery and Data Mining}, KDD '16, pp.\  1135–1144, New York, NY, USA, 2016. Association for Computing Machinery.
\newblock ISBN 9781450342322.
\newblock \doi{10.1145/2939672.2939778}.
\newblock URL \url{https://doi.org/10.1145/2939672.2939778}.

\bibitem[Simonyan et~al.(2013)Simonyan, Vedaldi, and Zisserman]{saliency}
Simonyan, K., Vedaldi, A., and Zisserman, A.
\newblock Deep inside convolutional networks: Visualising image classification models and saliency maps.
\newblock \emph{CoRR}, abs/1312.6034, 2013.
\newblock URL \url{http://dblp.uni-trier.de/db/journals/corr/corr1312.html#SimonyanVZ13}.

\bibitem[Slack et~al.(2020)Slack, Hilgard, Jia, Singh, and Lakkaraju]{slack2020fooling}
Slack, D., Hilgard, S., Jia, E., Singh, S., and Lakkaraju, H.
\newblock Fooling lime and shap: Adversarial attacks on post hoc explanation methods.
\newblock In \emph{Proceedings of the AAAI/ACM Conference on AI, Ethics, and Society}, AIES '20, pp.\  180–186, New York, NY, USA, 2020. Association for Computing Machinery.
\newblock ISBN 9781450371100.
\newblock \doi{10.1145/3375627.3375830}.
\newblock URL \url{https://doi.org/10.1145/3375627.3375830}.

\bibitem[Sundararajan et~al.(2017)Sundararajan, Taly, and Yan]{saliency2}
Sundararajan, M., Taly, A., and Yan, Q.
\newblock Axiomatic attribution for deep networks.
\newblock In Precup, D. and Teh, Y.~W. (eds.), \emph{Proceedings of the 34th International Conference on Machine Learning}, volume~70 of \emph{Proceedings of Machine Learning Research}, pp.\  3319--3328. PMLR, 06--11 Aug 2017.
\newblock URL \url{https://proceedings.mlr.press/v70/sundararajan17a.html}.

\bibitem[Tommasi et~al.(2017)Tommasi, Patricia, Caputo, and Tuytelaars]{tommasi2017deeper}
Tommasi, T., Patricia, N., Caputo, B., and Tuytelaars, T.
\newblock A deeper look at dataset bias.
\newblock In \emph{Domain adaptation in computer vision applications}, pp.\  37--55. Springer, 2017.

\bibitem[Zemel et~al.(2013)Zemel, Wu, Swersky, Pitassi, and Dwork]{zem_fair}
Zemel, R.~S., Wu, Y., Swersky, K., Pitassi, T., and Dwork, C.
\newblock Learning fair representations.
\newblock In \emph{Proceedings of the 30th International Conference on Machine Learning, {ICML} 2013, Atlanta, GA, USA, 16-21 June 2013}, volume~28 of \emph{{JMLR} Workshop and Conference Proceedings}, pp.\  325--333. JMLR.org, 2013.
\newblock URL \url{http://proceedings.mlr.press/v28/zemel13.html}.

\bibitem[Zhang \& Neill(2016)Zhang and Neill]{Zhang2016ltss}
Zhang, Z. and Neill, D.~B.
\newblock Identifying significant predictive bias in classifiers.
\newblock \emph{ArXiv}, abs/1611.08292, 2016.
\newblock URL \url{https://api.semanticscholar.org/CorpusID:2256765}.

\end{thebibliography}
\bibliographystyle{icml2024}

\newpage
\appendix
\onecolumn

\newpage
\appendix

\newpage
\section{Limitations}
\label{sec:limitations}
Importantly, we eschew any broader claims that large $\fid$ \emph{necessarily} implies a mathematical conclusion about the \emph{fairness} of the underlying classification model in all cases. It is known that even the most popular and natural fairness metrics are impossible to satisfy simultaneously, and so we would run up against the problem of determining what it means for a model to be \emph{fair} \cite{chouldechova, kleinberg}. By detecting anomalous subgroups with respect to feature importance, our approach can signal to a domain expert that perhaps there are issues such as feature collection or measurement bias. This will facilitate the next steps of testing the resulting hypotheses, and ultimately intervening to address disparities and improve fairness outcomes. Concerns about the stability and robustness of the most widely used feature importance notions, including the ones we study, have been raised \cite{Dai_2022, agarwal2022rethinking, alvarezmelis2018robustness, bansal2020sam, Dimanov2020YouST, slack2020fooling} and these notions are often at odds with each other, so none can be considered definitive \cite{disagree}. Regardless of these limitations, these notions are used widely in practice today, and are still useful as a diagnostic tool as we propose here in order to uncover potentially interesting biases. Lastly, our methods, like nearly all prior works on fairness, require tabular datasets that have defined the sensitive features apriori, a process more difficult in text or image datasets where bias is still a concern \cite{gendershades, home}. Overall, the methods developed here represent a part of the algorithmic toolkit that domain experts may use in rooting out bias.



\section{Reproducibility}
\label{sec:reproducibility}

Specific details for the experiments such as the hyperparameters used are available in Appendix~\ref{sec:details_app}. The source code used for these experiments is provided in the supplementary material. Specifically, \texttt{run\_separable.py} and \texttt{run\_linear.py} are the scripts where the importance notion (Appendix~\ref{subsec:fid_notions}), dataset (Appendix~\ref{subsec:datasets}), and other parameters are specified before running. The \texttt{experiments/} directory contains scripts used for the comparison of rich and marginal subgroups as seen in Appendix~\ref{sec:rich_marg} and for the fairness comparison experiments in Subsection~\ref{subsec:fairness}.

\section{Additional Related Work}
\label{sec:rel_work_app}

\textbf{Fairness in Machine Learning}.
Much of the work in fairness in machine learning typically concerns the implementation of a new fairness notion in a given learning setting; either an individual fairness notion \cite{fairness_through_awareness, meritocratic}, one based on equalizing a statistical rate across protected subgroups \cite{moritz, calib2}, or one based on an underlying causal model \cite{causal}. With a given notion of fairness in hand, approaches to learning fair classifiers can be typically classified as ``in-processing", or trying to simultaneously learn a classifier and satisfy a fairness constraint, ``post-processing" which takes a learned classifier and post-processes it to satisfy a fairness definition \cite{moritz}, or most closely related to the motivation behind this paper, pre-processing the data to remove bias. Existing work on dataset bias serve as high level motivation for our work.

\textbf{Feature Importance Notions}. 
The local explanation methods mentioned in Section~\ref{sec:related} include model-agnostic methods like LIME or SHAP \cite{lime, shap}, methods like saliency maps \cite{saliency, saliency2, saliency3} that require $h$ to be differentiable in $x$, or model-specific methods that depend on the classifier. In addition to these explanation methods, there are also global methods that attempt to explain the entire model behavior and so can be run on the entire subgroup. Our $\linfid$ method as described in Appendix~\ref{sec:lin_fid} is a global method that relies on training an inherently interpretable model (linear regression) on the subgroup and inspecting its coefficients. Other inherently interpretable models that could be used to define a notion of subgroup importance include decision trees \cite{trees} and generalized additive models \cite{gam}.

\textbf{Fairness and Interpretability}. Although no existing work examines the role of feature importance notions in detecting disparities in rich subgroups, there is a small amount of existing work examining explainability in the context of fairness. The recent \cite{marry} formalizes induced discrimination as a function of the SHAP values assigned to sensitive features, and proposes a method to learn classifiers where the protected attributes have low influence. \cite{exp_fair} applies a similar approach, attributing a models overall unfairness to its individual features using the Shapley value, and proposing an intervention to improve fairness. \cite{recid_fair} examines machine learning models to predict recidivism, and empirically shows tradeoffs between model accuracy, fairness, and interpretability.

Additionally, \cite{lundberg2020explaining} decomposes feature attribution explanations and fairness metrics into additive components and observes the relationship between the fairness metrics and input features. Our work does not try to decompose fairness metrics into additive components and also focuses on non-additive feature explanations. Furthermore, our consideration of rich subgroups is a novel addition to the space.

\section{Computing the Optimal Subgroup}
\label{sec:proof_app}
\begin{algorithm}[!ht]
\caption{Iterative Constrained Optimization}
\label{alg:cap}
\begin{algorithmic}[1]
\STATE \textbf{Input:} Dataset $X^n, |X^n| = n$, hypothesis $h$, feature of interest $f_j$, separable feature importance function $F$, size constraints $\alpha_L$ and $\alpha_U$, size violation indicators $\Phi_L$ and $\Phi_U$, size penalty bound $B$, CSC oracle for $\mathcal{G}$, $CSC_{\mathcal{G}}(c^0,c^1)$,  accuracy $\nu$.
\STATE \textbf{Initialize:}
\STATE 
Feature importance vector $\mathbf{C}=(F(f_j, X_i, h))_{i=1}^n$

\STATE Gradient weight parameter $\theta_1 = (0,0)$

\STATE Learning rate $\eta = \frac{\nu}{2n^2B}$

\FOR{$t=1,2,...$}

\STATE \textit{\# Exponentiated gradient weights}
\STATE $\lambda_{t,0}=B\frac{exp(\theta_{t,0})}{1+exp(\theta_{t,1})}$, \quad $\lambda_{t,1}=B\frac{exp(\theta_{t,1})}{1+exp(\theta_{t,0})}$
\STATE  

\STATE \textit{\# Costs vector}
\STATE $c_t^1 = (\mathbf{C}_i-\lambda_{t,0}+\lambda_{t,1})_{i=1}^n$
\STATE  

\STATE \textit{\# Get g with max disparity computed via CSC oracle}
\STATE $g_t = CSC_{\mathcal{G}} (\mathbf{0},c_t^1)$ 
\STATE  

\STATE \textit{\# Compute Lagrangian}
\STATE $\hat{p}_{\mathcal{G}}^{t} = \frac{1}{t}\sum_{t'=1}^{t}g_{t'}$
\STATE $\lambda_t' = (B \Phi_L (\hat{p}_{\mathcal{G}}^{t}), B \Phi_U (\hat{p}_{\mathcal{G}}^{t}))$
\STATE $\overline{L} = L(\hat{p}_{\mathcal{G}}^{t}, \lambda_t')$

\STATE

\STATE $\hat{p}_{\lambda}^{t} = \frac{1}{t}\sum_{t'=1}^{t}(\lambda_{t', 0}, \lambda_{t', 1})$
\STATE $g_t' = CSC_{\mathcal{G}} (\mathbf{0}, (\mathbf{C}_i - \hat{p}_{\lambda_0}^{t} + \hat{p}_{\lambda_1}^{t})_{i=1}^n)$
\STATE $\underline{L} = L(g_t', \hat{p}_{\lambda}^{t})$

\STATE

\STATE $v_t = \max \big( \lvert L(\hat{p}_{\mathcal{G}}^{t}, \hat{p}_{\lambda}^{t}) - \underline{L} \rvert, \lvert \overline{L} -  L(\hat{p}_{\mathcal{G}}^{t}, \hat{p}_{\lambda}^{t})\rvert \big)$ 
\STATE  

\STATE \textit{\# Check termination condition}
\IF{$v_t \leq v$}
\STATE Return $\hat{p}_{\mathcal{G}}^{t}, \hat{p}_{\lambda}^{t}$
\ENDIF
\STATE  

\STATE \textit{\# Exponentiated gradient update}
\STATE Set $\theta_{t+1} = \theta_t + \eta (\alpha_L-|g_t|, |g_t|-\alpha_U)$ 

\ENDFOR
\end{algorithmic}
\end{algorithm}

We start by showing that for the unconstrained problem, computing the subgroup $g_j^{*}$ that maximizes $\fid(f_j, g, h)$ over $\mathcal{G}$ can be computed in two calls to $\cscg$ when $F$ is separable. 
\begin{lemma}
\label{lem:sep}
If $F$ is separable and $\text{CSC}_{\mathcal{G}}$ is a CSC oracle for $\mathcal{G}$, then for any feature $f_j$, $g_{j}^*$ can be computed with two oracle calls. 
\end{lemma}

\begin{proof}
By definition $g^*_j = \argmax_{g \in \mathcal{G}}\fid(j,g) = \argmax_{g \in \mathcal{G}}|F(f_j, X^n, h) - F(f_j, g, h)| = \argmax_{g \in \{g^{+}, g^-\}}\fid(j,g)$, where $g^{+} = \argmax_{g \in \mathcal{G}}F(f_j, g, h), g^{-} = \argmin_{g \in \mathcal{G}}F(f_j, g, h)$. 
By the definition of separability, we can write $$ F(f_j,g(X^n),h) = \sum_{X \in g(X^n)}F'(f_j,X,h) = 
\sum_{i = 1}^{n}g(X_i)F'(f_j,X_i,h)$$
Then letting $c_k^{0} = 0$ and $c_k^1 = -F'(f_j,X_k,h)$ for $k = 1, \ldots n$, we see that $g^{+} = \text{CSC}_g((c_k^{0}, c_k^{1})), g^{-} = \text{CSC}_g((c_k^{0}, -c_k^{1}))$. This establishes the claim. 
\end{proof}

Theorem~\ref{thm:sep}: Let $F$ be a separable notion, fix a classifier $h$, subgroup class $\mathcal{G}$, and oracle $\text{CSC}_{\mathcal{G}}$. Then fixing a feature of interest $f_j$,  we will run Algorithm~\ref{alg:cap} twice; once with $\fid$ given by $F$, and once with $\fid$ given by $-F$. Let $\hat{p}_{\mathcal{G}}^{T}$ be the distribution returned after $T = O(\frac{4n^2 B^2}{\nu^2})$ iterations by Algorithm~\ref{alg:cap} that achieves the larger value of $\mathbb{E}[\fid(j,g)]$. Then:  

\begin{equation}
\begin{aligned}
 & \fid(j, g_j^{*})-  \mathbb{E}_{g \sim \hat{p}_{\mathcal{G}}^{T}}[\fid(j,g)] \leq \nu \\
 & \quad \lvert \Phi_L(g) \rvert, \lvert \Phi_U(g) \rvert  \leq  \frac{1 + 2\nu}{B}
\end{aligned}
\end{equation}

\begin{proof}
We start by transforming our constrained optimization into optimizing a $\min-\max$ objective. The $\min$ player, referred to as the \emph{subgroup player} will be solving a CSC problem over the class $\mathcal{G}$ at each iteration, while the $\max$ player, called the \emph{dual player}, will be adjusting the dual weights $\lambda$ on the two constraints using the exponentiated gradient algorithm \cite{expgrad}. By Lemma~\ref{lem:freund} \cite{freund}, we know that if each player implements a \emph{no-regret} strategy, then the error of subgroup found after $T$ rounds is sub-optimal by at most the average cumulative regret of both players. The regret bound for the exponentiated gradient descent ensures this occurs in \emph{poly(n)} rounds. 

As in \cite{gerry, msr}, we first relax Equation~\ref{eq:constr} to optimize over all \emph{distributions} over subgroups, and we enforce that our constraints hold in expectation over this distribution.  Our new optimization problem becomes: 

\begin{equation}
\label{eq:constr_prob}
\begin{aligned}
\min_{p_g \in \Delta(\mathcal{G})} \quad & \mathbb{E}_{g \sim p_g}[\sum_{i = 1}^{n}g(x_i)F'(f_j,x_i,h)]\\
\textrm{s.t.} \quad & \mathbb{E}_{g \sim p_g}[\Phi_L(g)] \leq 0\\
\quad &  \mathbb{E}_{g \sim p_g}[\Phi_U(g)] \leq 0 \\
\end{aligned}
\end{equation}

We note that while $|\mathcal{G}|$ may be infinite, the number of distinct labelings of $X$ by elements of $\mathcal{G}$ is finite; we denote the number of these by $|\mathcal{G}(X)|$. Then since Equation~\ref{eq:constr_prob} is a finite linear program in $|\mathcal{G}(X)|$ variables, it satisfies strong duality, and we can write:

\begin{equation*}
\begin{aligned}
(p_g^{*}, \lambda^{*}) &= \argmin_{p_g \in \Delta(\mathcal{G})}\argmax_{\lambda \in \Lambda}\mathbb{E}_{g \sim p_g}[L(g, \lambda)] = \argmin_{p_g \in \Delta(\mathcal{G})}\argmax_{\lambda \in \Lambda}L(p_g, \lambda) \\
\textrm{with} \quad  L(g, \lambda) &= \sum_{x \in X} g(x) F(f_j,x,h) + \lambda_L \Phi_L + \lambda_U \Phi_U, \quad L(p_g, \lambda) = \mathbb{E}_{g \sim p_g}[L(g, \lambda)] \\
\end{aligned}
\end{equation*}
As in \cite{gerry} $\Lambda = \{ \lambda \in \mathbb{R}^2 \mid \lVert \lambda \rVert_1 \leq B \}$ is chosen to make the domain compact, and does not change the optimal parameters as long as $B$ is sufficiently large, i.e. $\lVert \lambda^{*} \rVert_1 \leq B$. In practice, this is a hyperparameter of Algorithm~\ref{alg:cap}, similar to \cite{msr, gerry}. 
Then we follow the development in \cite{msr, gerry} to show that we can compute $(p_g^{*}, \lambda^{*})$ efficiently by implementing \emph{no-regret} strategies for the subgroup player ($p_g$) and the dual player ($\lambda$). 

Formally, since $\mathbb{E}_{g \sim p_g}[L(g, \Lambda)]$ is bi-linear in $p_g, \lambda$, and $\Lambda, \Delta(\mathcal{G})$ are convex and compact, by Sion's minimax theorem \cite{sion}: 

\begin{equation}
\label{eq:minmax}
    \min_{p_g \in \Delta(G)} \max_{\lambda \in \Lambda} L(p_g, \lambda) = \max_{\lambda \in \Lambda} \min_{p_g \in \Delta(G)} L(p_g, \lambda) = \text{OPT}
\end{equation}

Then by Theorem 4.5 in \cite{gerry}, we know that if $(p_{g}^*, \lambda^*)$ is a $\nu$-approximate min-max solution to Equation~\ref{eq:minmax} in the sense that
\begin{equation}
\begin{aligned}
\text{if:} \quad L(p_g^*, \lambda^*) \leq \min_{p \in \Delta(\mathcal{G})}L(p, \lambda^*) + \nu, L(p_g, \lambda) \geq \max_{\lambda \in \Lambda}L(p^*, \lambda), \\
\text{then:} \quad F(f_j, p_g^*, h) \leq OPT + 2\nu, \quad \lvert \Phi_L(g) \rvert, \lvert \Phi_U(g) \rvert \leq  \frac{1 + 2\nu}{B}
\end{aligned}
\end{equation}

So in order to compute an approximately optimal subgroup distribution $p_g^*$, it suffices to compute an approximate min-max solution of Equation~\ref{eq:minmax}. In order to do that we rely on the classic result of \cite{freund} that states that if the subgroup player best responds, and if the dual player achieves low regret, then as the average regret converges to zero, so does the sub-optimality of the average strategies found so far. 

\begin{lemma}[\cite{freund}]
\label{lem:freund}
Let $p^\lambda_1, \ldots p^{\lambda}_T$ be a sequence of distributions over $\Lambda$, played by the dual player, and let $g^{1}, \ldots g^{T}$ be the subgroup players best responses against these distributions respectively. Let $\hat{\lambda}_T = \frac{1}{T}\sum_{t=1}^{T}p^{\lambda}_t, \hat{p}_g = \frac{1}{T}\sum_{t = 1}^{T}g_t$.
Then if 
$$
\sum_{t = 1}^{T}\mathbb{E}_{\lambda \sim p^{\lambda}_t}[L(g_t, \lambda)] - \min_{\lambda \in \Lambda}\sum_{t=1}^{T}[L(g_t, \lambda)] \leq \nu T,
$$
Then $(\hat{\lambda}_T, \hat{p}_g)$ is a $\nu$-approximate minimax equilibrium of the game.
\end{lemma}

To establish Theorem~\ref{thm:sep}, we need to show (i) that we can efficiently implement the subgroup players best response using $\text{CSC}_{\mathcal{G}}$ and (ii) we need to translate the regret bound for the dual players best response into a statement about optimality, using Lemma~\ref{lem:freund}. Establishing $(i)$ is immediate, since at each round $t$, if $\lambda_{t, 0} = \mathbb{E}_{p^{\lambda}_t}[\lambda_L], \lambda_{t, 1} = \mathbb{E}_{p^{\lambda}_t}[\lambda_U]$, then the best response problem is:

$$\argmin_{p_g \in \Delta(G)} \mathbb{E}_{g \sim p_g}[\sum_{x \in X} g(x) F(f_j,x,h) + \lambda_{t, 0} \Phi_L + \lambda_{t, 1} \Phi_U]$$

Which can further be simplified to:
\begin{equation}
    \argmin_{g \in G} \sum_{x \in X} g(x) (F(f_j,x,h) - \lambda_L + \lambda_U)
\end{equation}
This can be computed with a single call of $\text{CSC}_{\mathcal{G}}$, as desired.  To establish (ii), the no-regret algorithm for the dual player's distributions, we note that at each round the dual player is playing online linear optimization over $2$ dimensions. Algorithm~\ref{alg:cap} implements the exponentiated gradient algorithm \cite{expgrad}, which has the following guarantee proven in Theorem $1$ of \cite{msr}, which follows easily from the regret bound of exponentiated gradient \cite{expgrad}, and Lemma~\ref{lem:freund}:
\begin{lemma}[\cite{msr}]
Setting $\eta = \frac{\nu}{2n^2B}$, Algorithm~\ref{alg:cap} returns $\hat{p}_{\lambda}^{T}$ that is a $\nu$-approximate min-max point in at most $O(\frac{4n^2 B^2}{\nu^2})$ iterations. 
\end{lemma}
Combining this result with Equation~\ref{eq:minmax} completes the proof.

\end{proof}

\section{Proof of $\avgfid$ Primitive}
\label{sec:avgsepfidproof_app}

In Section~\ref{sec:optimizing}, we presented our approach that optimizing for $\fid$ constrained across a range of subgroup sizes will allow us to efficiently optimize for $\avgfid$. We provide a more complete proof of that claim here:

Let $g^*$ be the subgroup that maximizes $\avgfid$. Without loss of generality, $g^*=\argmax_{g \in \mathcal{G}}\frac{1}{n|g|}\sum g(x)F'(f_j,X,h)$ (we drop the absolute value because we can also set $F'=-F$). Then it is necessarily true, that $g^*$ also solves the constrained optimization problem $\argmax_{g \in \mathcal{G}}\frac{1}{n}\sum g(x)F'(f_j,X,h)$ such that $|g|=|g^*|$, where we have dropped the normalizing term $\frac{1}{|g|}$ in the objective function, and so we are maximizing the constrained $\fid$. 

Now consider an interval $I = [|g^*| - \alpha, |g^*| + \alpha]$, and suppose we solve $g_I^* = \argmax_{g \in \mathcal{G}}\frac{1}{n}\sum g(x)F'(f_j,X,h)$ such that $g \in I$. Then since $g^* \in I$, we know that $\frac{1}{n}\sum g^* F'(f_j,X,h) \leq \frac{1}{n} \sum g_I^*(x)F'(f_j,X,h)$. This implies that:

\begin{align*}
    \avgfid (g_I^*) &\geq \frac{1}{|g_I^*|}\frac{1}{n} \sum g^*(x) F'(f_j,X,h) \\
      &= \avgfid (g^*) + (\frac{1}{|g_I^*|+\alpha}-\frac{1}{|g_I^*|})\fid (g^*) \\
      &= \avgfid (g^*) - \frac{\alpha}{|g^*|(|g^*|+\alpha)}\cdot \fid (g^*)
\end{align*}

Given the above derivation, as $\alpha \rightarrow 0$, we have $\avgfid (g_I^*) \rightarrow \avgfid (g^*)$.

Hence we can compute a subgroup $g$ that approximately optimizes the $\avgfid$ if we find an appropriately small interval $I$ aroudn $|g^*|$. Since the discretization in Section~\ref{sec:optimizing} covers the unit interval, we are guaranteed for sufficiently large $n$ to find such an interval.

\section{Cost Sensitive Classifier, $\cscg$}
\begin{definition}
\label{def:csc}
    (Cost Sensitive Classification) A Cost Sensitive Classification (CSC) problem for a hypothesis class $\mathcal{G}$ is given by a set of $n$ tuples $\{ (X_i, c_i^0, c_i^1) \}_{i=1}^n$, where $c_i^0$ and $c_i^1$ are the costs of assigning labels $0$ and $1$ to $X_i$ respectively. A CSC oracle finds the classifier $\hat{g} \in \mathcal{G}$ that minimizes the total cost across all points:
    \begin{equation}
        \hat{g} = \operatorname*{argmin}_{g\in \mathcal{G}}\sum_i \Big( g(X_i) c_i^1+(1-g(X_i)) c_i^0 \Big)
    \end{equation}
\end{definition}

\label{sec:cscg_app}
\begin{algorithm}[h]
\caption{$\text{CSC}_{\mathcal{G}}$}
\label{alg:csc}
\begin{algorithmic}
\STATE \textbf{Input:} Dataset $X \subset \mathbb{R}^{d_{sens}} \times \mathbb{R}^{d_{safe}}$, costs $(c^0, c^1) \in \mathbb{R}^n$
\STATE Let $X_{sens}$ consist of the sensitive attributes $x$ of each $(x, x') \in X$.
\STATE 
\STATE \textit{\# learn to predict the cost $c^0$}
\STATE Train linear regressor $r_0: \mathbb{R}^{d_{sens}} \to \mathbb{R}$ on dataset $(X_{sens}, c^0)$
\STATE 
\STATE \textit{\# learn to predict the cost $c^1$}
\STATE Train linear regressor $r_1: \mathbb{R}^{d_{sens}} \to \mathbb{R}$ on dataset $(X_{sens}, c^1)$
\STATE 
\STATE \textit{\# predict $0$ if the estimated $c_0 < c_1$}
\STATE Define $g((x,x')) \defeq \textbf{1}\{(r_0-r_1)(x) > 0\}$
\STATE Return $g$
\end{algorithmic}
\end{algorithm}

\section{NP-Completeness}
\label{sec:hardness_app}

We will show below that the fully general version of this problem (allowing any poly-time $F$) is NP complete. First, we will define a decision variant of the problem:

\begin{equation*}
     \delta_{X, F, h, A} = \max_{g\in \mathcal{G}, f_j} (|F(f_j, g, h) - F(f_j, X, h)|)\geq A
\end{equation*}

Note that a solution to the original problem trivially solves the decision variant. First, we will show the decision variant is in NP, then we will show it is NP hard via reduction to the max-cut problem.

\begin{lemma}
The decision version of this problem is in NP.
\end{lemma}
\begin{proof}
Our witness will be the subset $g$ and feature $f_j$ such that 
\begin{equation*}
    (|F(f_j, g, h) - F(f_j, X, h)|)\geq A
\end{equation*}
Given these $2$, evaluation of the absolute value is polytime given that $F$ is polytime, so the solution can be verified in polytime.
\end{proof}

\begin{lemma}
The decision version of this problem is NP hard.
\end{lemma}

\begin{proof}
We will define our variables to reduce our problem to $\textrm{maxcut}(Q, k)$.  Given a graph defined with $V,\ E$ as the vertex and edge sets of $Q$ (with edges defined as pairs of vertices), we will define our $F$, $X$, $G$, $A$, and $h$ as follows:
\begin{eqnarray*}
X &=& V \\
h &=& \textrm{constant classifier, maps every value to 1} \\
\mathcal{G}&=&\mathcal{P}(V) \; \textrm{i.e. all possible subsets of vertices} \\
F(f_j, g, h) &=& |{x\in E : x[0] \in g, x[1] \in g^c}| \\
& & \textrm{--i.e.}\ F(j, g, h)\ \textrm{returns the number of} \\
& & \textrm{edges cut by a particular subset, ignoring} \\
& & \textrm{its first and third argument.} \\
& & \textrm{(this is trivially computable in polynomial} \\
& & \textrm{time by iterating over the set of edges).} \\
A &=& k
\end{eqnarray*}

Note that $F(f_j, X, h) = 0$ by definition, and that $F \geq 0$. Therefore, $|F(f_j, g, h) - F(f_j, X, h)| = F(f_j, g, h)$, and we see that
$(|F(f_j, g, h) - F(f_j, X, h)|)\geq A$ if and only if $g$ is a subset on $Q$ that cuts at least $A=k$ edges.  Therefore an algorithm solving the decision variant of the feature importance problem also solves maxcut.
\end{proof}

\section{Linear Feature Importance Disparity}
\label{sec:lin_fid}

The \emph{non-separable} $\fid$ notion considered in this paper corresponds to training a model that is inherently interpretable on only the data in the subgroup $g$, and comparing the influence of feature $j$ to the influence when trained on the dataset as a whole. Since all of the points in the subgroup can interact to produce the interpretable model, this notions typically are not separable. Below we formalize this in the case of linear regression, which is the non-separable notion we investigate in the experiments. 

\begin{definition}
\label{def:lin}
(Linear Feature Importance Disparity).
Given a subgroup $g$, let $\theta_g = \inf_{\theta \in \mathbb{R}^d}\mathbb{E}_{(X, y) \sim \mathcal{R}}[g(X)(\theta' X - y)^2]$, and
$\theta_{\mathcal{R}} = \inf_{\theta \in \mathbb{R}^d}\mathbb{E}_{(X, y) \sim \mathcal{R}}[(\theta' X - y)^2]$. Then if $e_j$ is the $j^{th}$ basis vector in $\mathbb{R}^d$,
we define the \emph{linear feature importance disparity} ($\linfid$) by 
  $$\linfid(j, g) = \lvert (\theta_g - \theta_{\mathcal{R}}) \cdot e_j \rvert$$
\end{definition}

$\linfid(j,g)$ is defined as the difference between the coefficient for feature $j$ when training the model on the subgroup $g$, versus training the model on points from $\mathcal{R}$. Expanding Definition~\ref{def:lin} using the standard weighted least squares estimator (WLS), the feature importance for a given feature $f_j$ and subgroup $g(X)$ is:

\begin{equation}
\label{eq:flinear}
    \flin (j,g) = \big( (X g(X) X^T)^{-1}(X^T g(X) Y) \big) \cdot e_j,
\end{equation}

Where $g(X)$ is a diagonal matrix of the output of the subgroup function. The coefficients of the linear regression model on the dataset $X$ can be computed using the results from ordinary least squares (OLS): $(X X^T)^{-1}(X^T Y) \cdot e_j$. 

We compute $\argmax_{g \in G} \linfid = \argmax_{g \in G} \lvert \flin (j, X^n) - \flin (j, g) \rvert$ by finding the minimum and maximum values of $\flin (j, g)$ and choosing the one with the larger difference. For the experiments in Section~\ref{sec:experiments}, we use logistic regression as the hypothesis class for $g$ because it is non-linear enough to capture complex relationships in the data, but maintains interpretability in the form of its coefficients, and importantly because Equation~\ref{eq:flinear} is then differentiable in the parameters $\theta$ of $g(X) = \sigma(X \cdot \theta), \sigma(x) = \frac{1}{1+e^{-x}}$. Since Equation~\ref{eq:flinear} is differentiable in $\theta$, we can use non-convex optimizers like SGD or ADAM to maximize Equation~\ref{eq:flinear} over $\theta$.

While this is an appealing notion due to its simplicity, it is not relevant unless the matrix $Xg(X)X^{T}$ is of full rank. We ensure this first by lower bounding the size of $g$ via a size penalty term $P_{size} = \max(\alpha_L-|g(X_{train})|,0) + \max(|g(X_{train})|-\alpha_U,0)$, which allows us to provide $\alpha$ constraints in the same manner as in the separable approach. We also add a small $l_2$ regularization term $\epsilon I$ to $X^{T}g(X)X$. This forces the matrix to be invertible, avoiding issues with extremely small subgroups. Incorporating these regularization terms, Equation~\ref{eq:flinear} becomes:

\begin{equation}
\label{eq:full_lin}
    \flin (j, g) = \lambda_s \cdot \big((X \sigma(X \cdot \theta_L^T) X^T + \epsilon I)^{-1}(X^T\sigma(X \cdot \theta_L^T)Y) \cdot e_j\big) + \lambda_c \cdot P_{size}
\end{equation}

We note that $\linfid$ is a similar notion to that of LIME \cite{lime}, but LIME estimates a local effect around each point which is then summed to get the effect in the subgroup, and so it is \emph{separable}. It is also the case that $\flin$ is non-convex as shown below:

\begin{lemma}
$\flin$ as defined in Equation~\ref{eq:flinear} is non-convex.
\end{lemma}

\begin{proof}
We will prove this by contradiction. Assume $\flin$ is convex, which means the Hessian is positive semi-definite everywhere. First we will fix $(X g(X) X^T))^{-1}$ to be the identity matrix, which we can do without loss of generality by scaling $g$ by a constant. This scaling will not affect the convexity of $\flin$.

Now, we have the simpler form of $\flin = (X^T g(X) Y) \cdot e_j$. We then can compute the values of the Hessian:
\begin{eqnarray*}
  \frac{\partial F^2}{\partial^2 g} &=& (X^T g''(X) Y) \cdot e_j
\end{eqnarray*}

Consider the case where $X^T$ is a $2 \times 2$ matrix with rows $1, 0$ and $0, -1$ and $Y$ is a vector of ones. If $g$ weights the second column (i.e. feature) greater than the first, then the output Hessian will be positive semi-definite. But if $g$ weights the first column greater than the first, then it will be negative semi-definite. Since the Hessian is not positive semi-definite everywhere, $\flin$ must be non-convex over the space of $g$.
\end{proof}

This means the stationary point we converge to via gradient descent may only be locally optimal. In Section~\ref{sec:experiments}, we optimize Equation~\ref{eq:full_lin} using the ADAM optimizer \cite{adam}. Additional details about implementation and parameter selection are in Appendix~\ref{sec:details_app}. Despite only locally optimal guarantees, we were still able to find (feature, subgroup) pairs with high $\linfid$ for all datasets.

\section{Experimental Details}
\label{sec:details_app}

\subsection{Algorithmic Details}
\label{subsec:alg_details}
\textbf{Separable Case}.  In order to implement Algorithm~\ref{alg:cap} over a range of $[\alpha_L, \alpha_U]$ values, we need to specify our dual norm  $B$, learning rate $\eta$, number of iterations used $T$, rich subgroup class $\mathcal{G}$, and the associated oracle $\cscg$. We note that for each feature $f_j$, Algorithm~\ref{alg:cap} is run twice; one corresponding to maximizing $\fid(f_j, g, h)$ and the other minimizing it. Note that in both cases our problem is a minimization, but when maximizing we simply negate all of the point wise feature importance values $F(f_j, x_i, h) \to -F(f_j, x_i, h)$. In all experiments our subgroup class $\mathcal{G}$ consists of linear threshold functions over the sensitive features: 
$\mathcal{G} = \{\theta \in \mathbb{R}^{d_{sens}}: \theta((x,x')) = \textbf{1}\{\theta'x > 0\}$. We implement $\cscg$ as in \cite{msr, gerry} via linear regression, see Algorithm~\ref{alg:csc} in Appendix~\ref{sec:cscg_app}.
To ensure the dual player's response is strong enough to enforce desired size constraints, we empirically found that setting the hyperparameter $B=10^4 \cdot \mu(f_j)$ worked well on all datasets, where $\mu(f_j)$ is the average absolute importance value for feature $j$ over $X$. We set the learning rate for exponentiated gradient descent to $\eta=10^{-5}$. Empirical testing showed that $\eta \cdot B$ should be on the order of $\mu(f_j)$ or smaller to ensure proper convergence. We found that setting the error tolerance hyperparameter $\nu=.05 \cdot \mu(f_j) \cdot n \cdot \alpha_L$ worked well in ensuring good results with decent convergence time across all datasets and values of $\alpha$. For all datasets and methods we ran for at most $T = 5000$ iterations, which we observe empirically was large enough for $\fid$ values to stabilize and for $\frac{1}{T}\sum_{t = 1}^{T}|g_t| \in [\alpha_L, \alpha_U]$, with the method typically converging in $T = 3000$ iterations or less. See Appendix~\ref{sec:opt_app} for a sample of convergence plots.

\textbf{Non-Separable Case}. For the non-separable approach, datasets were once again split into train and test sets. For Student, it was split 50-50, while COMPAS, Bank, and Folktables were split 80-20 train/test. The 50-50 split for Student was chosen so that a linear regression model would be properly fit on a small $g(X_{test})$. The parameter vector $\theta$ for a logistic regression classifier was randomly initialized with a PyTorch random seed of $0$ for reproducability. We used an ADAM \cite{adam} optimizer with a learning rate of $.05$ as our heuristic solver for the loss function.

To enforce subgroup size constraints, $\lambda_{s}P_{size}$ must be on a significantly larger order than $\lambda_{c} \flin (j,g)$. Empirical testing found that values of $\lambda_{s} = 10^5$ and $\lambda_{c}=10^{-1}$ returned appropriate subgroup sizes and also ensured smooth convergence. The optimizer ran until it converged upon a minimized linear regression coefficient, subject to the size constraints. Experimentally, this took at most $1000$ iterations, see Appendix~\ref{sec:conv_app} for a sample of convergence plots. After solving twice for the minimum and maximum $\flin (j,g)$ values and our subgroup function $g$ is chosen, we fit the linear regression on both $X_{test}$ and $g(X_{test})$ to get the final \fid.

\subsection{$\fid$ Notions}
\label{subsec:fid_notions}
\textbf{LIME}: A random forest model $h$ was trained on dataset $X^n$. Then each data point along with the corresponding probability outputs from the classifier were input into the \href{https://lime-ml.readthedocs.io/en/latest/lime.html}{LIME Tabular Explainer} Python module. This returned the corresponding LIME explanation values.

\textbf{SHAP}: This was done with the same method as LIME, except using the \href{https://shap.readthedocs.io/en/latest/}{SHAP Explainer} Python module.

\textbf{Vanilla Gradient}: Labeled as \textit{GRAD} in charts, the vanilla gradient importance notion was computed using the Gradient method from the \href{https://open-xai.github.io/}{OpenXAI} library \cite{agarwal2022openxai}. This notion only works on differentiable classifiers so in this case, $h$ is a logistic regression classifier. We found there was no substantial difference between the choice of random forest or logistic regression for $h$ when tested on other importance notions (Appendix~\ref{subsec:valid}). Due to constraints on computation time, this method was only tested on the COMPAS dataset (using \texttt{Two Year Recidivism} as the target variable).

\textbf{Linear Regression}: For the linear regression notion, the subgroup $g$ was chosen to be in the logistic regression hypothesis class. For a given subgroup $g(X)$, the weighted least squares (WLS) solution is found whose linear coefficients $\theta_g$ then define the feature importance value $e_j \cdot \theta_g$.

For details on the consistency of these importance notions, see Appendix~\ref{sec:gen_app}.

\subsection{Datasets}
\label{subsec:datasets}

These four datasets were selected on the basis of three criterion: (i) they all use features which could be considered \emph{sensitive} to make predictions about individuals in a context where bias in a significant concern (ii) they are heavily used datasets in research on interpretability and fairness, and as such issues of bias in the datasets should be of importance to the community, and (iii) they trace out a range of number of datapoints and number of features and sensitive features, which we summarise in Table~\ref{tab:datasets}. For each dataset, we specified features that were "sensitive." That is, when searching for subgroups with high $\fid$, we only considered rich subgroups defined by features generally covered by equal protection or privacy laws (e.g. race, gender, age, health data).

\textbf{Student}: This dataset aims to predict student performance in a Portugese grade school using demographic and familial data. For the purposes of this experiment, the target variable was math grades at the end of the academic year. Student was by far the smallest of the four datasets with 395 data points. The sensitive features in Student are \texttt{gender}, \texttt{parental status}, \texttt{address} (urban or rural), \texttt{daily alcohol consumption}, \texttt{weekly alcohol consumption}, and \texttt{health}. \texttt{Age} typically would be considered sensitive but since in the context of school, age is primarily an indicator of class year, this was not included as a sensitive feature. The categorical features \texttt{address}, \texttt{Mother's Job}, \texttt{Father's Job}, and \texttt{Legal Guardian} were one hot encoded.

\textbf{COMPAS}: This dataset uses a pre-trial defendant's personal background and past criminal record to predict risk of committing new crimes. To improve generalizability, we removed any criminal charge features that appeared fewer than 10 times. Binary counting features (e.g. \texttt{25-45 yrs old} or \texttt{5+ misdemeanors}) were dropped in favor of using the continuous feature equivalents. Additionally, the categorical variable \texttt{Race} was one-hot encoded. This brought the total number of features to 95. The sensitive features in COMPAS are \texttt{age}, \texttt{gender}, and \texttt{race} (Caucasian, African-American, Asian, Hispanic, Native American, and Other). For COMPAS, we ran all methodologies twice, once using the binary variable, \texttt{Two Year Recidivism}, as the target variable and once using the continuous variable \texttt{Decile Score}. \texttt{Two Year Recidivism} is what the model is intended to predict and is labeled as \textit{COMPAS R} in the results. Meanwhile, \texttt{Decile Score} is what the COMPAS system uses in practice to make recommendations to judges and is labeled as \textit{COMPAS D} in the results.

\textbf{Bank}: This dataset looks at whether a potential client signed up for a bank account after being contacted by marketing personnel. The sensitive features in Bank are \texttt{age} and \texttt{marital status} (married, single, or divorced). The \texttt{age} feature in Bank is a binary variable representing whether the individual is above the age of 25. 

\textbf{Folktables}: This dataset is derived from US Census Data. Folktables covers a variety of tasks, but we used the ACSIncome task, which predicts whether an individual makes more than \$50k per year. The ACSIncome task is meant to mirror the popular Adult dataset, but with modifications to address sampling issues. For this paper, we used data from the state of Michigan in 2018. To reduce sparseness of the dataset, the \texttt{Place of Birth} feature was dropped and the \texttt{Occupation} features were consolidated into categories of work as specified in the official Census dictionary \cite{pums}, (e.g. people who work for the US Army, Air Force, Navy, etc. were all consolidated into \texttt{Occupation=Military}). The sensitive features in Folktables are \texttt{age}, \texttt{sex}, \texttt{marital status} (married, widowed, divorced, separated, never married/under 15 yrs old), and \texttt{race} (Caucasian, African-American, Asian, Native Hawaiian, Native American singular tribe, Native American general, Other, and 2+ races).

\begin{table}[h]
\centering
  \caption{Summary of Datasets}
  \label{tab:datasets}
  \begin{tabular}{ccccl}
    \toprule
    Dataset & Data Points & \# of Features & \# of Sensitive Features \\
    \midrule
    Student & 395 & 32 & 6 \\
    COMPAS & 6172 & 95 & 8 \\
    Bank & 30488 & 57 & 4 \\
    Folktables Income & 50008 & 52 & 16 \\
    \bottomrule
  \end{tabular}
\end{table}

\section{Synthetic Experiment}
\label{sec:synthetic_experiment}
In addition to the empirical experiments on real-world datasets, we generated two synthetic datasets and used them to validate our methods in a controlled environment. In our baseline experiment, we created a dataset where the outcome $y$ is independent of the sensitive features to confirm that our algorithm does not result in any false discovery. Next, we modified the distribution of the outcome for a subset of individuals, injecting a large $\fid$ in the subgroup $g$ for feature $f_j$. We then confirmed that our algorithm is able to find that feature importance disparity. We discuss the dataset generation, experimental setup, and results from those two experiments here.

\subsection{Baseline Case}
\textbf{Experimental Setup:} We generated a synthetic dataset of size $n=4000$. Each person in the dataset had randomly generated sensitive features: \texttt{age}, \texttt{sex}, and \texttt{race}. \texttt{sex} and \texttt{race} were drawn based on US Census data and $\texttt{age} \sim \mathcal{N}(50,7)$. Three more variables were generated for each individual: a binary variable $x_1$ and normally distributed variables $x_2 \sim \mathcal{N}(100,5)$ and $x_3 \sim \mathcal{N}(100,5)$. These three additional variables were drawn independently of the sensitive features and each other. We then generated outcome $y \sim x_1 + x_2 + x_3 + \eta, \eta \sim \mathcal{N}(0,1)$; note that $y$ is generated from the same model for any sensitive group, so there should be no subgroups with large $\avgfid$. We then trained a random forest model on this dataset and computed feature importance values using SHAP.

\textbf{Results:} We summarize the ($f_j$, $g$) pairs with the largest $\avgfid$ in Table~\ref{tab:synth_summary}. As expected, we see that Algorithm~\ref{alg:cap} does not find any significant $\avgfid$ for this baseline case. $\avgfid$ is not exactly zero, which is expected because $\avgfid$ is measured as an absolute value, meaning that any difference in feature importance due to random variation will result in a non-zero value. 

\subsection{Injected Case}
\textbf{Experimental Setup:} In our second experiment, we started with the same individuals generated in the baseline case but we injected FID for a subgroup $g$ of older, hispanic individuals consisting of approximately $13\%$ of the population. For this second dataset, if an individual was in $g$, then we generated $y \sim x_1 + 50*x_2 + x_3 + \eta, \eta \sim \mathcal{N}(0,1)$. Otherwise, $y$ was generated as in the baseline case. Also as in the baseline scenario, we used a random forest model with SHAP as the feature importance notion.

\textbf{Results:} As seen in Table~\ref{tab:synth_summary}, the three features with the highest FID subgroups found were \texttt{Hispanic}, \texttt{age}, and $x_2$. Finding $x_2$ was expected, but it is not unusual for \texttt{Hispanic} and \texttt{age} to also be found since in our synthetic example, $y$ is dependent on these two features for points in $g$ and completely independent for $X / g$. The fourth largest $\avgfid$ found was significantly smaller than the top three and is comparable in magnitude to that of the baseline case. The subgroup found by Algorithm~\ref{alg:cap} for the top feature captured $64\%$ of the older Hispanic subgroup where the disparity was injected. This is not a perfect result, but was obtained without extensive tuning and illustrates our method can detect an injected disparity in a controlled environment. 

\begin{table*}[t]
\centering
  \caption{Summary of the top $(f_j,g)$ pairs found for the two synthetic dataset experiments. $\mu(F)$ is the average feature importance value on the specified group. We can see that in the baseline experiment, there was very little $\avgfid$. In the injected case, Algorithm~\ref{alg:cap} found very large $\avgfid$ subgroups on the three features which were effected by the injection. The next largest pair in the injected case had $\avgfid$ comparable to the baseline case.}
  \label{tab:synth_summary}
  \scalebox{.8}{
  \begin{tabular}{llcccc}
    \toprule
    Experiment & Feature $f_j$ & $|g|$ & $\mu(F(f_j, X))$ & $\mu(F(f_j, g))$ & $\avgfid$ \\
    \midrule
    Baseline & $x_3$ & $.13$ & $-.61$ & $-.48$ & $\mathbf{.13}$\\
             & $x_2$ & $.11$ & $-1.16$ & $-1.23$ & $\mathbf{.06}$\\
             & $x_1$ & $.12$ & $-.05$ & $-.01$ & $\mathbf{.04}$\\
    \midrule
    Injected & Hispanic-American & $.15$ & $.86$ & $28.0$ & $\mathbf{27.13}$\\
             & Age & $.17$ & $3.57$ & $-7.07$ & $\mathbf{10.6}$\\
             & $x_2$ & $.15$ & $-1.56$ & $.22$ & $\mathbf{1.78}$\\
             & Black-American & $.15$ & $.01$ & $.04$ & $\mathbf{.03}$\\
  \bottomrule
\end{tabular}}
\end{table*}

\section{Comparison of $\fid$ Values on Rich vs. Marginal Subgroups}
\label{sec:rich_marg}

This appendix provides expanded information from Section~\ref{sec:rich_marg_main}. Here we are justifying the use of rich subgroups by searching for maximal $\avgfid$ subgroups in the marginal subgroup space. Marginal subgroups are those defined by a single sensitive characteristic making them straightforward to search. In Figure~\ref{fig:rich_marg}, we compare the maximal $\avgfid$ rich subgroups shown in Figure~\ref{fig:all_de} to the maximal $\avgfid$ marginal subgroup for the same feature. In about half the cases, expanding our subgroup classes to include rich subgroups defined by linear functions of the sensitive attributes enabled us to find a subgroup that had a higher $\avgfid$. In the other cases, the $\avgfid$ of the marginal subgroup was similar to the rich subgroup. Sometimes, the marginal subgroup outperformed the rich subgroup; this happens when using rich subgroups does not offer any substantial advantage over marginal subgroups, and the empirical error tolerance in Algorithm~\ref{alg:cap} stopped the convergence early.


\begin{figure}
    \centering
    \begin{subfigure}{.25\textwidth}
        \centering
        \includegraphics[width=\linewidth]{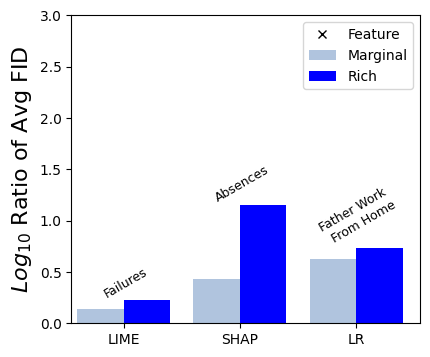}
        \caption{Student}
        \label{fig:rich_marg_sub1}
    \end{subfigure}
    \begin{subfigure}{.315\textwidth}
        \centering
        \includegraphics[width=\linewidth]{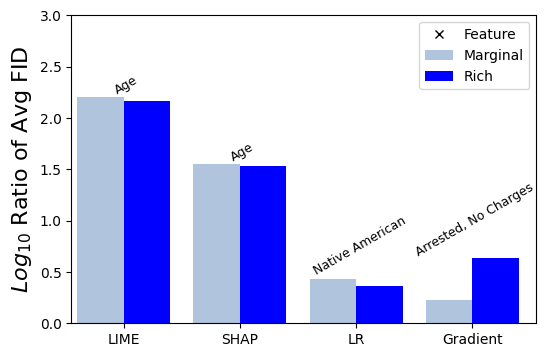}
        \caption{COMPAS R}
        \label{fig:rich_marg_sub2}
    \end{subfigure}
    \vskip .2in
    \begin{subfigure}{.25\textwidth}
        \centering
        \includegraphics[width=\linewidth]{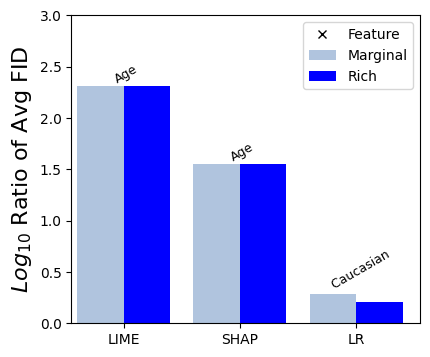}
        \caption{COMPAS D}
        \label{fig:rich_marg_sub3}
    \end{subfigure}
    \begin{subfigure}{.25\textwidth}
        \centering
        \includegraphics[width=\linewidth]{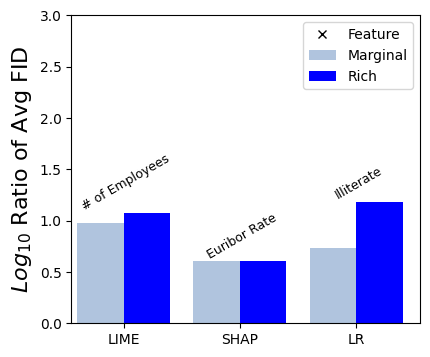}
        \caption{Bank}
        \label{fig:rich_marg_sub4}
    \end{subfigure}
    \begin{subfigure}{.25\textwidth}
        \centering
        \includegraphics[width=\linewidth]{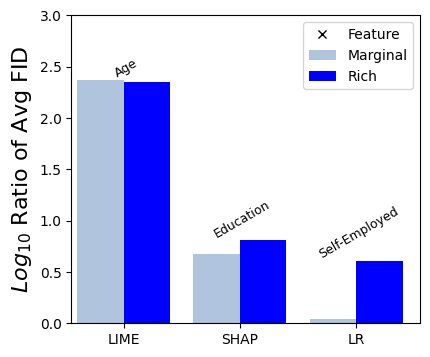}
        \caption{Folktables}
        \label{fig:rich_marg_sub5}
    \end{subfigure}
\caption{Comparison of the maximal $\fid$ rich subgroups from Figure~\ref{fig:all_de} to the maximal $\fid$ marginal subgroup on the same feature. This is displayed as $\lvert log_{10}(R) \rvert$ where $R$ is the ratio of average importance per data point for separable notions and the ratio of coefficients for the linear coefficient notion. The feature associated with the subgroups is written above each bar.}
\label{fig:rich_marg}
\end{figure}


\newpage

\section{Statistical Validity of Results: Generalization of $\fid$ and $|g|$}
\label{subsec:valid}

When confirming the validity of our findings, there are two potential concerns: (1) Are the subgroup sizes found in-sample approximately the same on the test set and (2) do the $\fid$'s found on the training set generalize out of sample? Taken together, (1) and (2) are sufficient to guarantee our maximal $\avgfid$ values generalize out of sample.

In Figure~\ref{fig:size}, we can see that when we take the maximal subgroup found for each feature $f_j$, $g_j^{*}$, and compute it's size $|g_j^{*}|$ on the test set, for both the separable and non-separable methods it almost always fell within the specified $[\alpha_L, \alpha_U]$ range; the average difference in $|g_j^{*}(X_{train})|$ and $|g_j^{*}(X_{test})|$ was less than $.005$ on all notions of feature importance and all datasets except for Student, which was closer to $.025$ due to its smaller size. A few rare subgroups were significantly outside the desired $\alpha$ range, which was typically due to the degenerate case of the feature importance values all being $0$ for the feature in question. Additional plots for all (dataset, notion) pairs are in Figures~\ref{fig:size_app}, \ref{fig:size_gen_app}. 

\begin{figure}
    \centering
    \begin{subfigure}{.35\textwidth}
        \centering
        \includegraphics[width=.85\linewidth]{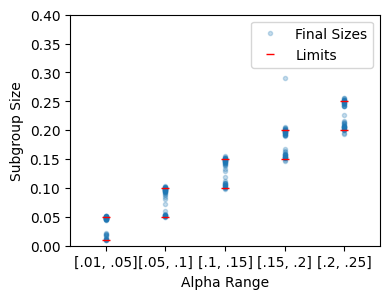}
        \caption{}
        \label{fig:size_sub1}
    \end{subfigure}
    \begin{subfigure}{.35\textwidth}
        \centering
        \includegraphics[width=.85\linewidth]{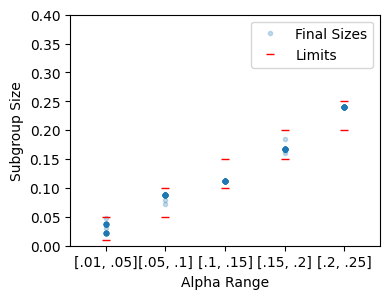}
        \caption{}
        \label{fig:size_sub2}
    \end{subfigure}
    \begin{subfigure}{.35\textwidth}
        \centering
        \includegraphics[width=.85\linewidth]{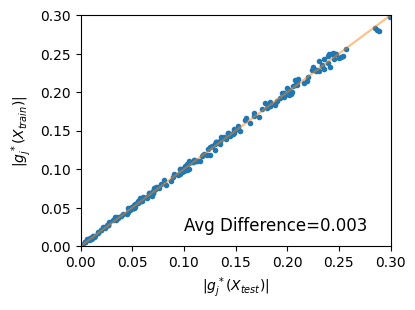}
        \caption{}
        \label{fig:size_sub3}
    \end{subfigure}
    \begin{subfigure}{.35\textwidth}
        \centering
        \includegraphics[width=.85\linewidth]{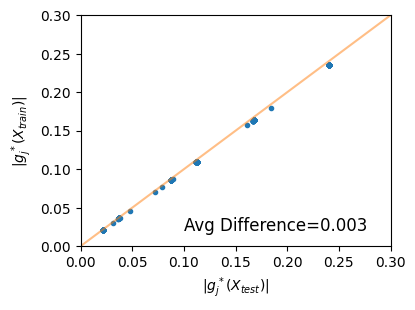}
        \caption{}
        \label{fig:size_sub4}
    \end{subfigure}
\caption{Generalizability of $|g|$ on the Folktables dataset. (a) Size outputs from Algorithm~\ref{alg:cap} for all features and separable notions and (b) from optimizing Equation~\ref{eq:full_lin} for $\linfid$ show that our size constraints hold in-sample. (c) Plots the corresponding values of $|g_j^{*}(X_{train})|$ vs $|g_j^{*}(X_{test})|$ for separable notions and (d) for $\linfid$, showing that the subgroup size generalizes out of sample.}
\label{fig:size}
\end{figure}

\begin{figure}
    \centering
    \begin{subfigure}{.35\textwidth}
        \centering
        \includegraphics[width=.85\linewidth]{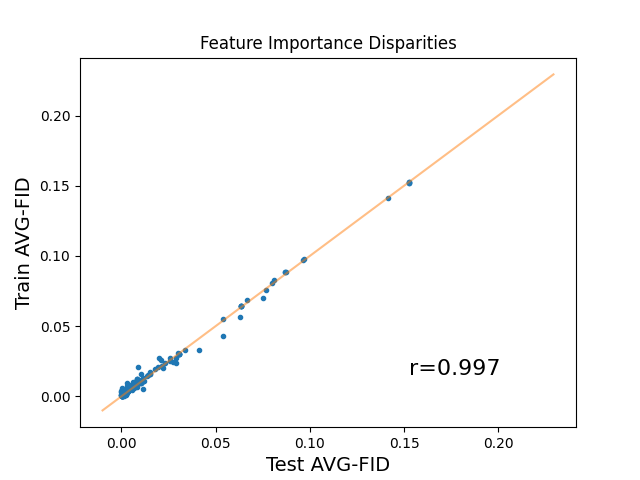}
        \caption{LIME}
        \label{fig:fid_gen_sub1}
    \end{subfigure}
    \begin{subfigure}{.35\textwidth}
        \centering
        \includegraphics[width=.85\linewidth]{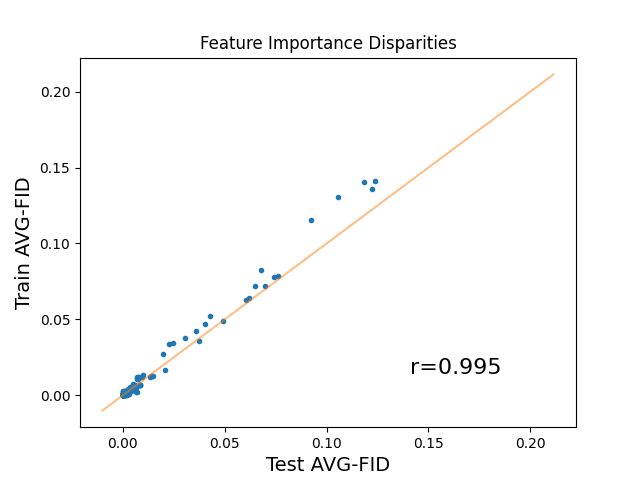}
        \caption{SHAP}
        \label{fig:fid_gen_sub2}
    \end{subfigure}
    \begin{subfigure}{.35\textwidth}
        \centering
        \includegraphics[width=.85\linewidth]{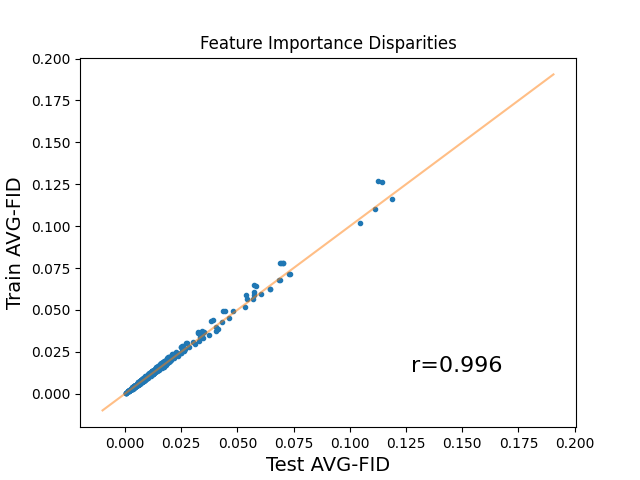}
        \caption{GRAD}
        \label{fig:fid_gen_sub3}
    \end{subfigure}
    \begin{subfigure}{.35\textwidth}
        \centering
        \includegraphics[width=.85\linewidth]{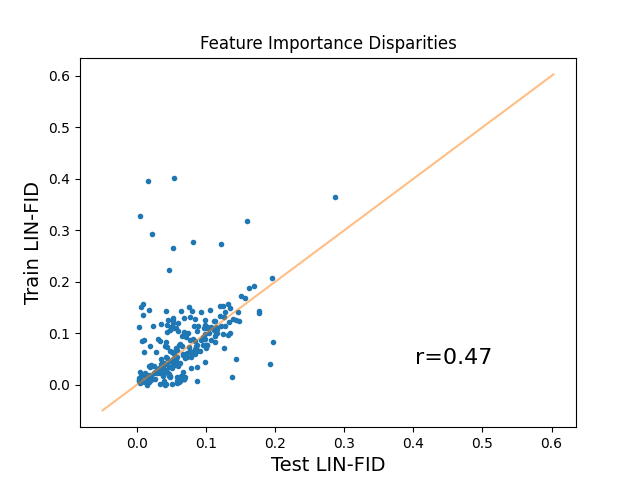}
        \caption{Linear Regression}
        \label{fig:fid_gen_sub4}
    \end{subfigure}
\caption{Out of sample generalization of the methods. Each dot represents a feature, plotting $\fid$ on $X_{test}$ vs on $X_{train}$. All are computed on the Folktables dataset except (c) is computed on COMPAS R. The diagonal line represents perfect generalization and the Pearson correlation coefficient is displayed in figure. The non-separable approach suffers from the instability of the WLS method.}
\label{fig:fid_gen}
\end{figure}

In Figure~\ref{fig:fid_gen}, we compare $\avgfid(f_j, g_j^{*}, X_{train})$ to $\avgfid(f_j, g_j^{*}, X_{test})$, or $\linfid$ in the case of the linear regression notion, to see how $\fid$ generalizes. The separable notions all generalized very well, producing very similar $\avgfid$ values for in and out of sample tests. The non-separable method still generalized, although not nearly as robustly, with outlier values occurring. This was due to ill-conditioned design matrices for small subgroups leading to instability in fitting the least squares estimator. In Appendix~\ref{sec:gen_app}, we investigate the robustness of the feature importance notions, evaluated on the entire dataset. We find that the coefficients of linear regression are not as stable, indicating the lack of generalization in Figure~\ref{fig:fid_gen} could be due to the feature importance notion itself lacking robustness, rather than an over-fit selection of $g_j^{*}$.

\begin{figure}[H]
\centering
\begin{subfigure}{.3\textwidth}
  \centering
  \includegraphics[width=\linewidth]{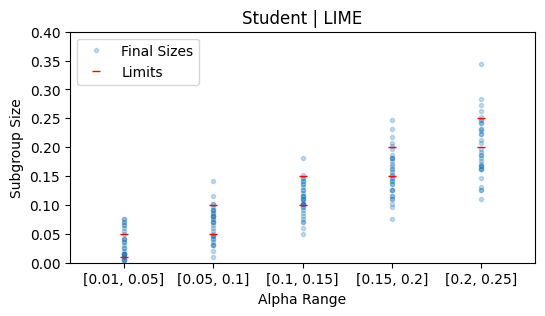}
\end{subfigure}%
\begin{subfigure}{.3\textwidth}
  \centering
  \includegraphics[width=\linewidth]{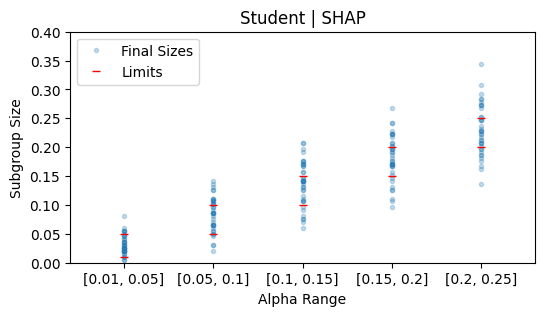}
\end{subfigure}
\begin{subfigure}{.3\textwidth}
  \centering
  \includegraphics[width=\linewidth]{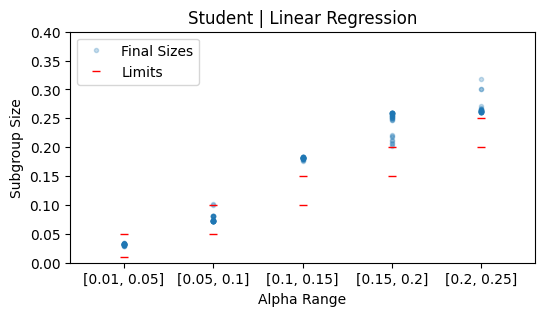}
\end{subfigure} \\

\begin{subfigure}{.3\textwidth}
  \centering
  \includegraphics[width=\linewidth]{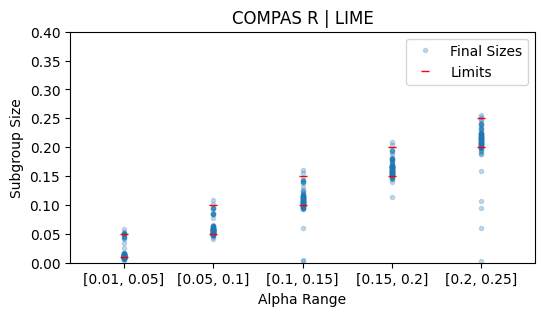}
\end{subfigure}
\begin{subfigure}{.3\textwidth}
  \centering
  \includegraphics[width=\linewidth]{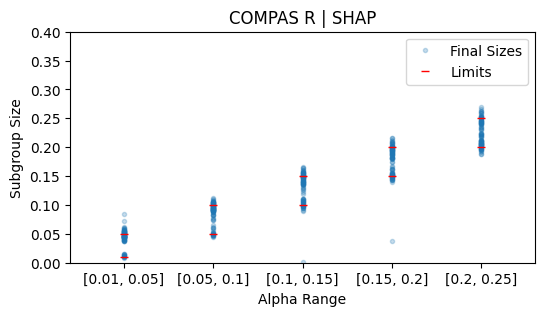}
\end{subfigure}%
\begin{subfigure}{.3\textwidth}
  \centering
  \includegraphics[width=\linewidth]{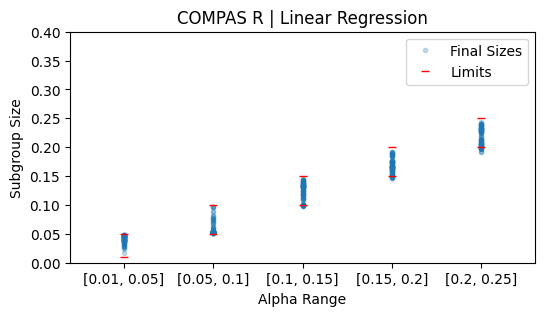}
\end{subfigure} \\

\begin{subfigure}{.3\textwidth}
  \centering
  \includegraphics[width=\linewidth]{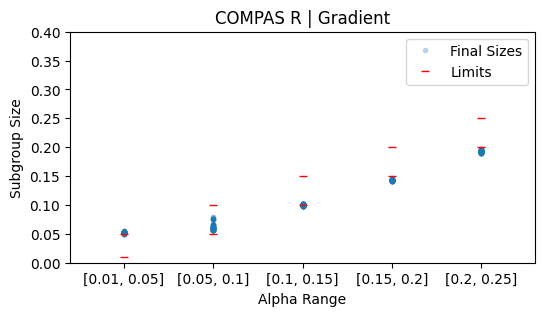}
\end{subfigure}
\begin{subfigure}{.3\textwidth}
  \centering
  \includegraphics[width=\linewidth]{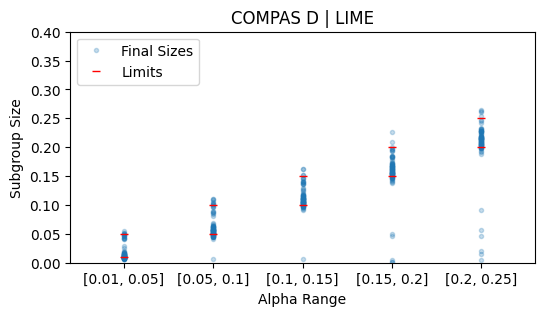}
\end{subfigure}
\begin{subfigure}{.3\textwidth}
  \centering
  \includegraphics[width=\linewidth]{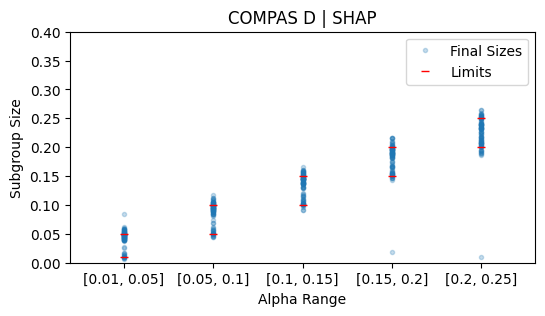}
\end{subfigure} \\

\begin{subfigure}{.3\textwidth}
  \centering
  \includegraphics[width=\linewidth]{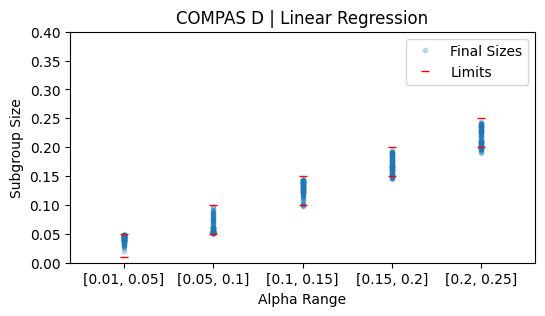}
\end{subfigure}
\begin{subfigure}{.3\textwidth}
  \centering
  \includegraphics[width=\linewidth]{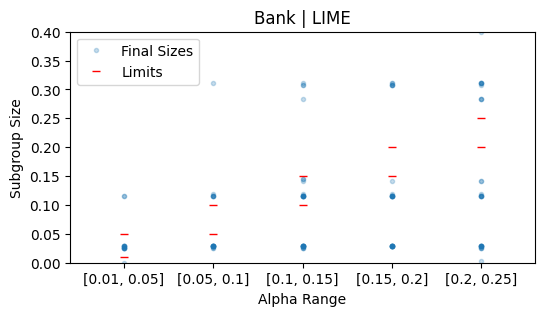}
\end{subfigure}
\begin{subfigure}{.3\textwidth}
  \centering
  \includegraphics[width=\linewidth]{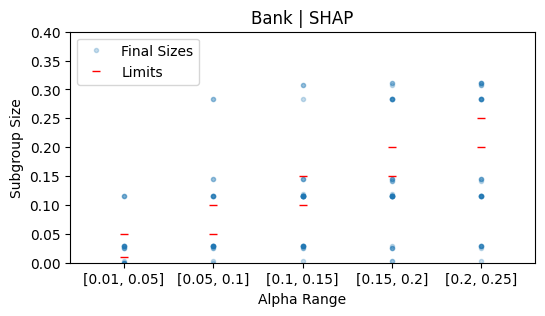}
\end{subfigure} \\

\begin{subfigure}{.3\textwidth}
  \centering
  \includegraphics[width=\linewidth]{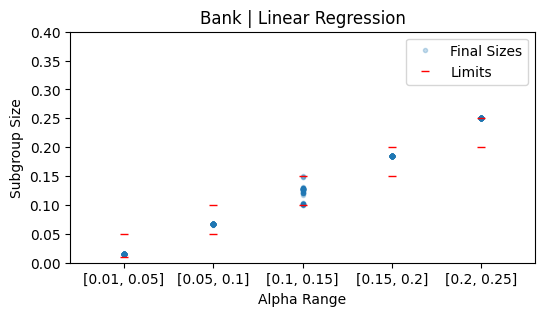}
\end{subfigure}%
\begin{subfigure}{.3\textwidth}
  \centering
  \includegraphics[width=\linewidth]{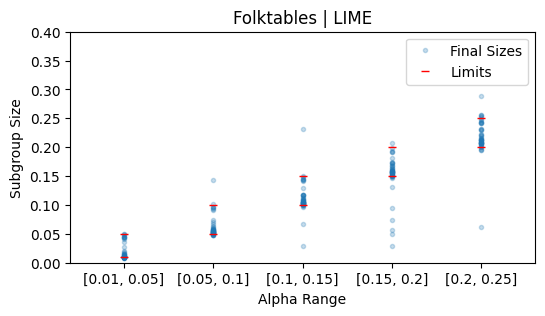}
\end{subfigure}
\begin{subfigure}{.3\textwidth}
  \centering
  \includegraphics[width=\linewidth]{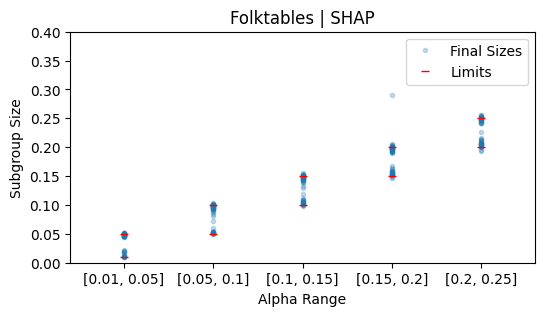}
\end{subfigure} \\

\begin{subfigure}{.3\textwidth}
  \centering
  \includegraphics[width=\linewidth]{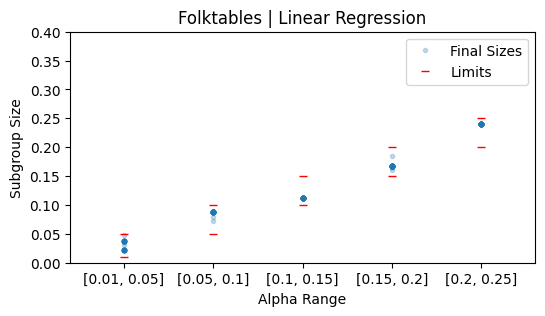}
\end{subfigure}
\caption{Final subgroup sizes of $g(X_{test})$ compared with $\alpha$ range. These almost always fall within the correct size range. Student has the largest errors given the small dataset size. Some subgroups fell significantly outside the expected ranged, mostly due to many importance values, $F(f_j, X, h)$, being zero for a given feature.}
\label{fig:size_app}
\end{figure}

\begin{figure}[H]
\centering
\begin{subfigure}{.245\textwidth}
  \centering
  \includegraphics[width=\linewidth]{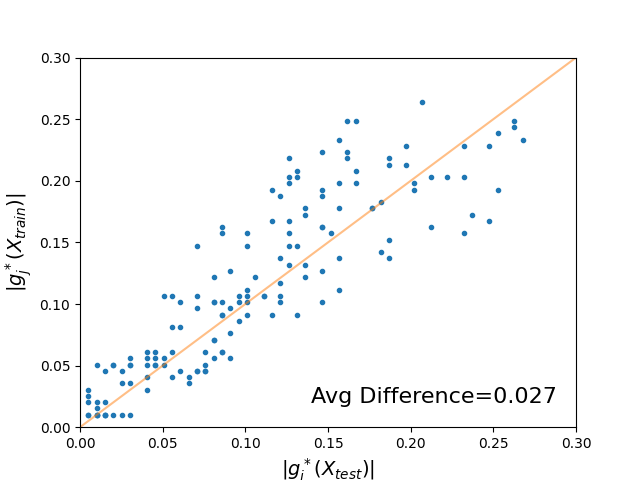}
  \caption{Student, LIME}
\end{subfigure}%
\begin{subfigure}{.245\textwidth}
  \centering
  \includegraphics[width=\linewidth]{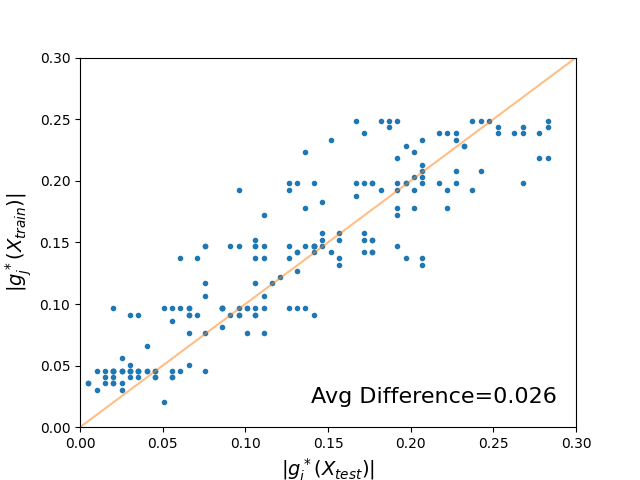}
  \caption{Student, SHAP}
\end{subfigure}
\begin{subfigure}{.245\textwidth}
  \centering
  \includegraphics[width=\linewidth]{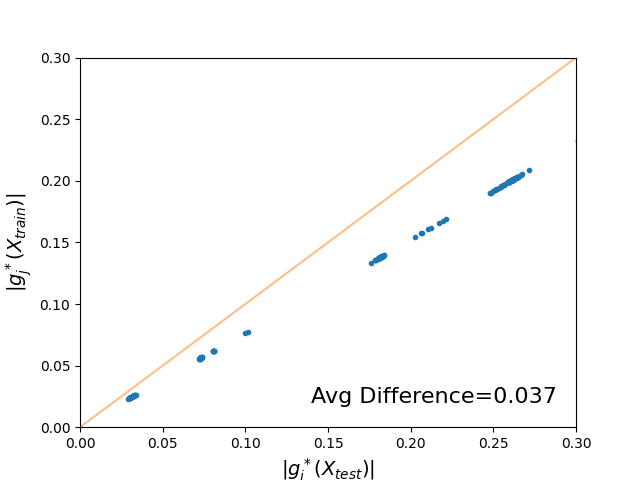}
  \caption{Student, LR}
\end{subfigure}
\begin{subfigure}{.245\textwidth}
  \centering
  \includegraphics[width=\linewidth]{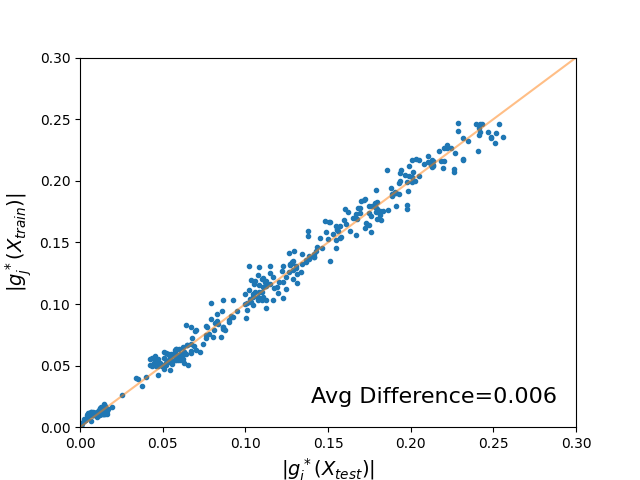}
  \caption{COMPAS R, LIME}
\end{subfigure} \\

\begin{subfigure}{.245\textwidth}
  \centering
  \includegraphics[width=\linewidth]{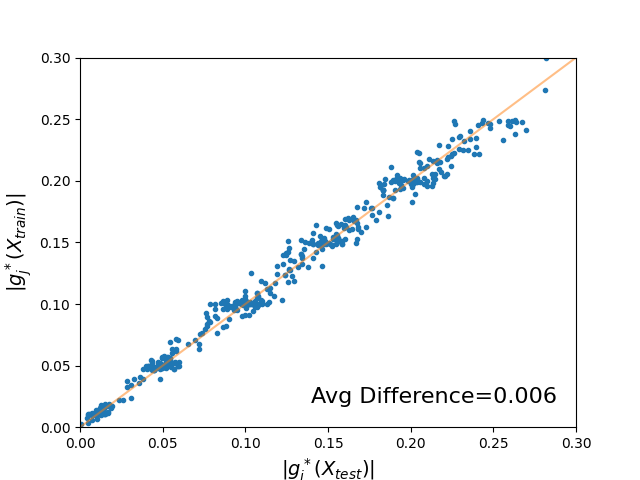}
  \caption{COMPAS R, SHAP}
\end{subfigure}%
\begin{subfigure}{.245\textwidth}
  \centering
  \includegraphics[width=\linewidth]{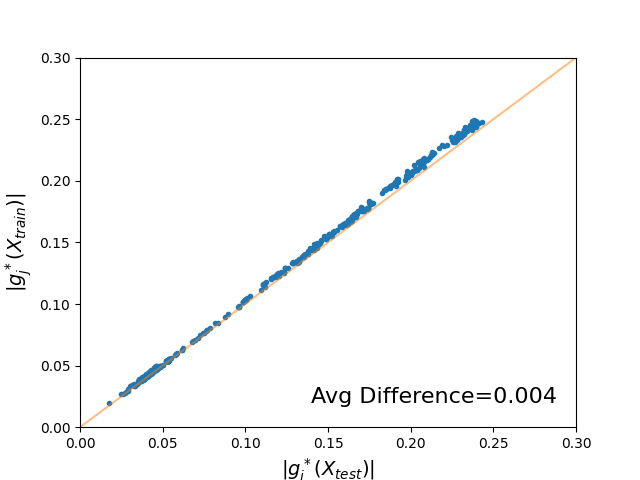}
  \caption{COMPAS R, LR}
\end{subfigure}
\begin{subfigure}{.245\textwidth}
  \centering
  \includegraphics[width=\linewidth]{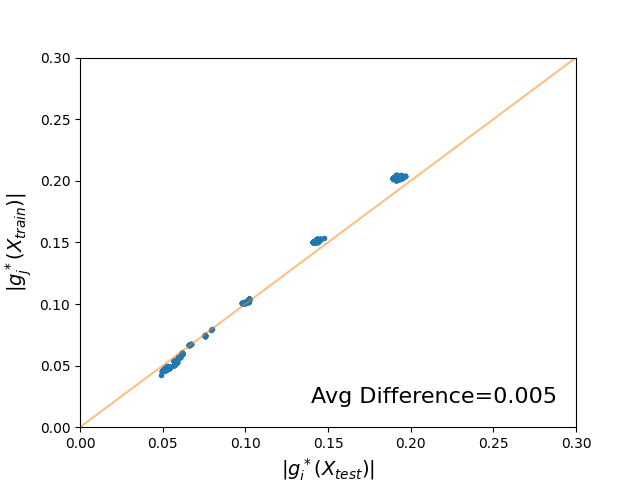}
  \caption{COMPAS R, GRAD}
\end{subfigure}
\begin{subfigure}{.245\textwidth}
  \centering
  \includegraphics[width=\linewidth]{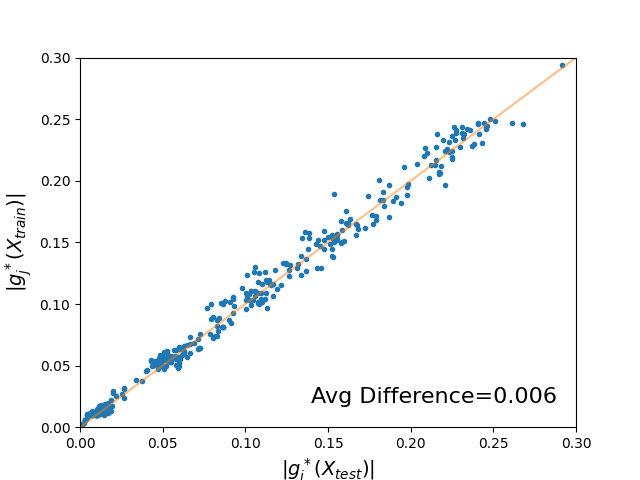}
  \caption{COMPAS D, LIME}
\end{subfigure} \\

\begin{subfigure}{.245\textwidth}
  \centering
  \includegraphics[width=\linewidth]{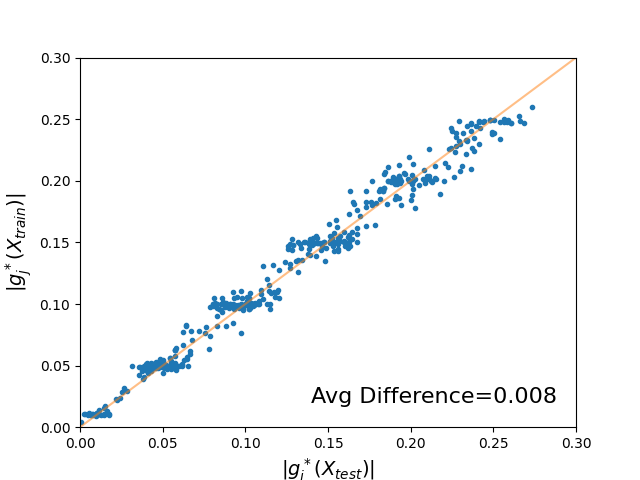}
  \caption{COMPAS D, SHAP}
\end{subfigure}%
\begin{subfigure}{.245\textwidth}
  \centering
  \includegraphics[width=\linewidth]{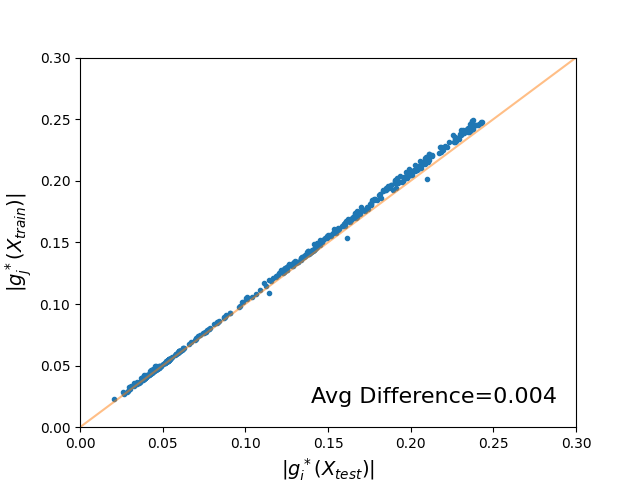}
  \caption{COMPAS D, LR}
\end{subfigure}
\begin{subfigure}{.245\textwidth}
  \centering
  \includegraphics[width=\linewidth]{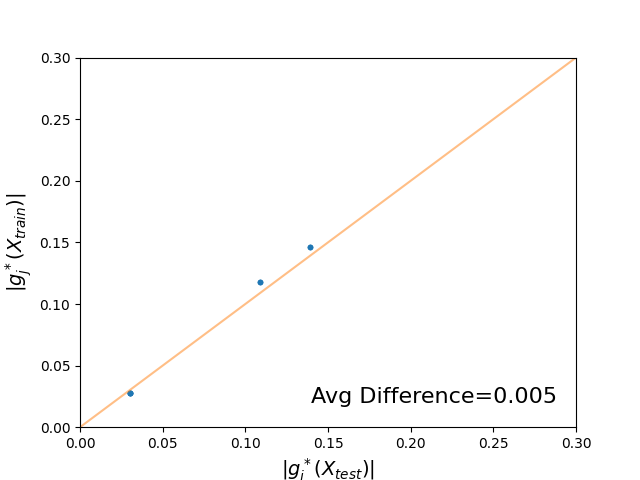}
  \caption{Bank, LIME}
\end{subfigure}
\begin{subfigure}{.245\textwidth}
  \centering
  \includegraphics[width=\linewidth]{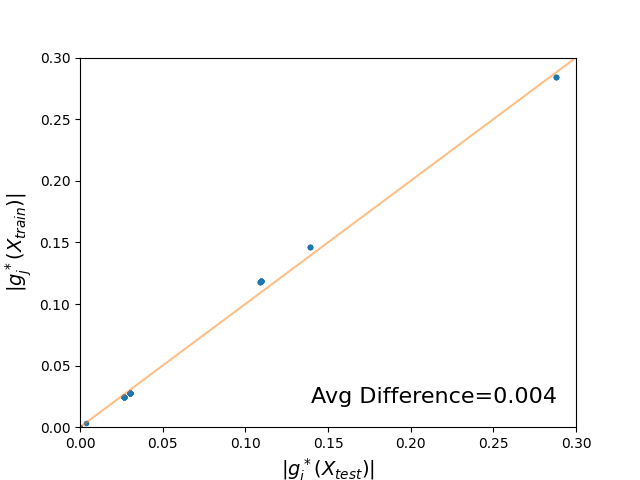}
  \caption{Bank, SHAP}
\end{subfigure} \\

\begin{subfigure}{.245\textwidth}
  \centering
  \includegraphics[width=\linewidth]{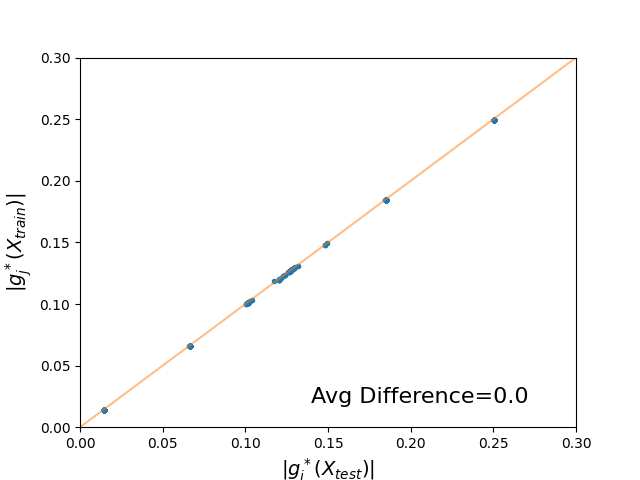}
  \caption{Bank, LR}
\end{subfigure}%
\begin{subfigure}{.245\textwidth}
  \centering
  \includegraphics[width=\linewidth]{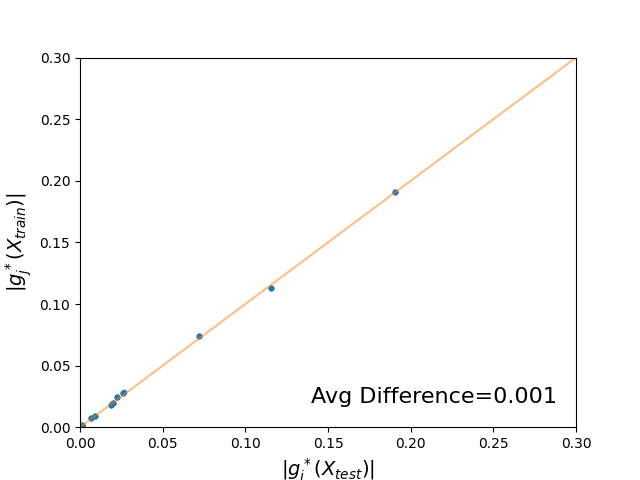}
  \caption{Folk, LIME}
\end{subfigure}
\begin{subfigure}{.245\textwidth}
  \centering
  \includegraphics[width=\linewidth]{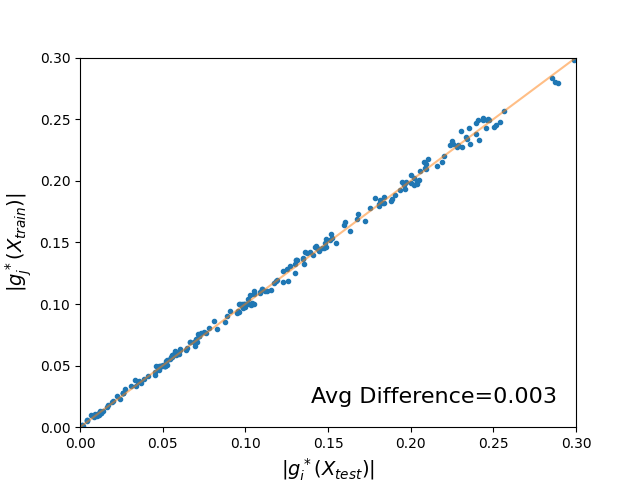}
  \caption{Folk, SHAP}
\end{subfigure}
\begin{subfigure}{.245\textwidth}
  \centering
  \includegraphics[width=\linewidth]{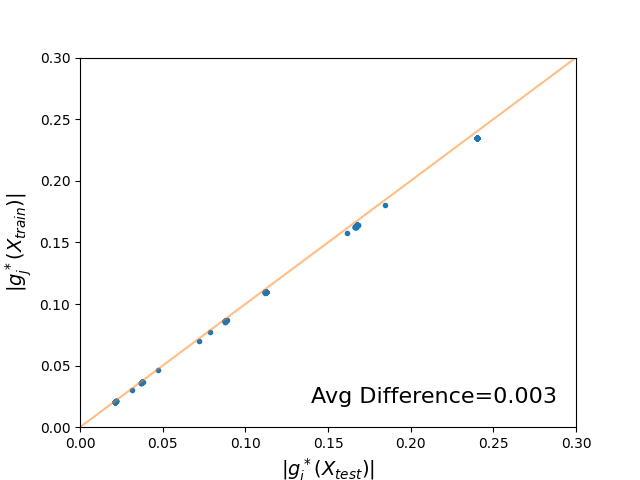}
  \caption{Folk, LR}
\end{subfigure} \\
\caption{Comparing $|g_j^*(X_{train})|$ and $|g_j^*(X_{test})|$. We can see that the size of the subgroup was consistent between the train and test set showing good generalization. The average difference was only noticeable on the Student dataset, due to its smaller size.}
\label{fig:size_gen_app}
\end{figure}

\newpage

\section{Choice of Hypothesis Class}
\label{sec:hypothesis_app}

One ablation study we explored was the choice of classification model $h$. While the main experiments used a random forest model, we also explored using a logistic regression model. The logistic regression model was implemented with the default sklearn hyperparameters. We found that the results are roughly consistent with each other no matter the choice of $h$. In Table~\ref{tab:hypothesis_lime} and Table~\ref{tab:hypothesis_shap}, we see that the features with the highest $\avgfid$, their subgroup sizes, and the $\avgfid$ values are consistent between the choice of hypothesis class. We then looked further into the features that were used to define these subgroups. In Figure~\ref{fig:hypothesis}, we see that the subgroups with high $\avgfid$ for the feature \texttt{Age} were both defined by young, non-Asians. Similarly consistent results were found across all feature importance notions and datasets. 

\begin{figure}[h]
    \centering
    \begin{subfigure}{.45\textwidth}
        \centering
        \includegraphics[width=\linewidth]{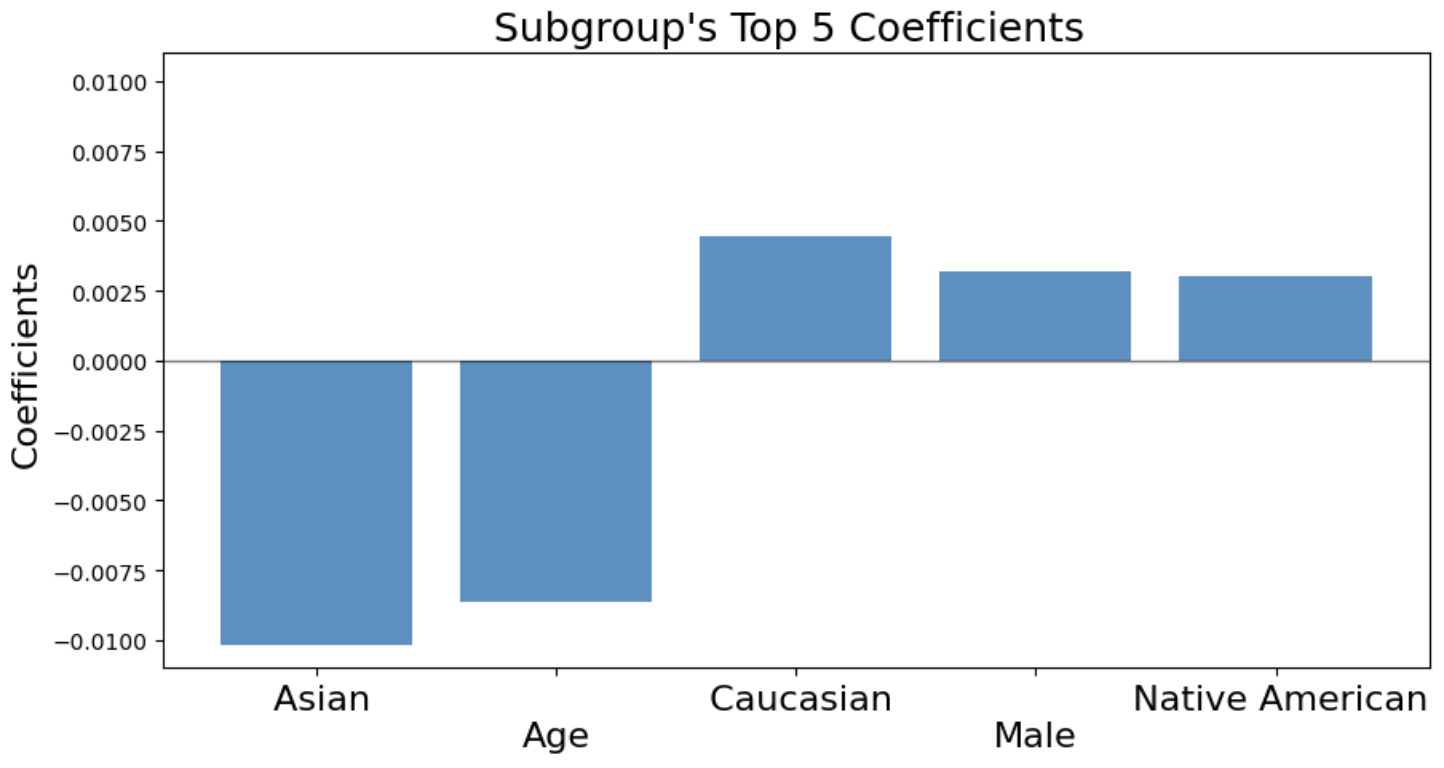}
        \caption{$h=$ Random Forest, $f_{j^*}=$ Age}
        \label{fig:hypothesis_sub2}
    \end{subfigure}
    \begin{subfigure}{.45\textwidth}
        \centering
        \includegraphics[width=\linewidth]{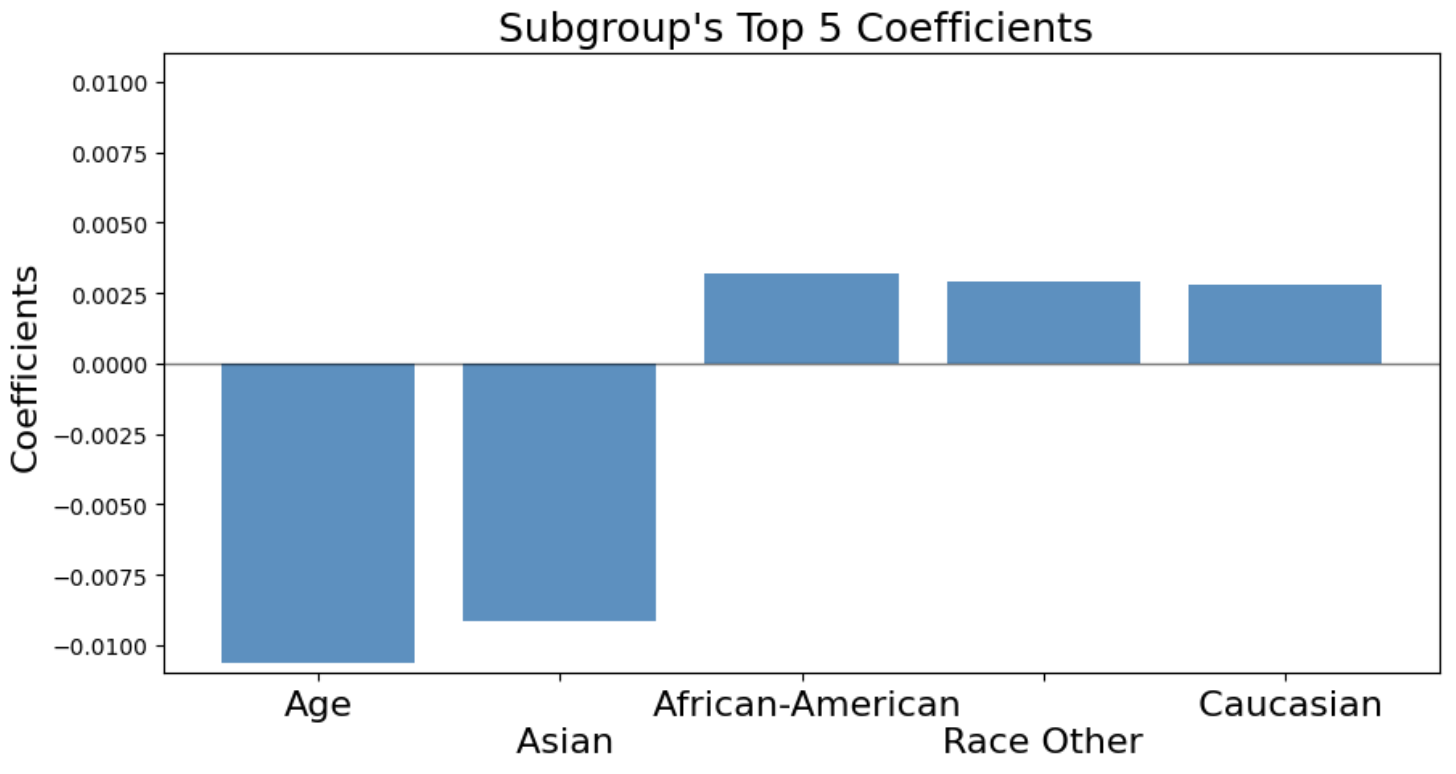}
        \caption{$h=$ Logistic Regression, $f_{j^*}=$ Age}
        \label{fig:hypothesis_sub1}
    \end{subfigure}
\caption{Comparing the choice of hypothesis class of $h$. Here we show the defining coefficients for the highest $\avgfid$ subgroup found on the COMPAS R dataset using LIME as the feature importance notion. For the feature \texttt{Age}, we find that young and non-Asian were the two most defining coefficients for $g^*$ no matter which choice of $h$.}
\label{fig:hypothesis}
\end{figure}

We acknowledge that this ablation study is not the most extensive study, however given that these initial results show no significant differences between choice of hypothesis class, we decided to proceed in the main experiments to present all results using random forest as the hypothesis class. This is certainly an area that could be explored in future studies.

\begin{table}
\caption{Comparing results between using random forest and logistic regression as the hypothesis class for classifier $h$ using LIME as the importance notion on the COMPAS R dataset. Here we display the features with the highest $\avgfid$, the subgroup size $|g|$, and the $\avgfid$. We can see that the choice of hypothesis class $h$ does not substantially affect the output. We used random forest for all of our main experiments.}
\begin{center}
    \begin{tabular}{ |p{2.25cm}p{1.25cm}p{2.15cm}|p{2.25cm}p{1.25cm}p{2.15cm}|  }
         \toprule
         \multicolumn{3}{c}{$h=$ Random Forest} & \multicolumn{3}{c}{$h=$ Logistic Regression} \\
         \midrule
         \hfil Feature &\hfil Size &\hfil $\avgfid$ &\hfil Feature &\hfil Size &\hfil $\avgfid$  \\
         \hline
         Age &\hfil $.05-.1$ &\hfil $.144$ &Age &\hfil $.05-.1$ &\hfil $.21$ \\
         Priors Count &\hfil $.01-.05$ &\hfil $.089$ &Priors count &\hfil $.01-.05$ &\hfil $.092$ \\
         Juv Other Count &\hfil $.01-.05$ &\hfil $.055$ &Juv Other Count &\hfil $.01-.05$ &\hfil $.055$ \\
         Other Features &\hfil - &\hfil $<.025$ &Other Features &\hfil - &\hfil $<.025$ \\
         \bottomrule
    \end{tabular}
    \label{tab:hypothesis_lime}
\end{center}
\end{table}

\begin{table}
    \centering
    \begin{tabular}{ |p{2.25cm}p{1.25cm}p{2.15cm}|p{2.25cm}p{1.25cm}p{2.15cm}|  }
         \toprule
         \multicolumn{3}{c}{$h=$ Random Forest} & \multicolumn{3}{c}{$h=$ Logistic Regression} \\
         \midrule
         \hfil Feature &\hfil Size &\hfil $\avgfid$ &\hfil Feature &\hfil Size &\hfil $\avgfid$ \\
         \hline
         Age &\hfil $.01-.05$ &\hfil $.4$ &Age &\hfil $.01-.05$ &\hfil $.21$ \\
         Priors Count &\hfil $.01-.05$ &\hfil $.11$ &Priors count &\hfil $.01-.05$ &\hfil $.14$ \\
         Other Features &\hfil - &\hfil $<.05$ &Other Features &\hfil - &\hfil $<.05$ \\
         \bottomrule
    \end{tabular}
    \caption{Same as Table~\ref{tab:hypothesis_lime} except using SHAP as the importance notion. With SHAP, there were fewer features with significant $\avgfid$ before dropping off but in both cases, the choice of $h$ did not significantly affect the outcome.}
    \label{tab:hypothesis_shap}
\end{table}

\section{Algorithm ~\ref{alg:cap} Optimization Convergence}
\label{sec:opt_app}

Here are additional graphs showing examples of the convergence of Algorithm~\ref{alg:cap}. Data was tracked every $10$ iterations, recording the Lagrangian values (to compute the error $v_t=max(\lvert L(\hat{p}_{\mathcal{G}}^{t}, \hat{p}_{\lambda}^{t}) - \underline{L} \rvert, \lvert \overline{L} - L(\hat{p}_{\mathcal{G}}^{t}, \hat{p}_{\lambda}^{t})  \vert)$), the subgroup size, and $\avgfid$ value, graphed respectively in Figure~\ref{fig:sep_conv_app}. We can see $\avgfid$ value moving upward, except when the subgroup size is outside the $\alpha$ range, and the Lagrangian error converging upon the set error bound $v$ before terminating.

While Theorem~\ref{alg:cap} states that convergence time may grow quadratically, in practice we found that computation time was not a significant concern. The time for convergence varied slightly based on dataset but for the most part, convergence for a given feature was achieved in a handful of iterations that took a few seconds to compute. Features which took several thousand iterations could take around 30 minutes to compute on larger datasets.

\begin{figure}
\centering
\begin{subfigure}{.4\textwidth}
  \centering
  \includegraphics[width=\linewidth]{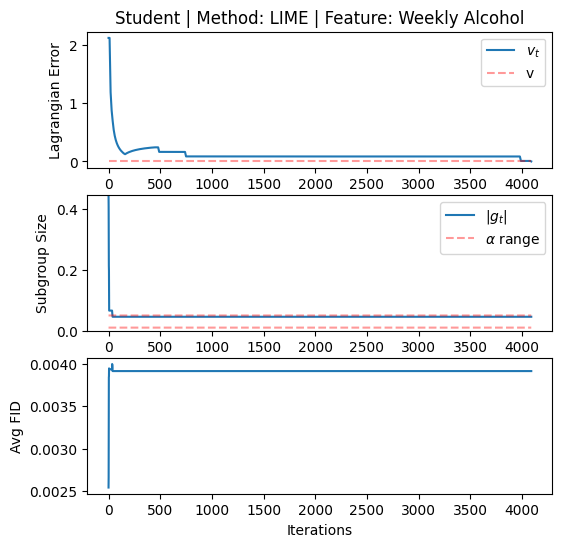}
\end{subfigure}%
\begin{subfigure}{.4\textwidth}
  \centering
  \includegraphics[width=\linewidth]{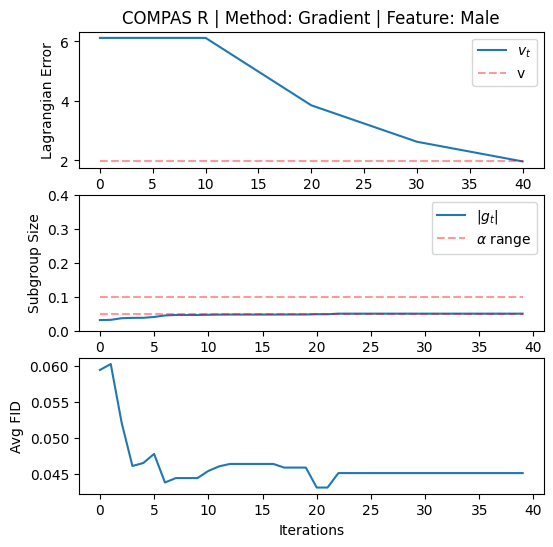}
\end{subfigure} \\
\begin{subfigure}{.4\textwidth}
  \centering
  \includegraphics[width=\linewidth]{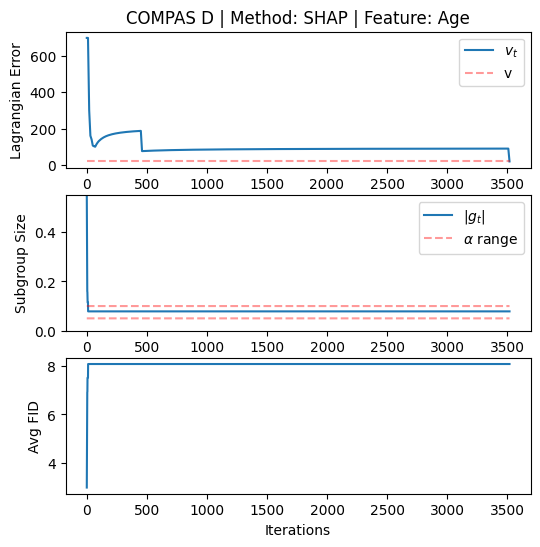}
\end{subfigure}%
\begin{subfigure}{.4\textwidth}
  \centering
  \includegraphics[width=\linewidth]{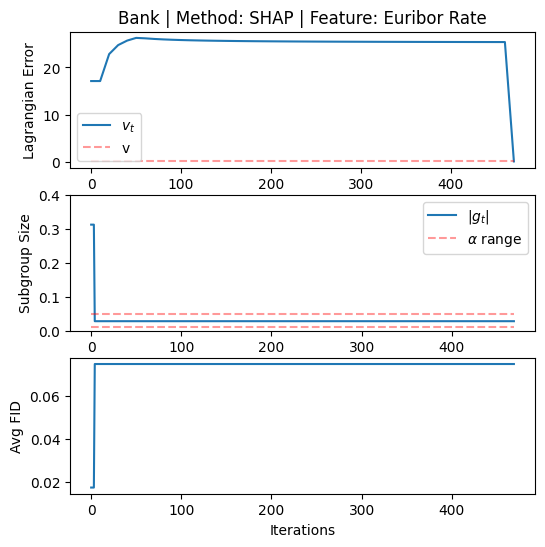}
\end{subfigure} \\
\begin{subfigure}{.4\textwidth}
  \centering
  \includegraphics[width=\linewidth]{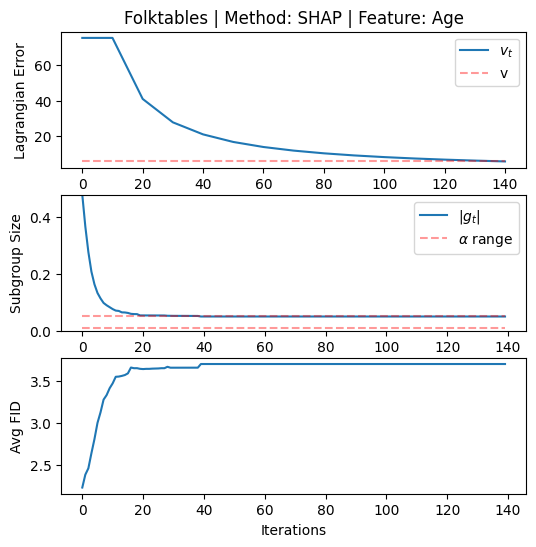}
\end{subfigure}
\caption{Plots detailing the convergence of Algorithm~\ref{alg:cap}. The top plot shows the error convergence, i.e. the max difference in Lagrangian values between our solution and the min/max-players' solution. The other two plots display the subgroup size and $\avgfid$ of the solution. Convergence almost always happened in fewer than $5000$ iterations, allaying concerns about theoretical run time.}
\label{fig:sep_conv_app}
\end{figure}

\newpage

\section{Non-Separable Optimization Convergence}
\label{sec:conv_app}

Here are additional graphs showing the convergence in the non-separable approach. Using the loss function that rewards minimizing the linear regression coefficient (or maximizing it) and having a size within the alpha constraints, we typically reach convergence after a few hundred iterations. In Figure~\ref{fig:conv_app}, we can see in the respective upper graphs that the subgroup size converges to the specified $\alpha$ range and stays there. Meanwhile, in the lower graph, we see the $\linfid$ attempt to maximize but oscillates as the appropriate size is found.

Convergence using this method almost always took $<1000$ iterations. Running this for all features took around 2 hours to compute on the largest datasets. The optimization was run using GPU computing on NVIDIA Tesla V100s.

\begin{figure}[H]
\centering
\begin{subfigure}{.39\textwidth}
  \centering
  \includegraphics[width=\linewidth]{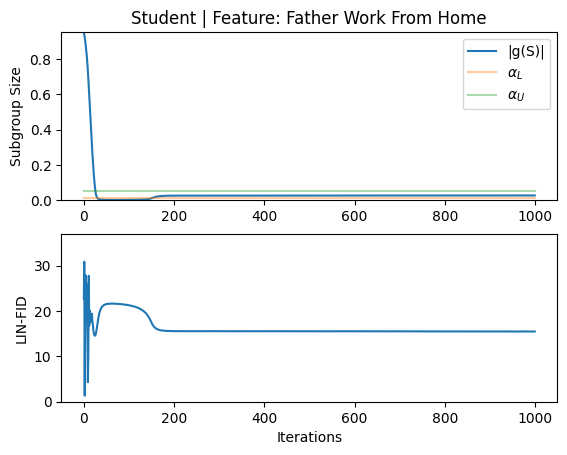}
\end{subfigure}%
\begin{subfigure}{.39\textwidth}
  \centering
  \includegraphics[width=\linewidth]{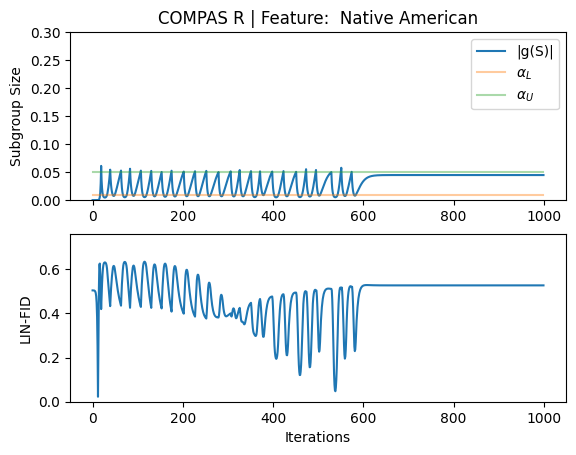}
\end{subfigure} \\
\begin{subfigure}{.39\textwidth}
  \centering
  \includegraphics[width=\linewidth]{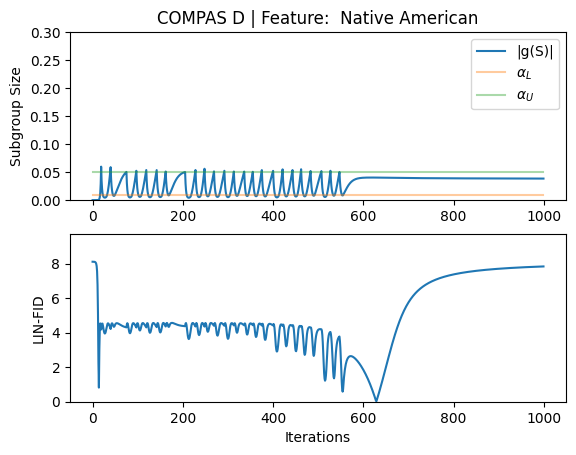}
\end{subfigure}%
\begin{subfigure}{.39\textwidth}
  \centering
  \includegraphics[width=\linewidth]{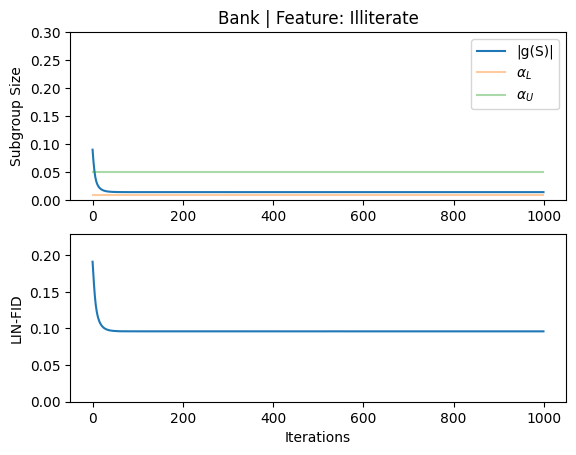}
\end{subfigure} \\
\begin{subfigure}{.39\textwidth}
  \centering
  \includegraphics[width=\linewidth]{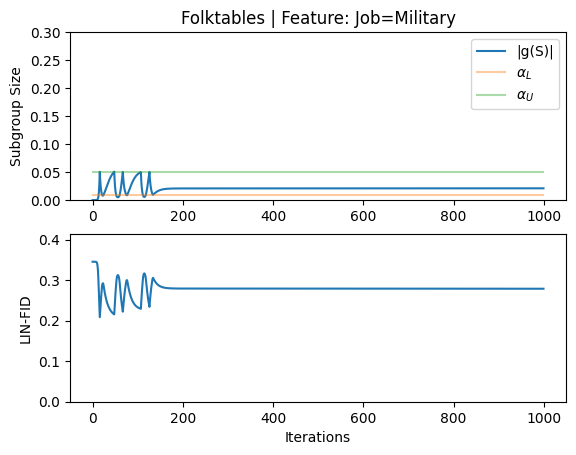}
\end{subfigure}
\caption{Plots of subgroup size and linear regression coefficient  of $g$ over the training iterations of the Adam optimizer. For each dataset, the feature with the highest $\linfid$ was displayed.}
\label{fig:conv_app}
\end{figure}

\section{Importance Notion Consistency}
\label{sec:gen_app}

To see how consistent importance notion methods were, we plotted the values of $F(f_j, X_{test},h)$ against $F(f_j, X_{train},h)$ with each point representing a feature $f_j$ of the COMPAS dataset. The closer these points track the diagonal line, the more consistent a method is in providing the importance values. As we can see in Figure~\ref{fig:gen_app}, LIME and GRAD are extremely consistent. Linear regression is less consistent, due to instability in fitting the least squares estimator on ill-conditioned design matrices. SHAP is also inconsistent in its feature importance attribution, however the $\avgfid$ still generalized well as seen in Figure~\ref{fig:fid_gen}. This could mean that while SHAP is inconsistent from dataset to dataset, it is consistent relative to itself. i.e. if $F(j,X_{train}) > F(j,X_{test})$ then $F(j,g(X_{train})) > F(f,X_{test})$ meaning the $\avgfid$ value would remain the same.

These inconsistencies seem to be inherent in some of these explainability methods as noted in other research \cite{disagree, Dai_2022, agarwal2022rethinking, alvarezmelis2018robustness, bansal2020sam}. Exploring these generalization properties would be an exciting future direction for this work.

\begin{figure}[H]
\centering
\begin{subfigure}{.4\textwidth}
  \centering
  \includegraphics[width=\linewidth]{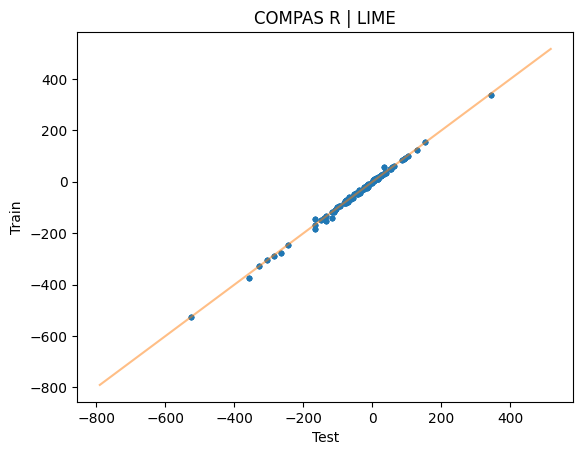}
\end{subfigure}
\begin{subfigure}{.4\textwidth}
  \centering
  \includegraphics[width=\linewidth]{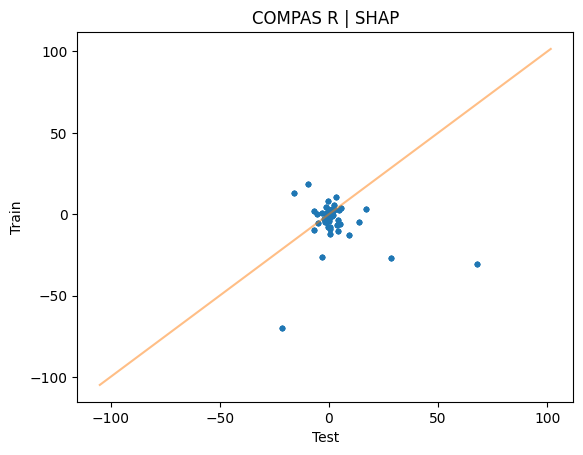}
\end{subfigure} \\
\begin{subfigure}{.4\textwidth}
  \centering
  \includegraphics[width=\linewidth]{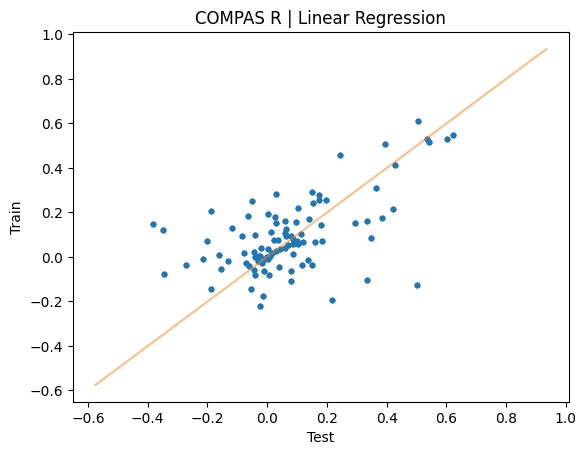}
\end{subfigure}
\begin{subfigure}{.4\textwidth}
  \centering
  \includegraphics[width=\linewidth]{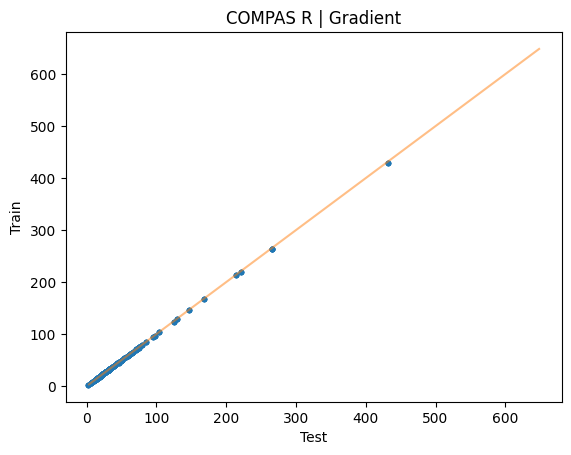}
\end{subfigure} \\
\caption{Consistencies of importance notions. Each point represents a feature, the x-value is $F(j,X_{test})$, and y-value is $F(j,X_{train})$. The closer the points are to the diagonal, the more consistent the notion is.}
\label{fig:gen_app}
\end{figure}

\section{Ethical Review}
\label{sec:ethics}

There were no substantial ethical issues that came up during this research process. The datasets used are all publicly available, de-identified, and have frequently been used in fair machine learning research. There was no component of this research that sought to re-identify the data or use it in any fashion other than to test our methodology.

\end{document}